%% file: main.tex
\documentclass[12pt]{article}
\usepackage{natbib}
\usepackage{times}
\usepackage{fullpage}

\usepackage{microtype}
\usepackage{graphicx}
\usepackage{booktabs} %

\usepackage[hypertexnames=false]{hyperref}
\usepackage{url}            %
\usepackage{booktabs}       %
\usepackage{amsfonts}       %
\usepackage{nicefrac}       %
\usepackage{microtype}      %
\usepackage{xspace}
\usepackage{amsmath}
\usepackage{algorithm}
\usepackage{algpseudocode}
\usepackage{xcolor}
\usepackage{enumitem}
\usepackage{comment}
\usepackage{authblk}
\usepackage{bm}
\usepackage{subcaption}
\usepackage{csquotes}
\usepackage{standard_definitions}
\usepackage{tabularx}
\usepackage{capt-of}
\usepackage{multirow}
\usepackage{wrapfig}
\usepackage{textcomp}
\usepackage{pifont}

\newtheorem*{theorem*}{Theorem}
\theoremstyle{definition}
\newtheorem*{runex*}{Running Example}
\newtheorem*{ex*}{Example}
\newtheorem*{remark*}{Remark}

\newcount\Comments  %
\definecolor{darkgreen}{rgb}{0,0.5,0}
\definecolor{darkred}{rgb}{0.7,0,0}
\definecolor{teal}{rgb}{0.3,0.8,0.8}
\definecolor{orange}{rgb}{1.0,0.5,0.0}
\definecolor{purple}{rgb}{0.8,0.0,0.8}
\newcommand{\kibitz}[2]{\ifnum\Comments=1{\textcolor{#1}{\textsf{\footnotesize #2}}}\fi}

\newcommand{\wen}[1]{\kibitz{green}{{#1}}}

\newcommand{\AlgLong}{Policy Cover-Policy Gradient }
\newcommand{\alg}{PC-PG\xspace}
\newcommand{\calK}{\mathcal{K}}

\newcommand{\E}{\ensuremath{\mathbb{E}}}
\newcommand{\ra}[1]{\renewcommand{\arraystretch}{#1}}
\makeatletter
\providecommand\theHALG@line{\thealgorithm.\arabic{ALG@line}}
\makeatother

\title{\alg: Policy Cover Directed Exploration for\\ Provable Policy Gradient Learning}
\author[1]{Alekh Agarwal}
\author[2]{Mikael Henaff}
\author[3,1]{Sham Kakade}
\author[4]{Wen Sun\thanks{Work done while MH and WS were at Microsoft Research}}
\affil[1]{Microsoft Research}
\affil[2]{Facebook AI Research}
\affil[3]{University of Washington}
\affil[4]{Cornell University}

\begin{document}
\renewcommand\algorithmicindent{1.0em}%

\maketitle

\input{abstract}

\input{intro}

\input{related_work}
\input{setting}

\input{algorithm_new}
\input{analysis}

\input{experiment}

\input{conclusion}

\bibliography{refs,references,sref}
\bibliographystyle{plainnat}

\onecolumn
\appendix

\tableofcontents
\newpage

\input{appendix_two_loop}

\input{app_examples}

\input{appendix_new}

\input{appendix_exp}

\end{document}

%% file: abstract.tex
\begin{abstract}
Direct policy gradient methods for reinforcement learning are a
successful approach for a variety of reasons: they are model free,
they directly optimize the performance metric of interest, and they
allow for richly parameterized policies. 
Their primary drawback is that, by being local in nature, they fail to
adequately explore the environment.
In contrast, while model-based approaches and Q-learning directly handle exploration through the use of
optimism, their ability to handle model misspecification and function
approximation is far
less evident.
This work introduces the the \emph{\AlgLong} (\alg) algorithm, which
provably balances the exploration vs. exploitation tradeoff using an
ensemble of learned policies (the policy cover). \alg enjoys
polynomial sample complexity and run time for both tabular MDPs and,
more generally, linear MDPs in an infinite dimensional RKHS.
Furthermore, \alg also has strong guarantees under model
misspecification that go beyond the standard worst case $\ell_\infty$
assumptions; this includes approximation guarantees for state
aggregation under an average case error assumption, along with guarantees
under a more general assumption where the approximation error under
distribution shift is controlled. We complement the theory with
empirical evaluation across a variety of domains in both reward-free
and reward-driven settings. \end{abstract}

%% file: intro.tex
\section{Introduction}
\label{section:intro}

Policy gradient methods are a successful class of
Reinforcement Learning (RL) methods, as they are amenable to parametric policy classes, including
neural policies~\citep{schulman2015trust, schulman2017proximal}), and
they directly optimizing the cost function of interest.
While
these methods have a long history in the RL
literature~\citep{williams1992simple, sutton1999policy,
  konda2000actor, Kakade01}, only recently have their theoretical convergence properties
been established: roughly when the objective function has wide coverage over the state
space, global convergence is possible~\citep{agarwal2019optimality,geist2019theory,russoGlobal,abbasi2019politex}.
In other words, the assumptions in these works imply that the state
space is already well-explored. Conversely, without such coverage (and, say, with sparse rewards),
policy gradients often suffer from the vanishing gradient problem.

With regards to exploration, at least in the tabular setting, there is
an established body of results which provably explore in order to
achieve sample efficient reinforcement learning, including model based
methods~\citep{kearns2002optimal,brafman2002r,kakade2003sample,
jaksch2010optimal,azar2017minimax,dann2015sample}, model free
approaches such as
Q-learning~\citep{strehl2006pac,li2009unifying,jin2018q,dong2019provably},
thompson
sampling~\citep{osband2014generalization,agrawal2017optimistic,russo2019worst},
and, more recently, policy optimization
approaches~\citep{efroni2020optimistic,cai2019provably}. In fact, more
recently, there are number of provable reinforcement learning
algorithms, balancing exploration and exploitation, for MDPs with
linearly parameterized dynamics,
including~\cite{jiang2017contextual,yang2019sample,jin2019provably,pmlr-v108-zanette20a,ayoub2020,zhou2020provably,cai2019provably}.

The motivation for our work is to develop algorithms and guarantees
which are more robust to violations in the underlying modeling assumptions; indeed, the
primary practical motivation for policy gradient methods is that the
overall methodology is disentangled from modeling
(and Markovian) assumptions, since they are an ``end-to-end'' approach, directly
optimizing the cost function of interest.
Furthermore, in support of these empirical findings, there is a body of theoretical
results, both on direct policy optimization approaches~\citep{kakade2002approximately,NIPS2003_2378,
  Scherrer:API,scherrer2014local} and
more recently on policy gradient
approaches~\citep{agarwal2019optimality}, which show that such incremental
policy improvement approaches are amenable to function
approximation and violations of modeling assumptions,
under certain coverage assumptions over the state
space.

This work focuses on
how policy gradient methods can be extended to handle exploration, while also
retaining their favorable properties with regards to how they handle function approximation and
model misspecification.
The practical relevance of answering these
questions is evident by the growing body of empirical techniques for
exploration in policy gradient methods such as pseudocounts
~\citep{bellemare2016pseudocounts}, dynamics model errors
~\citep{pathak2017curiosity}, or random network distillation
(RND)~\citep{burda2018exploration}.

\begin{table*}[t!]
\centering
\aboverulesep=0ex
 \belowrulesep=0ex
\ra{2}
\begin{tabular}{|>{\centering\arraybackslash}m{9.5cm}|>{\centering\arraybackslash}m{2.8cm}|>{\centering\arraybackslash}m{3cm}|}
\toprule Algorithm & Sample Complexity & Misspecified State Aggregation\\ \midrule
E$^3$, Rmax, UCBVI \hspace{9cm}\hfill \footnotesize{\citep{kearns2002optimal,brafman2002r,jaksch2010optimal,azar2017minimax}} & $\text{poly}(S,A, H, \frac{1}{\epsilon})$  & $\ell_\infty$\\ \midrule
Thompson Sampling \hspace{9cm}\hfill ~\footnotesize{\citep{osband2014generalization,agrawal2017optimistic,russo2019worst}} & $\text{poly}\left(S,A, H, \frac{1}{\epsilon}\right)$  & $\ell_\infty$\\ \midrule
Q-learning ($\epsilon$ greedy) & $\Omega(A^{H})$  & $\ell_\infty$\\ \midrule
delayed/UCB Q-learning \hspace{9cm}\hfill \footnotesize{\citep{strehl2006pac,li2009unifying,jin2018q,dong2019provably}} & $\text{poly}\left(S,A, H, \frac{1}{\epsilon}\right)$
 & $\ell_\infty$ for $Q^\star$\\ \midrule
Policy Optimization \hspace{9cm}\hfill \footnotesize{(PG\citep{williams1992simple,sutton1999policy}, NPG \citep{Kakade01, agarwal2019optimality}, MD-MPI \cite{geist2019theory})} & $\Omega(A^h)$  & ?\\ \midrule
Optimistic Policy Optimization in the Empirical Model
  ~\footnotesize{\cite{cai2019provably,efroni2020optimistic}}
& $\text{poly}\left(S, A, H, \frac{1}{\epsilon}\right)$ & $\ell_\infty$\\ \midrule
\alg (this paper) & $\text{poly}\left(S,A, H, \frac{1}{\epsilon}\right)$  & local $\ell_\infty$ \\ \bottomrule
\end{tabular}
\caption{Comparison of algorithms in tabular
  (and state-aggregation) settings. For the last column, state-aggregation
provides a means to compare
tabular approaches when the aggregated MDP may only approximately be an
MDP (i.e. when there is a model misspecification).
 We assume the agent
  starts at a fixed starting state $s_0$ and only has the ability to do rollouts from
  the state $s_0$. Sample complexity is for the number of samples
  required to learn an $\epsilon$-optimal policy. $Q$-learning and standard
  policy optimization have an exponential sample complexity in $H :=
  1/(1-\gamma)$ due to that they do not actively explore. If the
  starting state distribution had coverage (as opposed to starting at
  a single state $s_0$), then stronger guarantees exist for policy
  optimization
  methods~\citep{kakade2002approximately,agarwal2019optimality}, both
  with regards to sample complexity and state-aggregation. The optimistic policy optimization approaches
 of~\citep{cai2019provably,efroni2020optimistic} build an empirical
 model of the transition dynamics and do optimistic policy updates in
 this empirical model; as such, the can also viewed as being model based, unlike $Q$-learning and \alg
 which do not store and use prior data.
 \alg removes the initial state distribution assumptions~\citep{kakade2002approximately,agarwal2019optimality} from prior
 policy gradient results through incorporating strategic exploration;
 this is done via learning
 an ensemble of policies, the policy cover. \alg extends to linear MDPs
 with linear function approximation as well, and it also works under a weaker
 error
 condition when state aggregation is performed as the type of
 function approximation.
}


\label{tbl:tabular}
\vspace{-5pt}
\end{table*}

\subsection{Our Contributions}
This work introduces the \AlgLong algorithm
(\alg), a
direct, model-free, policy optimization approach which addresses
exploration through the use of a learned ensemble of policies, the latter
provides a policy cover over the state space. The use of a learned policy
cover addresses exploration, and also addresses what is the ``catastrophic forgetting'' problem in policy gradient
approaches (which use reward bonuses); while the on-policy nature avoids the ``delusional bias''
inherent to Bellman backup-based approaches, where approximation
errors due to model misspecification amplify  (see~\citep{lu2018non} for discussion).

It is a conceptually different approach from the predominant prior
(and provable) RL
algorithms, which are either model-based --- variants
of
UCB~\cite{kearns2002optimal,brafman2002r,jaksch2010optimal,azar2017minimax}
or based on Thompson sampling~\cite{agrawal2017optimistic,russo2019worst} --- or model-free and value based, such as
Q-learning~\cite{jin2018q,strehl2006pac}. Our work adds policy optimization methods to
this list, as a direct alternative: the use of learned covers permits a
\emph{a model-free approach} by allowing the algorithm to plan in the real
world, using the cover for initializing the underlying policy optimizer. We remark that only a handful of prior
(provable) exploration algorithms~\cite{jin2018q,strehl2006pac} are model-free in
the tabular setting, and these are largely value based.

Table~\ref{tbl:tabular} shows the relative
landscape of results for the tabular case. Here, we can compare
tabular approaches when the MDP may only approximately be an
MDP. For the latter, we consider the question of
\emph{state-aggregation}, where states are aggregated into
``meta-states'' due to some given state-aggregation function~\citep{li2006towards}. The hope
is that the aggregated MDP is also approximately an MDP (with a
smaller number of aggregated state). Table~\ref{tbl:tabular} compares the
effectiveness of tabular algorithms in this case, where the state-aggregation function
introduces model misspecification. Importantly, \alg provides a local guarantee, in a more model
agnostic sense, unlike model-based and Bellman-backup based methods.

Our main results show that \alg is provably sample and
computationally efficient for \emph{both tabular and linear
  MDPs}, where \alg finds a near optimal policy with a polynomial
sample complexity in all the relevant parameters in the (linear) MDP.
Furthermore, we give theoretical support that the direct approach is
particularly favorable with regards to function approximation and
model misspecification. Highlights are as follows:

\paragraph{RKHS in Linear MDPs:} For the linear MDPs proposed by
\cite{jin2019provably}, 
our results hold when the linear MDP features
live in an infinite dimensional Reproducing Kernel Hilbert Space
(RKHS).  It is not immediately evident how to extend the prior work on linear
MDPs (e.g.~\citep{jin2019provably})
to this setting (due to concentration issues with data re-use).
The following informal theorem summarizes
this contribution.

\begin{theorem}[Informal theorem for \alg on linear MDPs] With high
  probability, \alg finds an $\epsilon$ near optimal policy with
  number of samples $\widetilde{O}\left(\text{poly}\left(1/(1-\gamma), \mathcal{I}_N,
      1/\epsilon, W \right)\right)$, where $W$ is related to the maximum RKHS
  norm of any policy's Q function and $\mathcal{I}_N$ is the maximum
  information gain defined with respect to the kernel. Here,
  $\mathcal{I}_N$ implicitly measures the effective dimensionality of the
  problem, and  $\mathcal{I}_N =
  \widetilde{O}(d)$ for a linear kernel with $d$-dimensional features.
\end{theorem}



\paragraph{Bounded transfer error and state aggregation:} When specialized to a state aggregation setting, we show
that \alg provides a different approximation guarantee in
comparison to
prior works. In particular, the aggregation need only be good locally, under the visitations of the comparison policy. This
means that quality of the aggregation need only be good in the
regions where a high value policy tends to visit. More generally, we
analyze \alg under a notion of a small transfer error in critic
fitting~\citep{agarwal2019optimality}---a condition on the error of a
best on-policy critic under a comparison policy's state
distribution---which generalizes the special case of state
aggregation, and show that \alg enjoys a favorable sample complexity
whenever this transfer error is small. We also instantiate the general
result with other concrete examples where \alg is effective, and where we argue
prior approaches will not be provably accurate.
The following is an informal statement for the special case of state-aggregation with
model-misspecification.

\begin{theorem}[Informal theorem for state aggregation]  With high
  probability, \alg finds an $\epsilon + \epsilon_{misspec}$ near
  optimal policy with
  $\widetilde{O}\left(\text{poly}\left(|\mathcal{Z}| , 1/(1-\gamma),
      1/\epsilon\right) \right)$ many samples, where $\mathcal{Z}$ is
  the set of abstracted states; $\epsilon_{misspec} =
  O\left({\EE_{s\sim d^\star}[\max_{a} \epsilon_{misspec}(s,a)}] / (1-\gamma)^3 \right)$ where $d^\star$ is the state
  visitation distribution of an optimal policy (the distribution of
  which states an optimal policy tends to visit),
  and $\epsilon_{misspec}(s,a)$ is a measure of the model-misspecification error
  at state action $s,a$ (a disagreement measure
  between dynamics and rewards of state-action pairs aggregated to the same abstract state as $s,a$).
\end{theorem}



\paragraph{Empirical evaluation:} We provide experiments showing the viability of \alg in settings where prior bonus based approaches such as Random Network Distillation~\citep{burda2018exploration} do not
  recover optimal policies with high probability. Our experiments show
  our basic approach complements and leverages existing deep learning
  approaches, implicitly also verifying the robustness of \alg outside the regime where the sample complexity bounds provably hold.


%% file: related_work.tex
\subsection{Related Work}
\label{sec:related}

We first discuss work with regards to policy gradient methods and
incremental policy optimization; we then discuss work with regards to
exploration in the context of explicit (or implicit) assumptions on
the MDP (which permit sample complexity that does not explicitly
depend on the number of states); and then ``on-policy'' exploration
methods. Finally, we discuss the recent and concurrent work of
\citet{cai2019provably,efroni2020optimistic}, which provide an
optimistic policy optimization approach which uses off-policy data.

Our line of work seeks to extend the recent line of provably correct
policy gradient
methods~\cite{agarwal2019optimality,fazel2018global,russoGlobal,caiTRPO,even-dar2009online,
DBLP:journals/corr/NeuJG17,Azar:2012:DPP:2503308.2503344,abbasi2019politex}
to incorporate exploration. As discussed in the intro, our focus is
that policy gradient methods, and more broadly ``incremental''
methods --- those methods which make gradual policy changes such as Conservative Policy Iteration (CPI)
\citep{kakade2002approximately,scherrer2014local,Scherrer:API}, Policy
Search by Dynamic Programming (PSDP)~\citep{NIPS2003_2378}, and
MD-MPI~\cite{geist2019theory} --- have guarantees with function
approximation that are stronger than the more abrupt approximate
dynamic programming methods, which rely on the boundedness of the more
stringent concentrability coefficients~\cite{munos2005error,
szepesvari2005finite, antos2008learning}; see
\citet{Scherrer:API,agarwal2019optimality,geist2019theory,chen2019information,shani2019adaptive} for
further discussion.  Our main agnostic result shows how \alg is more robust than all
extant bounds with function approximation in terms of both
concentrability coefficients and distribution mismatch coefficients;
as such, our results require
substantially weaker assumptions, building on the recent work of~\citet{agarwal2019optimality} who develop a similar notion of robustness in the policy optimization setting without exploration. Specifically, when specializing to linear MDPs and tabular MDPs, our algorithm is PAC while algorithms such as CPI and NPG are not PAC without further assumption on the reset distribution \citep{{agarwal2019optimality}}.

We now discuss results with regards to exploration in the
context of explicit (or implicit) assumptions on the underlying
MDP. To our knowledge, all prior works only provide provable
algorithms, under either realizability assumptions or under well
specified modelling assumptions; the violations tolerated in these
settings are, at best, in an $\ell_\infty$-bounded, worst case sense.
The most general set of results are those in
\cite{jiang2017contextual}, which proposed the concept of Bellman Rank
to characterize the sample complexity of value-based learning methods
and gave an algorithm that has polynomial sample complexity in terms
of the Bellman Rank, though the proposed algorithm is not
computationally efficient.  Bellman rank is bounded for a wide range
of problems, including MDPs with small number of hidden states,  linear
MDPs, LQRs, etc.  Later work gave computationally efficient algorithms
for certain special
cases~\citep{dann2018polynomial,du2019provably,yang2019reinforcement,jin2019provably, homer}.
Recently, Witness rank, a generalization of Bellman rank to model-based methods, was proposed by~\cite{sun2019model} and was later extended to model-based reward-free exploration by \cite{henaff2019explicit}.
We focus on the linear MDP model, studied in~\cite{yang2019reinforcement,jin2019provably}.  We note
that \citet{yang2019reinforcement} also prove a result for a type of linear
MDPs, though their model is significantly more restrictive than the
model in ~\citet{jin2019provably}. Another notable result is due to
\citet{wen2013efficient}, who showed that in deterministic systems, if the
optimal $Q$-function is within a pre-specified function class which
has bounded Eluder dimension (for which the class of linear functions
is a special case), then the agent can learn the optimal policy using
a polynomial number of samples; this result has been generalized by \cite{du2019provably}
to deal with stochastic rewards, using further assumptions such as low
variance transitions and strictly positive optimality gap.

With regards to ``on-policy'' exploration
methods, to our knowledge, there are relatively few provable results
which are limited to the tabular case. These are all based on
Q-learning with uncertainty bonuses in the tabular setting, including
the works in ~\cite{strehl2006pac,jin2018q}. More generally, there are
a host of results in the tabular MDP setting that handle exploration,
which are either model-based or which re-use data (the re-use of data
is often simply planning in the empirical model), which include
~\cite{brafman2003r,kearns2002optimal,azar2017minimax,kakade2003sample,jaksch2010optimal,agrawal2017optimistic,
lattimore2012pac,lattimore2014near, dann2015sample, szita2010model}.

\citet{cai2019provably,efroni2020optimistic} recently study algorithms
based on exponential gradient updates for tabular MDPs, utilizing the
mirror descent analysis first developed in ~\cite{even-dar2009online}
along with idea of optimism in the face of uncertainty. Both
approaches use a critic computed from
off-policy data and can be viewed as model-based, since the algorithm
stores all previous off-policy data and plans in what is effectively
the empirically estimated model (with appropriately chosen uncertainty
bonuses); in constrast, the model-free approaches such as
$Q$-learning do not store the empirical model and have a substantially
lower memory
footprint (see~\cite{jin2018q} for discussion on this latter point).
\citet{cai2019provably} further analyze their algorithm in the linear
kernel MDP model~\citep{zhou2020provably}, which is a different model
from what is referred to as the linear MDP
model~\citet{jin2019provably}. Notably, neither model is a special case of
the other. It is worth observing that the linear kernel MDP model of \cite{zhou2020provably} is characterized by at most $d$
parameters, where $d$ is the feature dimensionality, so that model-based learning
is feasible; in contrast, the linear MDP model of
~\citet{jin2019provably} requires a number of parameter that is
$S\cdot d$ and so it is not
describable using a small number of parameters (and yet, sample efficient RL
is still possible). See
~\citet{jin2019provably} for further discussion.

%% file: setting.tex
\section{Setting}
\label{section:setting}

A Markov Decision Process (MDP) $\mathcal{M} = (\Scal, \Acal, P, r, \gamma,s_0)$
is specified by a state space $\Scal$; an action space $\Acal$; a
transition model $P: \mathcal{S} \times \mathcal{A} \rightarrow \Delta(\mathcal{S})$ (where $\Delta(\mathcal{S})$ denotes a distribution over states),
a reward function $r: \Scal\times \Acal \to [0,1]$,
a discount factor $\gamma \in [0, 1)$, and a starting state
 $s_0$. We assume $\Acal$ is discrete and denote $A = \lvert\Acal\rvert$. Our results generalize to a starting state distribution
 $\mu_0\in\Delta(\Scal)$ but we use a single starting state $s_0$ to
 emphasize the need to perform exploration. A policy $\pi: \Scal \to
 \Delta(\Acal)$
specifies a decision-making strategy in which the agent chooses
actions based on the current state, i.e., $a \sim\pi(\cdot | s)$.

The value function $V^\pi(\cdot,r): \Scal \to \mathbb{R}$ is
defined as the expected discounted sum of future rewards, under reward
function $r$, starting at state $s$
and executing $\pi$, i.e.
\begin{align*}
V^\pi(s;r) := \EE \left[\sum_{t=0}^\infty \gamma^t  r(s_t, a_t)
  | \pi, s_0 = s\right],
\end{align*}
where the expectation is taken with respect to the randomness of the policy and environment $\mathcal{M}$.
The \emph{state-action} value function $Q^\pi(\cdot,\cdot;r): \Scal
\times \Acal \to \mathbb{R}$
is defined as
\begin{align*}
Q^\pi(s,a;r) := \EE\left[\sum_{t=0}^\infty \gamma^t  r(s_t, a_t) | \pi,
  s_0 = s, a_0 = a \right].
  \end{align*}

We define the discounted state-action
distribution $d_{s}^\pi$ of a policy $\pi$:
\begin{center}
\mbox{$d_{s'}^\pi(s,a) := (1-\gamma) \sum_{t=0}^\infty \gamma^t {\Pr}^\pi(s_t=s,a_t=a|s_0=s')$},
\end{center}
where $\Pr^\pi(s_t=s,a_t=a|s_0=s')$ is the 
probability that $s_t=s$ and $a_t=a$, after we execute $\pi$ from $t=0$ onwards starting at state
$s'$ in model $\mathcal{M}$. Similarly, we define $d^{\pi}_{s',a'}(s,a)$ as:
\begin{align*}
d^{\pi}_{s',a'}(s,a) := (1-\gamma) \sum_{t=0}^{\infty} \gamma^t {\Pr}^{\pi}(s_t = s, a_t = s | s_0=s', a_0 = a').
\end{align*}For any state-action distribution $\nu$, we write $d^{\pi}_{\nu}(s,a):= \sum_{(s',a')\in\mathcal{S}\times\mathcal{A}} \nu(s',a') d^{\pi}_{s',a'}(s,a)$. For ease of presentation, we assume that the agent can reset to $s_0$ at any point in the trajectory.\footnote{This can be replaced with a termination at each step with probability $1-\gamma$.} We denote $d^{\pi}_{\nu}(s) = \sum_{a}d^{\pi}_{\nu}(s,a)$.

The goal of the agent is to find a policy $\pi$
that maximizes the expected value from the starting state $s_0$, i.e. the optimization problem is:
$  \max_\pi V^{\pi}(s_0)$,
where the $\max$ is over some policy class.

For completeness, we specify a $d^{\pi}_{\nu}$-sampler and an unbiased estimator of $Q^{\pi}(s,a; r)$ in Algorithm~\ref{alg:sampler_est}, which are standard in discounted MDPs. The $d^{\pi}_\nu$ sampler samples $\sa$ i.i.d from $d^{\pi}_{\nu}$, and the $Q^{\pi}$ sampler returns an unbiased estimate of $Q^{\pi}(s,a;r)$ for a given triple $(s,a,r)$ by a single roll-out from $\sa$.

\paragraph{Notation.}
When clear from context, we write $d^\pi(s,a)$ and $d^\pi(s)$ to denote
$d_{s_0}^\pi(s,a)$ and $d^{\pi}_{s_0}(s)$ respectively, where $s_0$ is the starting state in our
MDP.
For iterative algorithms which obtain policies at each episode, we let $V^{n}$,$Q^{n}$ and $A^{n}$ denote the
corresponding quantities associated with episode $n$.
For a vector $v$, we denote $\|v\|_2=\sqrt{\sum_i v_i^2}$, $\|v\|_1=\sum_i |v_i|$,
and $\|v\|_\infty=\max_i |v_i|$. For a matrix $V$, we define $\|V\|_2 =
\sup_{x:\|x\|_2\leq 1}\| V x \|_2$,
and $\det(V)$ as the determinant of $V$. We use $\text{Uniform}(\Acal)$ (in short $\text{Unif}_{\Acal}$) to represent a uniform distribution over the set $\Acal$.

\begin{algorithm}[!t]
\caption{$d^{\pi}_\nu$ sampler and $Q^{\pi}$ estimator}
\label{alg:sampler_est}
\begin{algorithmic}[1]
\setcounter{algorithm}{-1}
\Function{$d_{\nu}^\pi$-sampler}{}
\State  \hspace*{-0.1cm}\textbf{Input}:  $\nu\in\Delta(\Scal\times\Acal), \pi, r\sa$
\State Sample $s_0,a_0 \sim \nu$
\State Execute $\pi$ from $s_0, a_0$; at any step $t$ with $(s_t,a_t)$, terminate the episode with probability $1-\gamma$
\State  \hspace*{-0.1cm}\textbf{Return}: $s_t,a_t$
\caption{$d^{\pi}$ sampler and $Q^{\pi}$ estimator}
\EndFunction
\Function{$Q^\pi$-estimator}{}
\State  \hspace*{-0.1cm}\textbf{Input}:  current state-action $\sa$, reward $r\sa$, $\pi$
\State Execute $\pi$ from $(s_0,a_0) = (s, a)$; at step $t$ with $(s_t,a_t)$, terminate with probability $1-\gamma$
\State  \hspace*{-0.1cm}\textbf{Return}: $\widehat{Q}^{\pi}\sa = \sum_{i=0}^t r(s_i,a_i)$ where $(s_0,a_0) = (s,a)$
\caption{$d^{\pi}$ sampler and $Q^{\pi}$ estimator}
\EndFunction
\setcounter{algorithm}{1}
\end{algorithmic}
\end{algorithm}

%% file: algorithm_new.tex


\section{The \AlgLong (\alg) Algorithm}
\label{section:alg}

To motivate the algorithm, first consider the original objective function:
\begin{equation}\label{eq:obj1}
\textrm{Original objective:} \quad \max_{\pi\in\Pi} V^\pi(s_0;r)
\end{equation}
where $r$ is the true cost function. Simply doing policy gradient ascent on
this objective function may easily lead to poor stationary points due
to lack of coverage (i.e. lack of exploration). For example, if
the initial visitation measure $d^{\pi^0}$ has poor coverage over the
state space (say $\pi^0$ is a random initial policy), then $\pi^0$ may already being a
stationary point of poor quality (e.g see Lemma 4.3 in~\cite{agarwal2019optimality}).

In such cases, a more desirable objective function is of the form:
\begin{equation}\label{eq:obj2}
\textrm{A wide coverage objective:} \quad \max_{\pi\in\Pi} \EE_{s_0,a_0\sim\rho_{\mix}}\left[ Q^\pi(s_0,a_0;r)\right]
\end{equation}
where $\rho_{\mix}$ is some initial state-action distribution which has wider
coverage over the state space. As argued
in~\citep{agarwal2019optimality,kakade2002approximately,scherrer2014local,Scherrer:API},
wide coverage initial distributions $\rho_{\mix}$ are critical to the success of policy
optimization methods. However, in the RL setting, our agent can only
start from $s_0$.


\input{pseudo_epoc}

The idea of our iterative algorithm, \alg (Algorithm~\ref{alg:epoc}), is to successively improve \emph{both} the
current policy $\pi$ and the coverage distribution $\rho_{\mix}$. The
algorithm starts with some policy $\pi^0$ (say random), and
works in episodes. At episode $n$, we have $n+1$ previous
policies $\pi^0, \ldots \pi^n$. Each of these policies $\pi^i$ induces
a distribution $d^i := d^{\pi^i}$ over the state space. Let us
consider the average state-action visitation measure over \emph{all} of
these previous policies:
\begin{align}
\rho_{\mix}^n(s,a) = \sum_{i=0}^{n} d^i(s,a)/(n+1)
\label{eq:cover_def}
\end{align}
Intuitively, $\rho_{\mix}^n$ reflects the coverage the
algorithm has over the state-action space at the start of the $n$-th episode. \alg then
uses $\rho_{\mix}^n$ in the previous
objective~\eqref{eq:obj2} with two modifications: \alg
modifies the instantaneous reward function $r$ with a bonus $b^n$ in
order to encourage the algorithm to find a policy $\pi^{n+1}$ which
covers a novel part of the state-action space.
It also modifies the policy class from $\Pi$ to $\Pi_{\text{bonus}}$,
where all policies $\pi \in \Pi_{\text{bonus}}$ are constrained to
simply take a  random rewarding action for those states where the
bonus is already large (random exploration
is reasonable when the exploration bonus is already large, see Eq~\ref{eq:pg_update} in Alg.~\ref{alg:npg}).
With this, \alg's objective at the
$n$-th episode is:
\begin{equation}\label{eq:obj3}
\textrm{\alg's objective:} \quad \max_{\pi\in\Pi_{\text{bonus}}} \EE_{s_0,a_0\sim\rho^n_{\mix}}\left[ Q^\pi(s_0,a_0;r+b^n)\right]
\end{equation}
The idea is that \alg can effectively optimize
over the region where $\rho^n_{\mix}$ has coverage. Furthermore, by
construction of the bonus, the algorithm is encouraged to escape the
current region of coverage to discover novel parts of the state-action
space. We now describe the bonus and optimization steps in more
detail.

\input{pseudo_npg}

\paragraph{Reward bonus construction.}
At each episode $n$, \alg maintains an estimate of feature covariance of the policy cover $\rho^n_\mix$
(Line \ref{line:feature_cov_mix} of \pref{alg:epoc}). Next we use this covariance matrix to identify state-action pairs which are adequately
covered by $\rho^n_\mix$.
The goal of the reward bonus is to identify state,
action pairs whose features are less explored by $\rho^n_{\mix}$ and
incentivize visiting them. The bonus $b^n(s,a)$ defined
in Line~\ref{line:bonus} achieves this.
If  $\widehat{\Sigma}^n_{\mix}$ has a small eigenvalue along $\phi\sa$, then we assign
the largest possible reward-to-go (i.e., $1/(1-\gamma)$) for this $\sa$ pair to encourage
exploration.\footnote{For an infinite dimensional RKHS,
  the bonus can be computed in the dual using the kernel
  trick (e.g., \cite{valko2013finite}).}

\paragraph{Policy Optimization.}  With the bonus, we update the policy via $T$ steps of natural policy gradient (\pref{alg:npg}). In the NPG update, we first approximate the
value function $Q^{\pi^t}(s,a; r+ b^n)$ under the policy cover $\rho^n_\mix$ (\pref{line:learn_critic}).  Specifically, we use linear function approximator to approximate $Q^{\pi^t}(s,a; r+b^n) - b^n\sa$ via constrained linear regression (\pref{line:learn_critic}), and then approximate $Q^{\pi^t}(s,a;r+b^n)$ by adding bonus back:
\begin{align*}
\overline{Q}^t_{b^n}\sa :=  b^n\sa + \theta^t\cdot \phi\sa,
\end{align*} Note that the error of $\overline{Q}^t_{b^n}\sa$ to $Q^{\pi^t}(s,a;r+b^n)$ is simply the prediction error of $\theta^t\cdot \phi\sa$ to the regression target $Q^{\pi^t}(s,a;r+b^n) - b^n$. The purpose of structuring the value function estimation this way, instead of directly approximating $Q^t(s,a;r+b^n)$ with a linear function, for instance, is that the regression problem defined in \pref{line:learn_critic} will have a good linear solution for the special case linear MDPs, while we cannot guarantee the same for $Q^t(s,a;r+b^n)$ due to the non-linearity of the bonus. 

 We then use the critic $\overline{Q}^t_{b^n}$  for updating policy (Eq.~\pref{eq:pg_update}).
These are the exponential gradient updates (as in ~\cite{Kakade01,agarwal2019optimality}), but are constrained for $s\in\Kcal^n$ (see \pref{line:known} for the definition of $\Kcal^n$). 
The initialization and the update ensure that $\pi^t$ chooses actions uniformly from $\{a: b^n\sa > 0\}\subseteq\mathcal{A}$ at any state $s$ with $\left\lvert  \{ a: b^n\sa > 0\}\right\vert > 0$ (the policy is restricted to act uniformly among positive bonus actions).

\paragraph{Intuition for tabular setting.}
In tabular MDPs (with ``one-hot'' features for each state-action pair), $(\widehat \Sigma^n_{\mix})^{-1}$ is a diagonal matrix
with entries proportional to $1/n_{s,a}$, where $n_{s,a}$ is the
number of times  $\sa$ is observed in the data collected to form the
matrix $\widehat\Sigma^n_{\mix}$. Hence the bonus simply rewards
infrequently visited state-action pairs,
and thus encourages reaching new state-action pairs.

%% file: pseudo_epoc.tex
\begin{algorithm}[!t]
\begin{algorithmic}[1]
\State \hspace*{-0.1cm}\textbf{Input}: iterations $N$, threshold $\beta$, regularizer $\lambda$
\State \hspace*{-0.1cm}Initialize $\pi^0(a|s)$ to be uniform 
\For{episode $n = 0, \dots N-1$}
\State  Estimate the covariance of $\pi^n$ as
$\widehat{\Sigma}^n = \sum_{i=1}^K \phi(s_i,a_i)\phi(s_i,a_i)^{\top}/K$ with $\{s_i,a_i\}_{i=1}^K \sim d^n$ \label{line:feature_cov}
\State Estimate the covariance of the policy cover as \label{line:feature_cov_mix}
$\widehat\Sigma_\mix^n  :=  \sum_{i=0}^n \widehat{\Sigma}^i + \lambda I $ \label{line:feature_cov_mix}
\State Set the exploration bonus $b^n$ to reward infrequently visited state-action under $\rho^n_{\mix}$\label{line:bonus} (\ref{eq:cover_def})
\begin{equation*}
b^n(s,a) = \frac{\one\{\sa~:~ \phi\sa^\top (\widehat{\Sigma}_{\mix}^n)^{-1}\phi\sa \geq \beta\}}{1-\gamma}.
\end{equation*}
\State Update $\pi^{n+1} = \textrm{NPG-Update}(\rho^n_{\mix}, b^n)$ (\pref{alg:npg})
\EndFor
\end{algorithmic}
\caption{\AlgLong (\alg)}
\label{alg:epoc}
\end{algorithm}

%% file: pseudo_npg.tex
\begin{algorithm}[!t]
\begin{algorithmic}[1]
\State \hspace*{-0.1cm}\textbf{Input}
$\rho^n_{\mix}$, $b^n$, learning rate $\eta$, sample size $M$ for critic fitting, iterations $T$
\State Define $\Kcal^n = \{s: \forall a\in\Acal, b^n\sa = 0\}$ \label{line:known}
\State Initialize policy $\pi^0: \Scal\to\Delta(\Acal)$, such that  
\begin{align*}
\pi^0(\cdot | s) = \begin{cases}
\text{Uniform}(\Acal) &  s\in\Kcal^n \\
\text{Uniform}(\{a\in\Acal: b^n\sa > 0\}) & s\not\in\Kcal^n.
\end{cases}
\end{align*}
\For{$ t = 0 \to T-1$} 
\State Draw $M$ i.i.d samples $\left\{s_i,a_i, \widehat{Q}^{\pi^t}(s_i,a_i; r+b^n)\right\}_{i=1}^M$ with $s_i,a_i\sim {\rho^n_\mix}$ (see Alg~\ref{alg:sampler_est})
	\State \textbf{Critic} fit: \label{line:learn_critic}
	\begin{align*}
		\theta^t = \argmin_{\|\theta\|\leq W} \sum_{i=1}^M \left( \theta\cdot  \phi(s_i,a_i) - \left(\widehat{Q}^{\pi^t}(s_i,a_i;r + b^n) - b^n(s_i,a_i)\right) \right)^2
	\end{align*}
	\State \textbf{Actor} update 
	\begin{equation}
	\pi^{t+1}(\cdot |s) \propto \pi^t(\cdot |s) \exp\left( \eta \left(b^n(s,\cdot) + \theta^t\cdot \phi(s,\cdot)  \right)    \one\{s\in\Kcal^n\}   \right)
	\label{eq:pg_update}
	\end{equation}
\EndFor
\State \Return $\pi := \argmax_{\pi\in\{\pi^0,\dots, \pi^{T-1}\}} V^{\pi}(s_0; r+b^n)$ 
\end{algorithmic}
\caption{Natural Policy Gradient (NPG) Update}
\label{alg:npg}
\end{algorithm}

%% file: analysis.tex
\section{Theory and Examples}
\label{sec:analysis}

For the analysis, we first state sample complexity results for linear MDPs. Specifically, we focus on analyzing linear MDPs with infinite dimensional features (i.e., the transition and reward live in an RKHS) and show that \alg's sample complexity scales polynomially with respect to the maximum information gain \citep{srinivas2010gaussian}.

We then demonstrate the robustness of \alg to model misspecification in two concrete ways. We first provide a result for state aggregation, showing that error incurred is only an average model error from aggregation averaged over the fixed comparator's abstracted state distribution, as opposed to an $\ell_{\infty}$ model error (i.e., the maximum possible model error over the entire state-action space due to state aggregation). We then move to a more general agnostic setting and show that our algorithm is robust to model-misspecification which is measured in a new concept of \emph{transfer error} introduced by \cite{agarwal2019optimality} recently.  Compared to the Q-NPG analysis from \cite{agarwal2019optimality}, we show that \alg eliminates the assumption of having access to a well conditioned initial distribution (recall in our setting agent can only reset to a fixed initial state $s_0$), as our algorithm actively maintains a policy cover.

We also provide other examples where the linear MDP assumption is only valid for a sub-part of the MDP, and the algorithm competes with the best policy on this sub-part, while most prior approaches fail due to the delusional bias of Bellman backups under function approximation and model misspecification~\citep{lu2018non}.

\subsection{Well specified case: Linear MDPs}
\label{sec:linear}
Let us define linear MDPs first \citep{jin2019provably}. Rather
than focusing on finite feature dimension as \citet{jin2019provably}
did, we directly work on linear MDPs in a general Reproducing Kernel
Hilbert space (RKHS).
\begin{definition}[Linear MDP] Let $\Hcal$ be a Reproducing Kernel Hilbert Space (RKHS), and define a
  feature mapping $\phi:\Scal\times\Acal\to \Hcal$. An MDP $(\Scal, \Acal, P, r, \gamma, s_0)$ is called a linear MDP if the reward
  function lives in $\Hcal$: $r\sa = \langle \theta, \phi\sa
  \rangle_{\Hcal}$, and the transition operator $P(s'|s,a)$ also lives in
  $\Hcal$: $P(s'|s,a) = \langle \mu(s'), \phi\sa \rangle_{\Hcal}$ for
  all $(s,a,s')$. Denote $\mu$ as a matrix whose each row corresponds to $\mu(s)$. We assume the parameter norms\footnote{The norms are induced by the inner product in the Hilbert space $\Hcal$, unless stated otherwise.} are bounded as $\|\theta\| \leq \omega$, $ \|v^{\top}\mu\| \leq \xi $ for all $v\in\mathbb{R}^{|\Scal|}$ with $\|v\|_{\infty} \leq 1$.
  \label{def:linear_mdp}
\end{definition}

As our feature vector $\phi$ could be infinite dimensional, to measure the sample complexity, we define the \emph{maximum information gain} of the underlying MDP $\mathcal{M}$.
First, denote the covariance matrix of any policy $\pi$ as $\Sigma^{\pi} = \EE_{\sa\sim d^{\pi}}\left[\phi\sa\phi\sa^{\top}\right]$.%
We define the maximum information gain below:
\begin{definition}[Maximum Information Gain $\mathcal{I}_N(\lambda)$] We define the maximum information gain as:
\begin{align*}
\mathcal{I}_N(\lambda) := \max_{\{\pi^i\}_{i=0}^{N-1}} \log\det\left( \frac{1}{\lambda}\sum_{i=0}^{N-1} \Sigma^{\pi^i} +  I \right),
\end{align*} where $\lambda \in\mathbb{R}^+$.
\label{def:int_dim}
\end{definition}
\begin{remark}
This quantity is identical to the maximum information gain in Gaussian
Process bandits  \citep{srinivas2010gaussian} from a Bayesian perspective. 
A related quantity occurs in a more
restricted linear MDP model, in \citet{yang2019reinforcement}. Note
that when $\phi\sa\in \mathbb{R}^d$, we have that
$\log\det\left(\sum_{i=1}^n \Sigma^{\pi^i} + \mathbf{I}\right) \leq d
\log(nB^2/\lambda + 1)$ assuming $\|\phi(s,a)\|_2\leq B$, which means that the
information gain is always at most $\widetilde{O}(d)$.  Note that $\mathcal{I}_N(\lambda)\ll d$ if the
covariance matrices from a sequence of policies are concentrated in a
low-dimensional subspace (e.g., $\phi$ is infinite dimensional while all policies only visits a two dimensional subspace).
\end{remark}




For linear MDPs, we leverage the following key observation: we have that $Q^{\pi}(s,a;r+b^n) - b^n\sa$ is linear with respect to $\phi\sa$ for any possible bonus function $b^n$ and policy $\pi$, which we prove in \pref{claim:linear_property}. The intuition is that the transition dynamics are still linear (as we do not modify the underlying transitions) with respect to $\phi$, so a Bellman backup $r\sa + \EE_{s'\sim P_{\sa}} V^{\pi}(s' ; r+b^n)$ is still linear with respect to $\phi\sa$ (recall that linear MDP has the property that a Bellman backup on any function $f(s')$ yields a linear function in features $\phi\sa$). This means that we can successfully find a linear critic to approximate $Q^{\pi^t}(s,a; r+b^n) - b^n\sa$ under $\rho^n_\mix$ up to a statistical error, i.e.,
\begin{align*}
\EE_{\sa\sim \rho^n_\mix} \left(\theta^t \cdot \phi\sa - \left(Q^{\pi^t}(s,a;r+b^n)  -b^n\sa  \right) \right)^2 = O\left( 1/\sqrt{M} \right),
\end{align*} where $M$ is number of samples used for constrained linear regression (\pref{line:learn_critic}). This further implies that $\theta^t\cdot \phi\sa + b^n\sa$ approximates $Q^{\pi^t}(s,a;r+b^n)$ up to the same statistical error.


With this intuition, the following theorem states the sample complexity of \alg under the linear MDP assumption.




\begin{theorem}[Sample Complexity of \alg for Linear MDPs]
Fix $\epsilon, \delta \in (0,1)$ and an arbitrary comparator policy $\pi^\star$ (not necessarily an optimal policy). Suppose that $\mathcal{M}$ is a linear MDP (\ref{def:linear_mdp}). 
There exists a setting of the parameters such that \alg
uses a number of samples at most $\text{poly}\left( \frac{1}{1-\gamma},\log(A), \frac{1}{\epsilon}, \mathcal{I}_N(1), \omega, \xi, \ln\left(\frac{1}{\delta}\right) \right)$
and, with probability greater than $1-\delta$, returns a policy $\widehat \pi$ such that:
\begin{align*}
V^{\widehat\pi}(s_0) \geq V^{\pi^\star}(s_0) - \epsilon.
\end{align*} 
\label{thm:linear_mdp}
\end{theorem}


A few remarks are in order:
\begin{remark}For tabular MDPs, as $\phi$ is a $|\Scal||\Acal|$ indictor vector, the theorem above immediately extends to tabular MDPs with $\mathcal{I}_N(1)$ being replaced by $|\Scal|\Acal| \log(N+1)$. 
\end{remark}
\begin{remark}
In contrast with LSVI-UCB \citep{jin2019provably}, \alg works for
infinite dimensional $\phi$ with a polynomial dependency on the
maximum information gain $\mathcal{I}_N(1)$. To the best of our knowledge,
this is the first efficient model-free on-policy policy gradient result for linear MDPs and also the first
 infinite dimensional result for the linear MDP model proposed
by \citet{jin2019provably}.
\end{remark}

Instead of proving \pref{thm:linear_mdp} directly, we will state and prove a general theorem of \alg for general MDPs with model-misspecification measured in a new concept \emph{transfer error} (\pref{ass:transfer_bias}) introduced by \cite{agarwal2019optimality} in \pref{sec:agnostic_result}.
\pref{thm:linear_mdp} can be understood as a corollary of a more general agnostic theorem (\pref{thm:agnostic}). Detailed proof of \pref{thm:linear_mdp} is included in \pref{app:app_to_linear_mdp}.

\input{state_aggregation}

\input{transfer_bias_no_ind}

\input{misspecified_example}

%% file: state_aggregation.tex
\subsection{State-Aggregation under Model Misspecification}

Consider a simple model-misspecified setting where the model error is introduced due to state action aggregation. Suppose we have an aggregation function $\phi:
\Scal\times\Acal \rightarrow\mathcal{Z}$, where $\mathcal{Z}$ is a finite
categorical set, the ``state abstractions", which we typically
think of as being much smaller than the (possibly infinite) number of state-action pairs. Intuitively, we aggregate state-action pairs that have similar transitions and rewards to an abstracted state $z$. 
This aggregation introduces model-misspecification, defined below.
\begin{definition}[State-Action Aggregation Model-Misspecification]
We define model-misspecification $\epsilon_{misspec}(z)$ for any $z\in\mathcal{Z}$ as
\begin{align*}
\epsilon_{misspec}(z) := \max_{(s,a),(s',a') \textrm{
    s.t. }\phi(s,a)=\phi(s',a')=z} \Big\{ \left\|P(\cdot|s,a) - P(\cdot|s',a')\right\|_1, \left\lvert r(s,a) - r(s',a') \right\rvert \Big\}.
\end{align*}
\end{definition}The model-misspecification measures the maximum possible disagreement in terms of transition and rewards of two state-action pairs which are mapped to the same abstracted state. 

We now argue that \alg provides a unique and stronger guarantee in the
case of error in our state aggregation. 
The folklore result is that with the definition
$\|\epsilon_{misspec}\|_\infty=\max_{z\in\mathcal{Z}} \epsilon_{misspec}(z) $, algorithms
such as UCB and $Q$-learning succeed with an additional additive
error of $\|\epsilon_{misspec}\|_\infty/(1-\gamma)^2$, and will have sample complexity
guarantees that are polynomial in only $|\mathcal{Z}|$.  Interestingly, see
~\cite{li2009unifying,dong2019provably} for conditions
which are limited to only $Q^\star$, but which are still \emph{global} in nature.
The
following theorem shows that \alg only requires a more local
guarantee where our aggregation needs to be only good
under the distribution of abstracted states where an optimal policy
tends to visit.


\begin{theorem}[Misspecified, State-Aggregation Bound]\label{thm:state_aggregation}
 Fix
  $\epsilon, \delta\in (0,1)$.
Let $\pi^\star$ be an arbitrary comparator policy. 
There exists a setting of the
  parameters such that \alg (\pref{alg:epoc})
uses a total number of samples at most  $\text{poly}\left(
  |\mathcal{Z}|, \log(A), \frac{1}{1-\gamma}, \frac{1}{\epsilon},
  \ln\left(\frac{1}{\delta}\right) \right)$
and, with probability greater than $1-\delta$, returns a policy
$\widehat \pi$ such that,
\[
V^{\widehat \pi}(s_0)  \geq V^{\pi^\star}(s_0) - \epsilon - \frac{2  \EE_{s \sim d^{\pi^\star} } \max_a \left[\epsilon_{misspec}(\phi(s,a))\right]  }{(1-\gamma)^3}.
\]
\label{thm:state_aggregation}
\end{theorem}

Here, it could be that
$ {\EE_{s\sim d^{\pi^\star}}\max_{a}[\epsilon_{misspec}(\phi(s,a))] }  \ll
\|\epsilon_{misspec}\|_\infty$
due to that our error notion is an average case one under the  comparator.  We refer readers
to \pref{app:state_agg} for detailed proof of the above theorem which can also be regarded as a corollary of a more general agnostic theorem (\pref{thm:agnostic}) that we present in the next section. Note that here we pay an additional $1/(1-\gamma)$ factor in the approximation error due to the fact that after reward bonus, we have $r\sa+b^n\sa \in [0,1/(1-\gamma)]$. \footnote{We note that instead of using reward bonus, we could construct absorbing MDPs to make rewards scale $[0,1]$. This way we will pay $1/(1-\gamma)^2$ in the approximation error instead.}


One point worth reflecting on is how few guarantees there are in the
more general RL setting (beyond dynamic programming),
which address model-misspecification in a manner that goes beyond global
$\ell_\infty$ bounds.  Our conjecture is that this is not merely an
analysis issue but an algorithmic one, where incremental algorithms
such as \alg are required for strong misspecified algorithmic
guarantees. We return to this point in \pref{sec:example}, with
an example showing why this might be the case.

%% file: transfer_bias_no_ind.tex
\subsection{Agnostic Guarantees with Bounded Transfer Error}

\label{sec:agnostic_result}

We now consider a general MDP in this section, where we do not assume
the linear MDP modeling assumptions hold. As $Q - b^n$ may not
be linear with respect to the given feature $\phi$, we need to
consider model misspecification due to the linear function
approximation with features $\phi$.  We use the new concept of transfer
error from \citep{agarwal2019optimality} below. We use the shorthand notation:
\[
Q^t_{b^n}(s,a) =Q^{\pi^t}(s,a; r+b^n)
\] below.
We capture model misspecification using the following assumption. 
\begin{assum}[Bounded Transfer Error] \label{ass:transfer_bias}
With respect to a target function $f:\Scal\times\Acal \rightarrow \mathbb{R}$,
define the critic loss function $L(\theta; d, f)$ with $d\in\Delta(\Scal\times\Acal)$ as:
\begin{align*}
L\left(\theta; d, f\right) :=  \EE_{\sa\sim d}\left( \theta\cdot \phi\sa - f \right)^2,
\end{align*} which is the square loss of using the critic $\theta\cdot\phi$
to predict a given target function $f$, under distribution $d$.
Consider an arbitrary comparator policy $\pi^\star$ (not necessarily an optimal policy)
and denote the state-action distribution $d^\star(s,a) := d^{\pi^\star}(s) \circ \text{Unif}_{\Acal}(a)$.
For all episode $n$ and all iteration $t$ inside episode $n$, define:
\begin{align*}
\theta^t_\star \in \argmin_{\|\theta\|\leq W} L\left( \theta; \rho^n_{\mix}, Q^t_{b^n} - b^n \right)
\end{align*}
Then we assume that (when running Algorithm~\ref{alg:epoc}),
$\theta^t_\star$ has a bounded prediction error when transferred to $d^\star_{}$ from
$\rho^n_{\mix}$; more formally:
\begin{align*}
L\left( \theta^t_\star; d^\star_{}, Q^t_{b^n} - b^n \right) \leq \epsilon_{bias} \in \mathbb{R}^+.
\end{align*}
\end{assum}
Note that the transfer error $\epsilon_{bias}$ measures the
prediction error, at episode $n$ and iteration $t$, of a best on-policy fit $\overline{Q}^t_{b^n}(s,a) :=
b^n\sa + \theta^t_\star\cdot \phi\sa$ measured under a fixed distribution
$d^\star$ from the fixed comparator (note $d^\star$ is different from
the training distribution $\rho^n_\mix$ hence the name
\emph{transfer}).

This assumption first appears in the recent work of~\citet{agarwal2019optimality} in order to analyze policy optimization methods under linear function approximation. As our subsequent examples illustrate in the following section, this is a milder notion of model misspecification than $\ell_\infty$-variants more prevalent in the literature, as it is an average-case quantity which can be significantly smaller in favorable cases. We also refer the reader to~\citet{agarwal2019optimality} for further discussion on this assumption.

With the above assumption on the transfer error, the next theorem states an agnostic result for the sample complexity of \alg:

\begin{theorem}[Agnostic Guarantee of \alg]  Fix $\epsilon, \delta \in (0,1)$ and consider an arbitrary comparator policy $\pi^\star$ (not necessarily an optimal policy).
Assume \pref{ass:transfer_bias} holds.
There exists a setting of the parameters ($\beta, \lambda, K, M, \eta, N, T$) such that \alg
uses a number of samples at most $\text{poly}\left( \frac{1}{1-\gamma},\log(A), \frac{1}{\epsilon}, \mathcal{I}_N(1), W, \ln\left(\frac{1}{\delta}\right) \right)$
and, with probability greater than $1-\delta$, returns a policy $\widehat \pi$ such that:
\begin{align*}
V^{\widehat\pi}(s_0) \geq  V^{\pi^\star}(s_0) - \epsilon - \frac{\sqrt{2A\epsilon_{bias}}}{1-\gamma}.
\end{align*}
\label{thm:agnostic}
\end{theorem}

The precise polynomial of the sample complexity, along with the settings of all the
hyperparameters --- $\beta$ (threshold for bonus), $\lambda$, $K$ (samples for estimating cover's covariance), $M$ (samples for fitting critic), $\eta$ (learning rate in
NPG),  $N$ (number of episodes), and $T$ (number of NPG iterations per episode) --- is provided in \pref{thm:detailed_bound_rmax_pg} (\pref{app:rmaxpg_sample}), {where we also discuss two specific examples of $\phi$---finite dimensional $\phi\in\mathbb{R}^d$ with bounded norm, and infinite dimensional $\phi$ in RKHS with RBF kernel (Remark \pref{remark:kernel_discussion}).}

The above theorem indicates that if the transfer error
$\epsilon_{bias}$ is small, then \alg finds a near optimal policy
in polynomial sample complexity without any further assumption on the
MDPs. Indeed, for well-specified cases such as tabular MDPs and linear
MDPs, due to that the regression target $Q^{\pi}(\cdot,\cdot;r+b^n) - b^n$ function is always a
linear function with respect to the features, one can easily show that
$\epsilon_{bias}=0$ (which we show in
Appendix~\ref{app:app_to_linear_mdp}), as one can pick the best on-policy fit $\theta^t_\star$ to be the exact linear representation of $Q^{\pi}(s,a;r+b^n) - b^n\sa$.   Further, in the state-aggregation example, we can show that $\epsilon_{bias}$ is upper bounded by the expected model-misspecification with respect to the comparator policy's distribution (\pref{app:state_agg}).   

A few remarks are in order to illustrate how the notion of transfer
error compares to prior work.

\begin{remark}[Comparison with concentrability assumptions
  \citep{kakade2002approximately,Scherrer:API,agarwal2019optimality}]
  In the theory for policy gradient methods without explicit
  exploration, a standard device to obtain global optimality
  guarantees for the learned policies is through the use of some
  exploratory distribution $\nu_0$ over initial states and actions in
  the optimization algorithm. Given such a distribution, a key
  quantity that has been used in prior analysis is the maximal density
  ratio to a comparator policy's state
  distribution~\citep{kakade2002approximately,Scherrer:API}:
  $\max_{s\in \Scal} \frac{d^\star(s)}{\nu_0(s)}$, where we use
  $d^\star(s)$ to refer to the probability of state $s$ under the
  comparator $\pi^\star$. It is easily seen that if \alg is run with a
  similar exploratory initial distribution, then the transfer error is
  always bounded as well:
\begin{align*}
\epsilon_{bias} \leq
\left\|\frac{ d^\star_{}  }{ \nu_0}\right\|_{\infty} L\left( \theta^t_\star; \nu_0, Q^t_{b^n} - b^n \right).
\end{align*}
In this work, we do not assume access to a such an exploratory measure
(with coverage); our goal is finding a policy with only access to
rollouts from $s_0$. This makes this concetrability-style analysis
inapplicable in general as the starting measure $\nu_0$ for the
algorithm is potentially the delta measure over the initial state
$s_0$, which
can induce an arbitrarily large density ratio.
In contrast, the transfer
error is always bounded and is zero in well-specified cases such as
tabular MDPs and linear MDPs which we show in
Appendix~\ref{app:app_to_linear_mdp}.
\end{remark}

\begin{remark}[Comparison with the NPG guarantees in \citet{agarwal2019optimality}] The bounded transfer error
  assumption (\pref{ass:transfer_bias}) stated here is developed in the
  recent work of~\citet{agarwal2019optimality}. Their work focuses on
  understanding the global convergence properties of policy gradient
  methods, including the specific NPG algorithm used here; it does
  not consider the design of exploration strategies. Consequently,
  Assumption~\ref{ass:transfer_bias} alone is not sufficient to
  guarantee convergence in their setting; ~\cite{agarwal2019optimality} make an
  additional assumption on a relative condition number between the covariance
  matrices of the comparator distribution $d^\star$ and the initial
  exploratory distribution $\nu_0$:
\[
    \kappa = \sup_{w \in \mathbb{R}^d} \frac{w^\top \Sigma_{d^\star}w}{w^\top \Sigma_{\nu_0} w}, \quad \mbox{where} \quad \Sigma_\upsilon = \E_{s,a\sim \upsilon} [\phi(s,a)\phi(s,a)^\top].
\]
Note that we consider a finite $d$-dimensional feature space for this
discussion to be consistent with the prior results. Under the
assumption that $\kappa < \infty$, \citet{agarwal2019optimality}
provide a bound on the iteration complexity of NPG-style updates with
an explicit dependence on $\sqrt{\kappa}$. Related (stronger) assumptions on the relative condition numbers for all possible policies or the initial distribution $\nu_0$ also appear in the recent works~\citep{abbasi2019politex} and~\citep{abbasi2019exploration} respectively (the latter work still assumes access to an exploratory initial policy). Our result does not have
any such dependence. In contrast, the distribution $\rho^n_{\mix}$
designed by the algorithm, serves as the initial
distribution at episode $n+1$, and the reward bonus explicitly
encourages our algorithm to visit places where the relative condition
number with the current distribution $\rho^n_{\mix}$ is large.
\label{remark:compare_to_Q_npg}
\end{remark}

%% file: misspecified_example.tex
\begin{figure}[t]
\centering
\includegraphics[width=.6\textwidth]{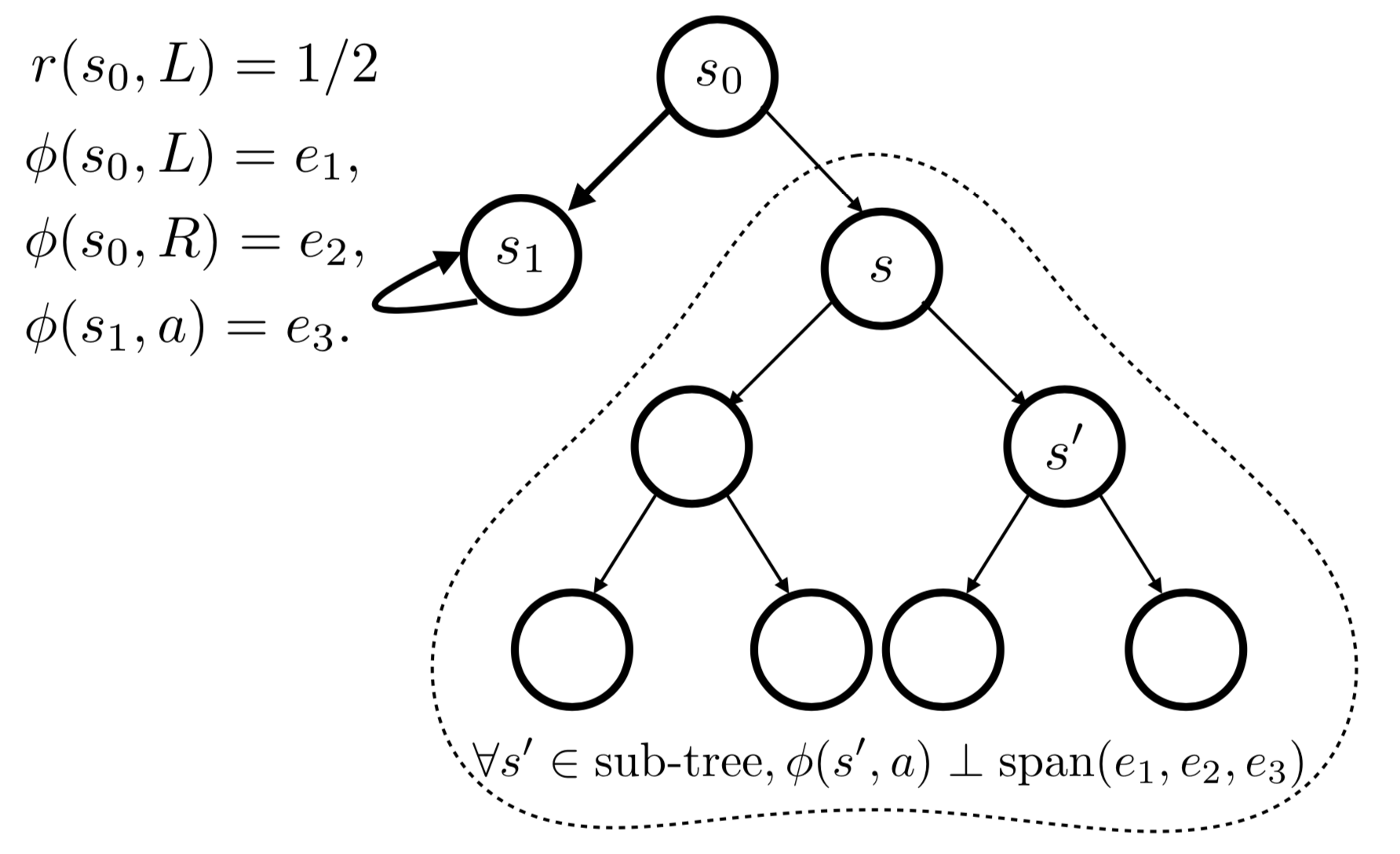}
\caption{The binary tree example. Note that here $s_0$ and $s_1$ have features only span in the first three standard basis, and the features for states inside the binary tree (dashed) contains features in the null space of the first three standard bases. Note that the features inside the binary tree could be arbitrary complicated. Unless the feature dimension scales $\exp(H)$ with $H$ being the depth of the tree, we cannot represent this problem in linear MDPs. The on-policy nature of \alg ensures that it succeeds in this example. Due to the complex features and large $\ell_{\infty}$ model-misspecification inside the binary tree, Bellman-backup based approaches (e.g., Q-learning) cannot guarantee successes.}
\label{fig:binary_tree}
\end{figure}

\subsection{Robustness to ``Delusional Bias'' with Partially Well-specified Models}
\label{sec:example}

In this section, we provide an additional example of model
misspecification where we show that \alg succeeds while Bellman backup
based algorithms do not. The basic spirit of the example is that if
our modeling assumption holds for a sub-part of the MDP, then \alg can
compete with the best policy that only visits states in this sub-part
with some additional assumptions. In contrast, prior model-based and
$Q$-learning based approaches heavily rely on the modeling assumptions
being globally correct, and bootstrapping-based methods
fail in particular due to their
susceptibility to the delusional bias problem~\citep{lu2018non}.

We emphasize that this constructed MDP and class of features have the following properties:

\begin{itemize}
\item It is not a linear MDP; we would need the dimension to be exponential in the depth $H$, i.e. $d=\Omega(2^H)$, in order to even approximate the MDP as a linear MDP.
\item We have no reason to believe that value based methods (that rely on Bellman
  backups, e.g., Q learning) or model based algorithms will provably succeed
  for this example (or simple variants of it).
\item Our example will have large worst case function approximation
  error, i.e. the $\ell_{\infty}$ error in approximating $Q^\star$
  will be (arbitrarily) large.
\item The example can be easily modified so that the concentrability
  coefficient (and the distribution mismatch coefficient) of the
  starting distribution (or a random initial policy)   will be
  $\Omega(2^H)$.
\end{itemize}

Furthermore, we will see that \alg succeeds on this example, provably.

We describe the construction below (see \pref{fig:binary_tree} for an example). There are two actions, denoted by $L$ and $R$. At initial state $s_0$, we have
$P(s_1 | s_0, L) = 1$; $P(s_1 | s_1, a) = 1$ for any $a\in \{L, R\}$. We set the reward of taking the left action at $s_0$ to be $1/2$, i.e.  $r(s_0,L)=1/2$.
This implies that there exists a policy which is guaranteed to obtain at least reward $1/2$.
When taking action $a= R$ at $s_0$, we deterministically transition
into a depth-H completely balanced binary tree.     We can further
constrain the MDP so that the optimal value is $1/2$ (coming from left
most branch), though, as we see later, this is not needed.


The feature construction of $\phi\in \mathbb{R}^d$ is as follows: For
$s_0, L$, we have $\phi(s_0, L) = e_1$ and
$\phi(s_0, R) = e_2$, and $\phi(s_1, a) = e_3$ for any $a\in\{L,R\}$,
where $e_1, e_2, e_3$ are the standard basis vectors. For all other
states $s\not\in\{s_0,s_1\}$, we have that $\phi\sa$ is constrained to be
orthogonal to $e_1$, $e_2$, and $e_3$, but otherwise arbitrary. In
other words, $\phi\sa$ has the first three coordinates equal to zero
for any $s\not\in \{s_0,s_1\}$ but can otherwise be pathological.

The intuition behind this construction is that the features $\phi$
are allowed to be arbitrary complicated for states inside the
depth-H binary tree, but are uncoupled with the features
on the left path. This implies that both \alg and any other algorithm do
not have access to a good global function approximator.

Furthermore, as discussed in the following remark, these features do not provide a good approximation of the true dynamics as a linear MDP.

\begin{remark}(Linear-MDP approximation failure). As the MDP is
  deterministic, we would need dimension $d =
\Omega(2^H)$ in order to approximate the MDP as a linear MDP (in the
sense required in~\cite{jin2019provably}). This is due to that the rank
of the transition matrix is $O(2^H)$.
\end{remark}

However the on-policy nature of \alg ensures that there always exists a best linear predictor that can predict $Q^{\pi}$ well under the optimal trajectory (the left most path) due to the fact that the features on $s_0$ and $s_1$ are decoupled from the features in the rest of the states inside the binary tree. Thus it means that the transfer error is always zero. This is formally stated in the following lemma.
\begin{corollary}[Corollary of Theorem~\ref{thm:agnostic}]
\alg is guaranteed to find a policy with value greater than $1/2-\epsilon$
with probability greater than $1-\delta$, using a number of samples
that is $O\left(\textrm{poly}(H,d, 1/\epsilon,\log(1/\delta))\right)$. This is due to
the transfer error being zero.
\label{cora:agnostic}
\end{corollary}

We provide a proof of the corollary in~\pref{app:examples}.

\paragraph{Intuition for the success of \alg.} Since the corresponding
features of the binary
subtree have no guarantees in the worst-case, \alg may not successfully find the
best global policy in general. However, it does succeed in finding a
policy competitive with the best policy that remains in the
\emph{favorable} sub-part of the MDP satisfying the modeling
assumptions (e.g., the left most trajectory in
\pref{fig:binary_tree}). We do note that the feature orthogonality is
important (at least for a provably guarantee), otherwise the errors in fitting value functions on the
binary subtree can damage our value estimates on the favorable parts
as well; this behavior effect may be less mild in practice.

\paragraph{Delusional bias and challenges with Bellman backup (and Model-based)
  approaches.} While we do not explicitly construct algorithm
dependent lower bounds in our construction, we now discuss why
obtaining guarantees similar to ours with Bellman backup-based
(or even model-based) approaches may be challenging with the current
approaches in the literature. We are not assuming any
guarantees about the quality of the features in the
right subtree (beyond the aforementioned
orthogonality). Specifically, for Bellman backup-based approaches, the
following two observations (similar to those stressed
in~\citet{lu2018non}), when taken together, suggest difficulties for
algorithms which enforce consistency by assuming the Markov property holds:

\begin{itemize}
\item (Bellman Consistency) The algorithm does value based backups, with the property that
  it does an exact backup if this is possible. Note that due to our
  construction, such algorithms will seek to do an exact backup for $Q(s_0,R)$,
  where they estimate $Q(s_0,R)$ to be their value estimate on the right subtree.
  This is due to that the feature $\phi(s_0,R)$ is orthogonal to all other
  features, so a $0$ error, Bellman backup is possible, without altering
  estimation in any other part of the tree.
\item (One Sided Errors) Suppose the true value of
  the  subtree is less than $1/2-\Delta$,  and suppose that there
  exists a set of features where the algorithm approximates the value
  of the subtree to be larger than $1/2$. Current algorithms are not guaranteed
  to return values with one side error; with an arbitrary
  featurization, it is not evident why such a property would hold.
\end{itemize}

More generally, what is interesting about the state aggregation featurization is 
that it permits us to run \emph{any} tabular RL learning
algorithm. Here, it
is not evident that \emph{any} other current tabular RL algorithm,
including model-based approaches, can
achieve guarantees similar to our average-case guarantees, due to their
strong reliance on how they use the Markov property.
In this sense, our work provides a unique guarantee with respect to model misspecification in the RL
setting.

\paragraph{Failure of concentrability-based approaches} Some of the
prior results on policy optimization algorithms, starting from the
Conservative Policy Iteration
algorithm~\citet{kakade2002approximately} and further studied in a
series of subsequent
papers~\citep{Scherrer:API,geist2019theory,agarwal2019optimality}
provide the strongest guarantees in settings without exploration, but
considering function approximation. As remarked in
Section~\ref{sec:linear}, most works in this literature make
assumptions on the maximal density ratio between the initial state
distribution and comparator policy to be bounded. In the MDP
of~\pref{fig:binary_tree}, this quantity seems fine since the ratio is
at most $H$ for the comparator policy that goes on the left path (by
acting randomly in the initial state). However, we can easily change
the left path into a fully balanced binary tree as well, with $O(H)$
additional features that let us realize the values on the leftmost
path (where the comparator goes) exactly, while keeping all the other
features orthogonal to these. It is unclear how to design an initial
distribution to have a good concentrability coefficient, but \alg
still competes with the comparator following the leftmost path since
it can realize the value functions on that path exactly and the
remaining parts of the MDP do not interfere with this estimation.

%% file: experiment.tex
\section{Experiments}
\label{section:experiments}

We provide experiments illustrating \alg's
performance on problems requiring exploration, and focus on showing the algorithm's flexibility to leverage existing policy gradient algorithms with neural networks (e.g., PPO~\citep{schulman2017proximal}).
Specifically, we show that for challenging exploration tasks, our algorithm combined with PPO significantly outperforms both vanilla PPO as well as PPO augmented with the popular RND exploration bonus~\cite{burda2018exploration}.

Specifically, we aim to show the following two properties of \alg:
\begin{enumerate}
\item \alg can build a policy cover that explores the state space widely; hence \alg is able to find near optimal policies even in tasks that have obvious local minimas and sparse rewards. 
\item the policy  cover in \alg avoids catastrophic forgetting issue one can experience in  policy gradient methods due to the possibility that the policy can become deterministic quickly.
\end{enumerate}

For all experiments, we use policies parameterized by fully-connected or convolutional neural networks.
We use a kernel $\phi(s, a)$ to compute bonus as \mbox{$b(s, a) = \phi(s, a)^\top \hat{\Sigma}^{-1}_{\mix} \phi(s, a)$}, where $\hat{\Sigma}_{\mix}$ is the empirical covariance matrix of the policy cover.
In order to prune any redundant policies from the cover, we use a rebalancing scheme to select a policy cover which induces maximal coverage over the state space.
This is done by finding weights $\alpha^{(n)}=(\alpha_1^{(n)}, ..., \alpha_n^{(n)})$ on the simplex at each episode which solve the optimization problem:
$\alpha^{(n)} = \argmax_{\alpha} \log \det \big[ \sum_{i=1}^n \alpha_i \hat{\Sigma}_i \big]$
where $\hat{\Sigma}_i$ is the empirical covariance matrix of $\pi_i$. Details of the implemented algorithm, network architectures and kernels can be found in Appendix \ref{app:exp}.

\subsection{Bidirectional Diabolical Combination Lock}
\begin{figure}[t!]
    \centering
    \begin{tabular}{cc}
    \begin{minipage}{0.47\columnwidth}
    \begin{figure}[H]
    \centering
    \includegraphics[width=\columnwidth]{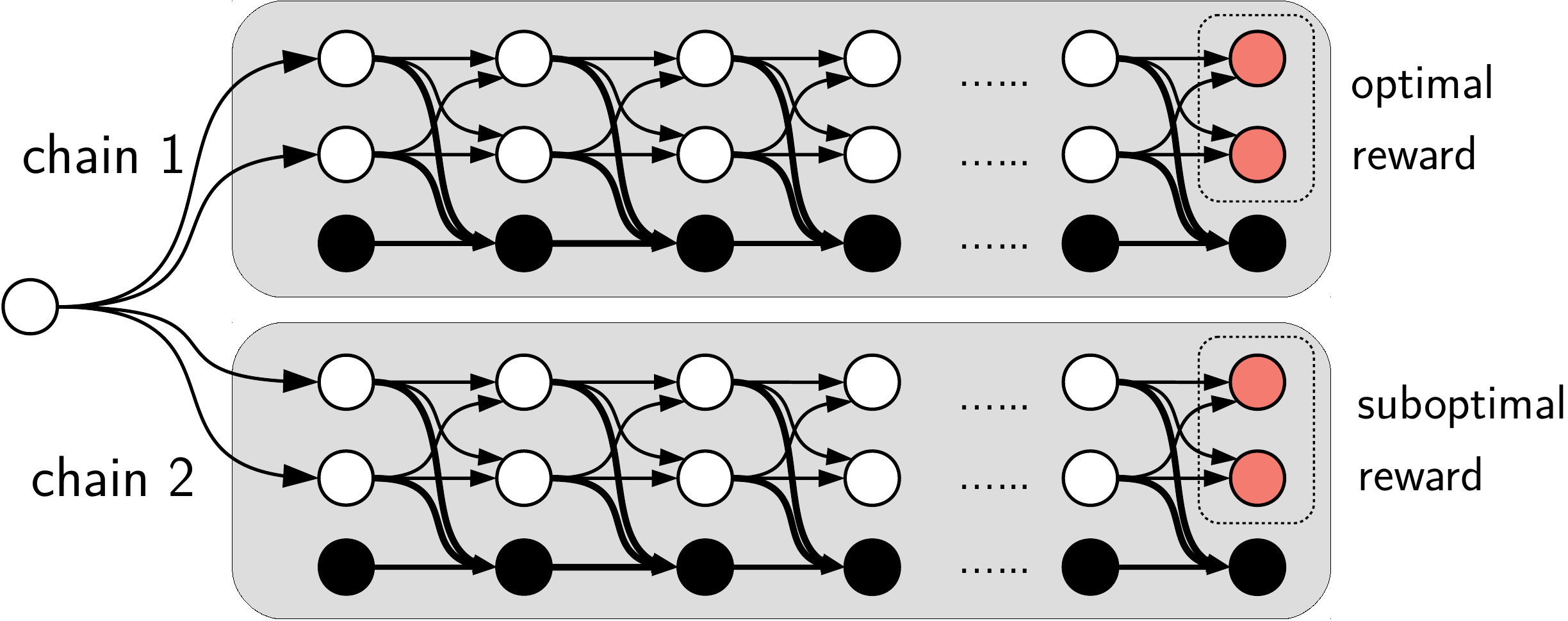} \end{figure}
    \end{minipage} &
  \begin{minipage}{0.4\columnwidth}
 \begin{tabular}{rrrrr}
    \toprule
    \multirow{2}*{Algorithm} & \multicolumn{4}{c}{Horizon}\\\cline{2-5} &$2$ & $5$ & $10$ & $15$ \\
    \midrule
      PPO & 1.0 & 0.0 & 0.0 & 0.0 \\
      PPO+RND & 0.75 & 0.40 & 0.50 & 0.55 \\
      \alg & 1.0 & 1.0 & 1.0 & 1.0 \\
      \bottomrule
  \end{tabular}
  \end{minipage}
  \end{tabular}
  \caption{\textbf{Left} panel shows the Bidirectional Diabolical Combination Lock domain (see text for details). \textbf{Right} panel shows success rate of different algorithms averaged over 20 different seeds.}
  \label{fig:mixture_visitations}
\end{figure}

We first provide experiments on an exploration problem designed to be
particularly difficult: the Bidirectional Diabolical Combination Lock (a harder version of the problem in \cite{homer}, see Figure \ref{fig:mixture_visitations}).
In this problem, the agent starts at an initial state $s_0$ (left most state), and based on its first action, transitions to one of two combination locks of length $H$.
Each combination lock consists of a chain of length $H$, at the end of
which are two states with high reward. At each level in the chain, $9$
out of $10$ actions lead the agent to a dead state (black) from which it
cannot recover and lead to zero reward.
The problem is challenging for exploration for several reasons: (1) \textit{Sparse positive rewards}: Uniform exploration has a $10^{-H}$ chance of reaching a high reward state; (2) \textit{Dense antishaped rewards}: The agent receives a reward of $-1/H$ for transitioning to a good state and $0$ to a dead state. A locally optimal policy is to transition to a dead state quickly; (3) \textit{Forgetting}: At the end of one of the locks, the agent receives a maximal reward of $+5$, and at the end of the other lock it receives a reward of $+2$.
Since there is no indication which lock has the optimal reward, if the agent does not explore to the end of both locks it will only have a $50\%$ chance of encountering the globally optimal reward. If it makes it to the end of one lock, it must remember to still visit the other one. 


\begin{figure}[h]
\centering
  \begin{subfigure}[b]{0.7\textwidth}
\centering
  \includegraphics[width=\textwidth]{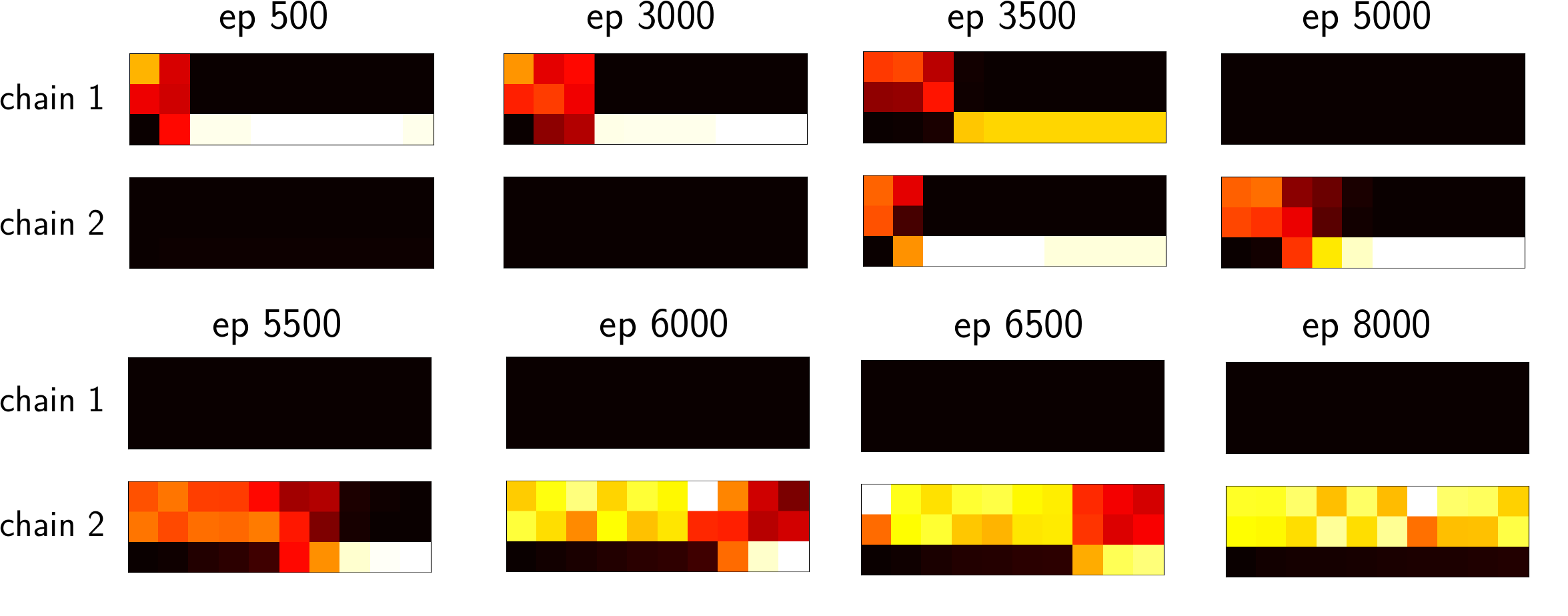}
  \caption{RND trace during training}
  \end{subfigure}
  \begin{subfigure}[b]{0.7\textwidth}
\centering
  \includegraphics[width=\textwidth]{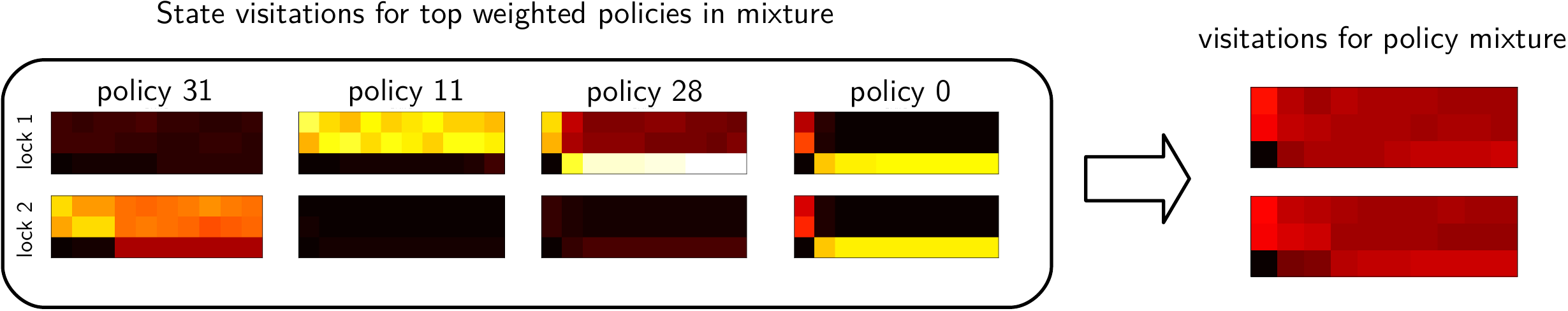}
  \caption{\alg final trace}
  \end{subfigure}
  \caption{ 
     \textbf{(a)} shows the state visitation frequencies (brighter color depicts higher
    visitation frequency) when the RND
    bonus~\cite{burda2018exploration} is applied to a policy gradient method throughout training
    on the above problem. 'Ep' denotes epoch number showing the
    progress during a single training run. Although the agent manages to explore to the
    end of one chain (chain 2 in this case), its policy quickly becomes deterministic and it
    ``forgets'' to explore the remaining chain, missing the optimal
    reward. RND obtains the optimal reward on roughly half of the
    initial seeds.
    \textbf{(b)} panel shows the traces of policies in the policy cover of \alg. Together the policy cover provides a near uniform coverage over both chains.}.
  \label{fig:lock}
\end{figure}

For both the policy network input and the kernel we used a binary vector encoding the current lock, state and time step as one-hot components. We compared to two other methods: a PPO agent, and a PPO agent with a RND exploration bonus, all of which used the same representation as input.

Performance for the different methods is shown in Figure \ref{fig:mixture_visitations} (left).
The PPO agent succeeds for the shortest problem of horizon $H=2$, but fails for longer horizons due to the antishaped reward leading it to the dead states.
The PPO+RND agent succeeds roughly $50\%$ of the time: due to its exploration bonus, it avoids the local minimum and explores to the end of one of the chains. However, as shown in Figure \ref{fig:lock} (a), the agent's policy quickly becomes deterministic and the agent forgets to go back and explore the other chain after it has reached the reward at the end of the first. \alg  succeeds over all seeds and horizon lengths.  We found that the policy cover provides near uniform coverage over both chains. In Figure~\ref{fig:lock} (b) we demonstrate the traces of some individual policies in the policy cover and the trace of the policy cover itself as a whole.


\subsection{Reward-free Exploration in Mazes}

  \begin{figure}[h!]
  \centering
    \includegraphics[width=1.\columnwidth]{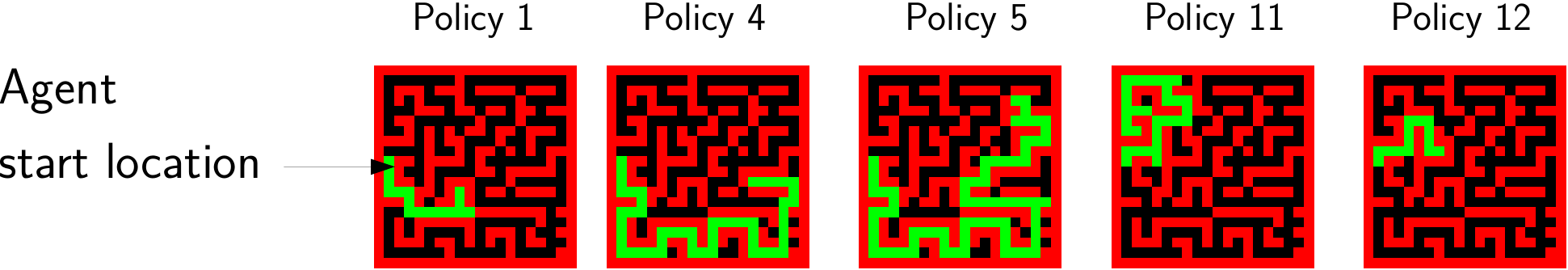}
    \caption{Different policies in the policy cover for the maze environment. All the locations visited by the agent during the policy execution are marked in green.}
    \label{fig:rpg_maze_policies}
  \end{figure}

\begin{figure}[t]
  \centering
  \begin{subfigure}[t]{0.48\columnwidth}
    \includegraphics[width=\columnwidth]{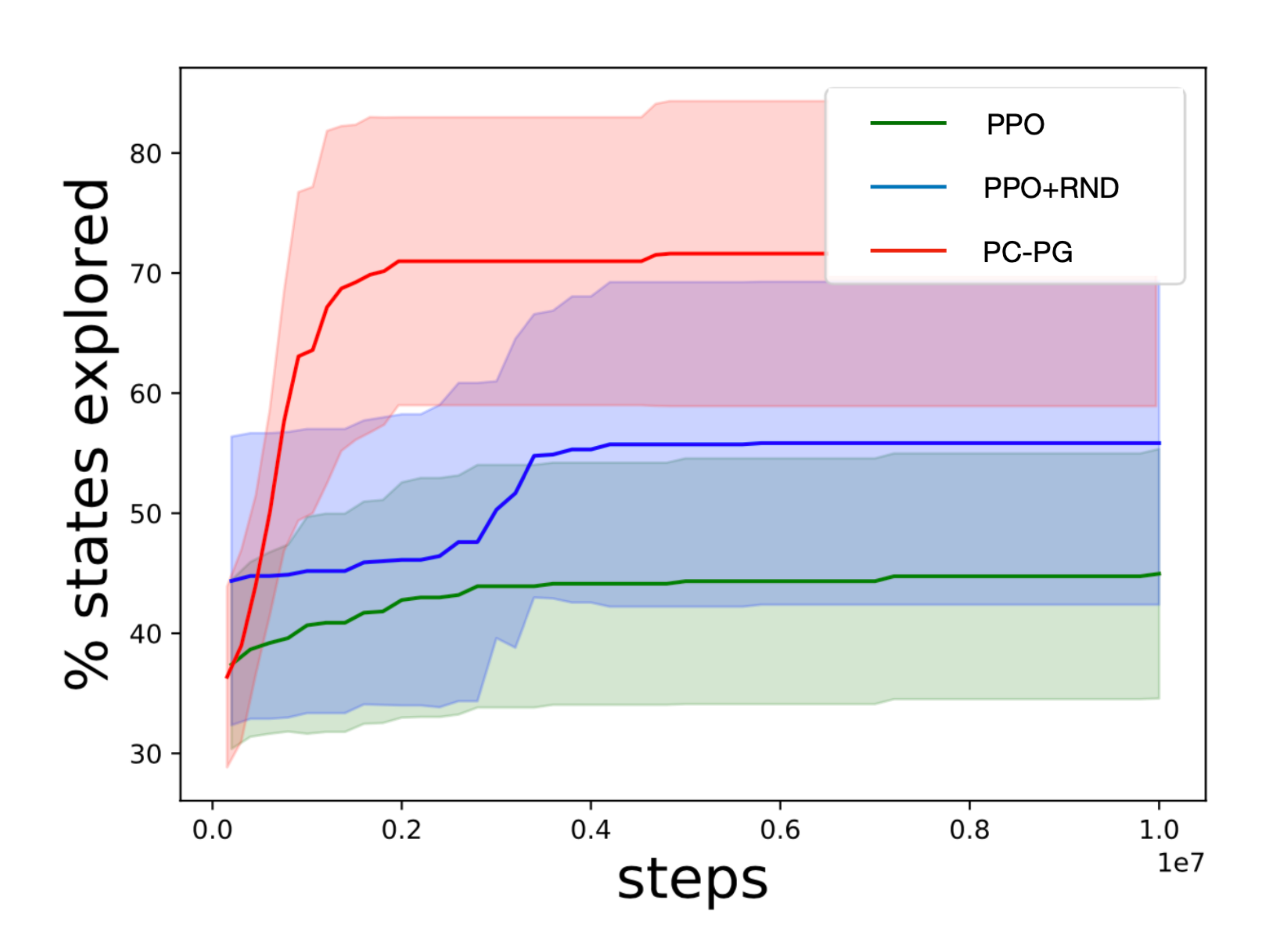}
    \end{subfigure}
    \begin{subfigure}[t]{0.48\columnwidth}
    \includegraphics[width=\columnwidth]{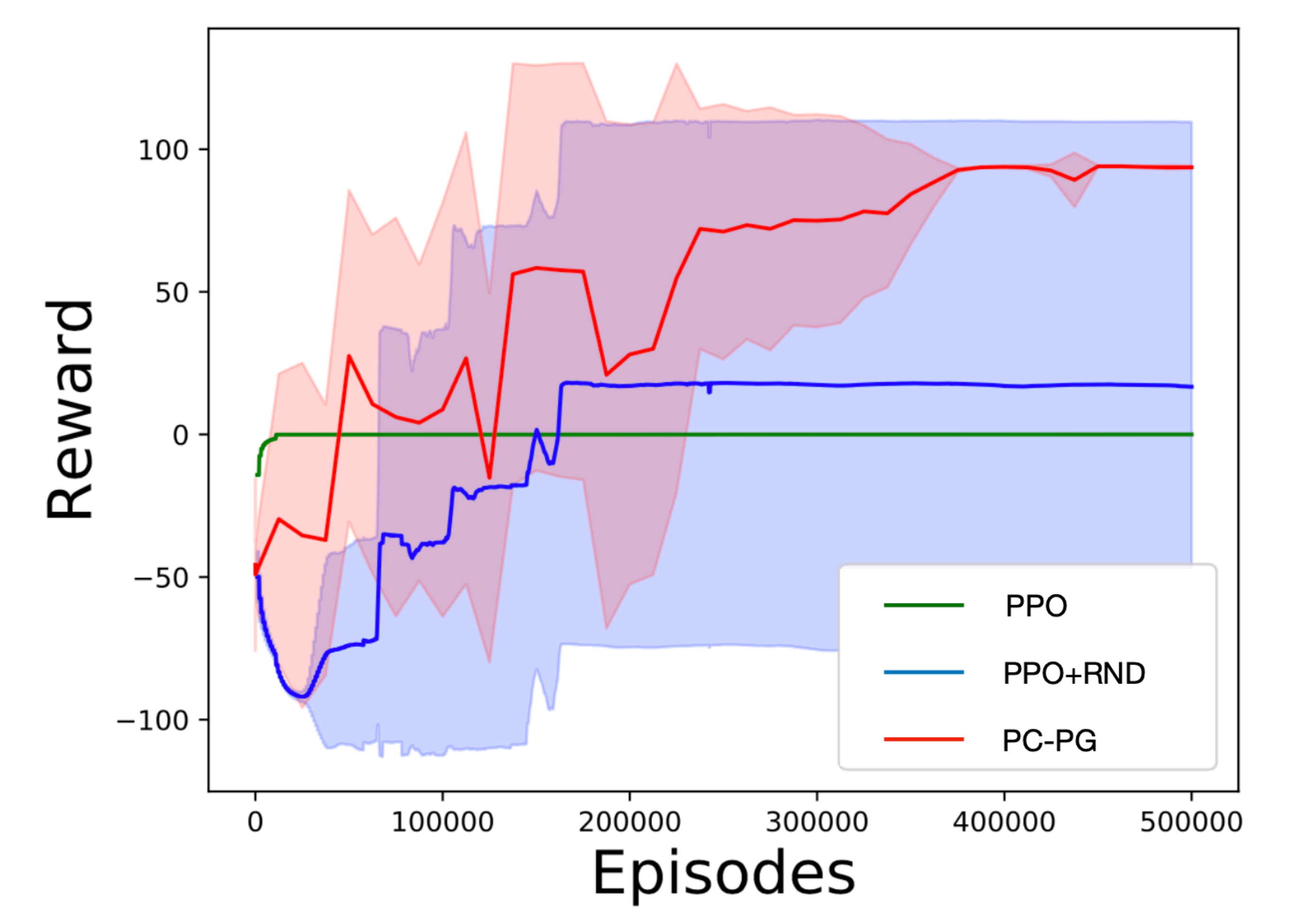}
    \end{subfigure}
  \caption{ \label{fig:reward_free_results}Results for maze (left) \& control (right). Solid line is mean and shaded region is standard deviation over $5$ seeds.}
%
\end{figure}

We next evaluated \alg in a reward-free exploration setting using maze environments adapted from \citep{VPN}.
At each step, the agent's observation consists of an RGB-image of the maze with the red channel representing the walls and the green channel representing the location of the agent (an example is shown in Figure \ref{fig:rpg_maze_policies}). 

We compare \alg, PPO and PPO+RND in the reward-free setting where the agent receives a constant environment reward of $0$ (note that PPO receives zero gradient; \alg and PPO+RND learn from their  reward bonus).
Figure \ref{fig:reward_free_results} (left) shows the percentage of locations in the maze visited by each of the agents over the course of 10 million steps.
The proportion of states visited by the PPO agent stays relatively constant, while the PPO+RND agent is able to explore to some degree. \alg quickly visits a significantly higher proportion of locations than the other two methods. Visualizations of traces from different policies in the policy cover can be seen in Figure \ref{fig:rpg_maze_policies} where we observe the diverse coverage of the policies in the policy cover.

\subsection{Continuous Control}

\begin{figure}[h!]
  \centering
  \includegraphics[width=0.5 \columnwidth]{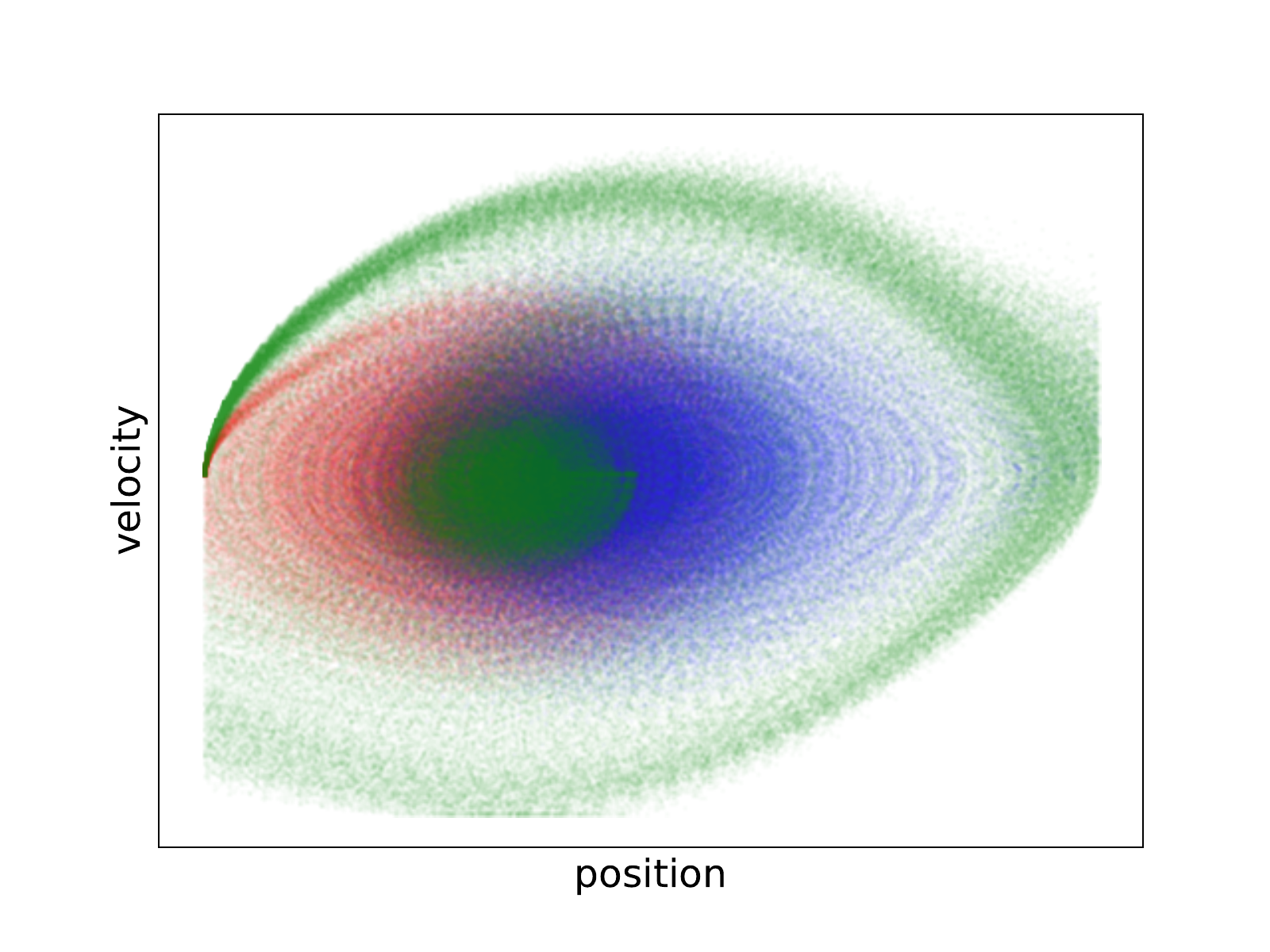}
  \caption{
  State visitations of different policies in \alg's policy cover on MountainCar.}
\label{fig:visitations}
\end{figure}

We further evaluated \alg on a continuous control task which requires exploration: continuous control MountainCar from OpenAI Gym \cite{brockman2016openai}.
Note here actions are continuous in $[-1, 1]$ and incur a small negative reward. Since the agent only receives a large reward $(+100)$ if it reaches the top of the hill, a locally optimal policy is to do nothing and avoid the action cost (e.g., PPO never escapes this local optimality in our experiments).  Results for PPO, PPO+RND and \alg are shown in Figure \ref{fig:reward_free_results}(right). The PPO agent quickly learns the locally optimal policy of doing nothing. 
The PPO+RND agent exhibits wide variability across seeds: some seeds solve the task while others not. 
The \alg agent consistently discovers a good policy across all seeds. In Figure~\ref{fig:visitations}, we show the traces of policies in the policy cover constructed by \alg.

%% file: conclusion.tex
\section{Discussion}

This work proposes a new policy gradient algorithm for balancing the
exploration-exploitation tradeoff in RL, which enjoys provable sample
efficiency guarantees in the linear and kernelized settings. Our
experiments provide evidence that the algorithm can be combined with
neural policy optimization methods and be effective in practice. An interesting direction for future work
would be to combine our approach with unsupervised feature learning
methods such as autoencoders~\cite{JarrettKRL09, DAE} or
noise-contrastive estimation~\cite{pmlr-v9-gutmann10a, CPC} in rich
observation settings to learn a good feature representation.

\section*{Acknowledgement}
The authors would like to thank Andrea Zanette and Ching-An Cheng for carefully reviewing the proofs, and Akshay Krishnamurthy for helpful discussions. 

%% file: appendix_two_loop.tex

\newcommand{\wt}{\widetilde}
\newcommand{\wh}{\widehat}
\newcommand{\calM}{\mathcal{M}}

\section{NPG Analysis (\pref{alg:npg})}
\label{app:npg_analysis}

In this section, we analyze \pref{alg:npg} for a particular episode $n$.


In order to carry out our analysis, we first set up some auxiliary MDPs which are needed in our analysis. Throughout this section, we focus on episode $n$.

\subsection{Set up of Augmented MDPs}
\label{app:setup_mdps}



Denote
\begin{equation}
\mathcal{K}^n := \left\{\sa:  \phi\sa^{\top} \left(\Sigma^n_{\mix}\right)^{-1}\phi\sa \leq \beta  \right\}.
\label{eqn:known_set}
\end{equation}

That is, $\mathcal{K}^n$ contains state-action pairs that obtain positive reward bonuses. We abuse notation a bit by denoting $s\in\mathcal{K}^n$ if and only if $\sa\in\calK^n$ for all $a\in\Acal$.

We also add an extra action denoted as $a^{\dagger}$ in $\mathcal{M}^n$. For any $s\not\in\calK^n$, we add $a^\dagger$ to the set of available actions one could take at $s$. We set rewards and transitions as follows:
\begin{align}
r^n\sa = r\sa + b^n\sa + \one\{a = a^\dagger\}; \quad P^n(\cdot | s,a) = P(\cdot | s,a), \forall \sa, \quad P^n(s|s,a^\dagger) = 1,
\label{eq:constructed_mdp}
\end{align} where $r(s,a^\dagger) = b^n(s,a^\dagger) = 0$ for any $s$.

Note that at this point, we have three different kinds of MDPs that we will cross during the analysis:
\begin{enumerate}
\item the original MDP $\calM$---the one that \alg is ultimately optimizing;
\item the MDP with reward bonus $b^n\sa$---the one is optimized by NPG in each episode $n$ \emph{in the algorithm}, which we denote as $\calM_{b^n} = \{P, r\sa + b^n\sa\}$ with $P$ and $r$ being the transition and reward from $\calM$;  \label{item:mdp_2}
\item the MDP $\calM^n$ that is constructed in Eq.~\pref{eq:constructed_mdp} which is \emph{only used in analysis but not in algorithm}. \label{item:mdp_3}
\end{enumerate}

The relationship between $\calM_{b^n}$ (item \pref{item:mdp_2}) and $\calM^n$ (item \pref{item:mdp_3}) is that NPG \pref{alg:npg} runs on $\calM_{b^n}$ (NPG is not even aware of the existence of $\calM^n$) but we use $\calM^n$ to analyze the performance of NPG below.

\paragraph{Additional Notations.} We are going to focus on a fixed comparator policy $\widetilde\pi \in \Pi$. We denote $\widetilde\pi^n$ as the policy such that $\widetilde{\pi}(\cdot |s) = \widetilde\pi^n(\cdot |s)$ for $s\in\calK^n$, and $\widetilde\pi^n(a^\dagger |s ) = 1$ for $s\not\in\calK^n$.  This means that the comparator policy $\wt\pi^n$ will self-loop in a state $s\not\in\calK^n$ and collect maximum rewards.
We denote $\widetilde{d}_{\calM_n}$ as the state-action distribution of $\widetilde{\pi}^n$ under $\mathcal{M}^n$, and $V^{\pi}_{\calM^n}, Q^\pi_{\calM^n}$, and $A^\pi_{\calM^n}$ as the value, Q, and advantage functions of $\pi$ under $\calM^n$.  We also denote $Q^{\pi}_{b^n}\sa$  in short of $Q^{\pi}(s,a ; r + b^n)$, similarly $A^{\pi}_{b^n}\sa$ in short of $A^{\pi}(s,a; r + b^n)$, and $V^\pi_{b^n}(s) $ in short of $V^{\pi}(s; r + b^n)$.

\begin{remark} \label{remark:relationship_two_mdps} Note that policies used in the algorithm do not pick $a^\dagger$ (i.e., algorithms does not even aware of $\Mcal^n$).  Hence for any policy $\pi$ that we would encounter during learning, we have $V^{\pi}_{\calM^n}(s) = V^{\pi}_{b^n}(s)$ for all $s$, $Q^{\pi}_{\calM^n}\sa = Q^{\pi}_{b^n}\sa$ and $A^{\pi}_{\calM^n}\sa = A^{\pi}_{b^n}\sa$ for all $s$  with $a\neq a^\dagger$. This fact is important as our algorithm is running on $\Mcal_{b^n}$ while the performance progress of the algorithm is tracked under $\Mcal^n$.
\end{remark}

\subsection{Performance of NPG (\pref{alg:npg}) on the Augmented MDP $\mathcal{M}^n$}

In this section, we focus on analyzing the performance of NPG (\pref{alg:npg}) on a specific episode $n$. Specifically we leverage the Mirror Descent analysis similar to~\citet{agarwal2019optimality} to show that regret between the sequence of learned policies $\{\pi^t\}_{t=1}^{T}$ and the comparator $\widetilde{\pi}^n$ on the constructed MDP $\mathcal{M}^n$.

Via performance difference lemma \cite{kakade2003sample}, we immediately have:
\begin{align*}
V^{\widetilde\pi^n}_{\mathcal{M}^n} - V^{\pi}_{\mathcal{M}^n}  = \frac{1}{1-\gamma} \EE_{\sa\sim \widetilde{d}_{\mathcal{M}^n}} \left[A^{\pi}_{\mathcal{M}^n}(s,a)\right].
\end{align*}
For notation simplicity below, given a policy $\pi$ and state $s$, we denote $\pi_s$ in short of $\pi(\cdot | s)$.

%

\begin{lemma}[NPG Convergence] %
Consider any episode $n$. Setting $\eta = \sqrt{\frac{\log(A)  }{ W^2 T}} $, assume NPG updates policy as:
\begin{align*}
\pi^{t+1}(\cdot | s) \propto \begin{cases}
\pi^t(\cdot | s) \exp\left( \eta \widehat{A}^t_{b^n}\sa \right), & s\in \Kcal^n, \\
\pi^t(\cdot |s), & \text{else},
\end{cases}
\end{align*} {with}  $\pi^0$ initialized as:
\begin{align*}
\pi^0(\cdot |s) = \begin{cases} \text{Uniform}(\Acal) & s\in\Kcal^n\\ \text{Uniform}(\{a\in\Acal: \sa\not\in\Kcal^n\}) & \text{else}.  \end{cases}
\end{align*}
Assume that $\sup_{s,a}\left\lvert\widehat{A}^t_{b^n}\sa\right\rvert \leq W$ and $\EE_{a'\sim \pi^t_s} \widehat{A}^t_{b^n}(s,a') = 0$ for all $t$. Then the NPG outputs a sequence of policies $\{\pi^t\}_{t=1}^T$ such that on $\mathcal{M}^n$, when comparing to $\wt\pi^n$: 
\begin{align*}
&\frac{1}{T}\sum_{t=1}^T \left(V^{\wt{\pi}^n}_{\mathcal{M}^n} - V^{t}_{\calM^n}  \right) = \frac{1}{T}  \sum_{t=1}^T \left(V^{\wt{\pi}^n}_{\mathcal{M}^n} - V^{t}_{{b^n}}  \right) \\
&\quad \leq  \frac{1}{1-\gamma}\left(2W\sqrt{\frac{\log(A)}{T}} + \frac{1}{T}\sum_{t=1}^T \left( \EE_{\sa\sim \wt{d}_{\mathcal{M}^n}}\left( A^t_{{b^n}}\sa - \wh{A}^t_{{b^n}}\sa \right) \one\{s\in\calK^n\} \right)\right),
\end{align*} 
\label{lem:npg_construction}
\end{lemma}
\begin{proof}
First consider any policy $\pi$ which uniformly picks actions among $\{a\in\Acal: \sa\not\in\calK^n\}$ at any $s\not\in\calK^n$. Via performance difference lemma, we have:
\begin{align*}
V^{\wt\pi^n}_{\calM^n} - V^{\pi}_{\calM^n} = \frac{1}{1-\gamma} \sum_{\sa} \wt{d}_{\calM^n}\sa A_{\calM^n}^{\pi}\sa \leq \frac{1}{1-\gamma} \sum_{\sa} \wt{d}_{\calM^n}\sa A_{\calM^n}^{\pi}\sa \one\{s \in \calK^n\},
\end{align*} where the last inequality comes from the fact that $A^{\pi}_{\calM^n}\sa\one\{s\not\in\calK^n\} \leq 0$. To see this, first note that that for any $s\not\in\calK^n$,  $\wt\pi^n$ will deterministically pick $a^\dagger$, and $Q^\pi_{\calM^n}(s,a^\dagger) = 1 + \gamma V^{\pi}_{\mathcal{M}^n}(s)$ as taking $a^\dagger$ leads the agent back to $s$. Second, since $\pi$ uniformly picks actions among $\{a: \sa\not\in\calK^n\}$, we have $V^{\pi}_{\mathcal{M}_n} \geq 1/(1-\gamma)$ as the reward bonus $b^n\sa$ on $\sa\not\in\calK^n$ is $1/(1-\gamma)$. Hence, we have \begin{align*}
A^{\pi}_{\calM^n}(s,a^\dagger) = Q^\pi_{\calM^n}(s,a^\dagger) - V^\pi_{\calM^n}(s) = 1 - (1-\gamma) V^{\pi}_{\calM^n}(s) \leq 0, \quad \forall s\not\in\calK^n.
\end{align*} %

Recall \pref{alg:npg}, $\pi^t$ chooses actions uniformly randomly among $\{a: \sa\not\in\calK^n\}$ for $s\not\in\calK^n$, thus we have:
\begin{align*}
(1-\gamma)\left(V^{\wt\pi^n}_{\calM^n} - V^{t}_{\calM^n}\right) \leq \sum_{\sa} \wt{d}_{\calM^n}\sa A_{\calM^n}^{t}\sa \one\{s \in \calK^n\} = \sum_{\sa} \wt{d}_{\calM^n}\sa A_{b^n}^{t}\sa \one\{s \in \calK^n\},
\end{align*} where the last equation uses the fact that $A^t_{b^n}(s,a) = A^t_{\Mcal^n}(s,a)$ for $a\neq a^\dagger$ and the fact that for $s\in\Kcal^n$, $\wt\pi^n$ never picks $a^\dagger$ (i.e., $\wt{d}_{\Mcal^n}(s,a^\dagger) = 0$ for $s\in\Kcal^n$). 

Recall the update rule of NPG,
\begin{align*}
{\pi}^{t+1}(\cdot |s) \propto \pi^t(\cdot | s) \exp\left(\eta \left(\widehat{A}^t_{b^n}(s,\cdot)\right)\one\{s\in\calK^n\}   \right), \forall s,
\end{align*} where the centered feature is defined as $\bar\phi^t\sa = \phi\sa - \EE_{a'\sim \pi^t(\cdot |s)} \phi(s,a') $. %
This is equivalent to updating $s\in \Kcal^n$ while holding $\pi(\cdot |s)$ fixed for $s\not\in\Kcal^n$, i.e.,
\begin{align*}
\pi^{t+1}(\cdot | s) \propto \begin{cases}
\pi^t(\cdot | s) \exp\left( \eta\widehat{A}^t_{b^n}(s,\cdot) \right), & s\in\Kcal^n, \\
\pi^t(\cdot | s), & else.
\end{cases}
\end{align*}
Now let us focus on any $s\in\Kcal^n$.
Denote the normalizer $z^t = \sum_{a}\pi^t(a | s) \exp\left(\eta \widehat{A}^t_{b^n}(s,a) \right) $. We have that:
\begin{align*}
\kl(\wt\pi^n_s, \pi^{t+1}_s) - \kl(\wt\pi^n_s, \pi^t_s) = \EE_{a\sim \wt\pi^n_s } \left[ -\eta \wh{A}^t_{b^n}\sa + \log(z^t) \right],
\end{align*}
where we use $\pi_s$ as a shorthand for the vector of probabilities $\pi(\cdot|s)$ over actions, given the state $s$. For $\log(z^t)$, using the assumption that $\eta \leq 1/W$, we have that $\eta \wh{A}^t_{b^n}\sa \leq 1$, which allows us to use the inequality $\exp(x) \leq 1 + x + x^2$ for any $x\leq 1$ and leads to the following inequality:
\begin{align*}
&\log(z^t)  = \log\left( \sum_{a} \pi^t(a|s) \exp(\eta\widehat{A}^t_{b^n}\sa) \right) \\
&\leq \log\left( \sum_{a}\pi^t(a|s) \left( 1 + \eta\wh{A}^t_{b^n}\sa   + \eta^2 \left(\widehat{A}^t_{b^n}\sa \right)^2  \right) \right) \\
&   = \log\left( 1 + \eta^2 W^2 \right) \leq \eta^2 W^2,
\end{align*} where we use the fact that $\sum_a \pi^t(a|s) \widehat{A}^t_{b^n}\sa = 0$.

Hence, for $s\in\Kcal^n$ we have:
\begin{align*}
\kl(\wt\pi^n_s, \pi^{t+1}_s) - \kl(\wt\pi^n_s, \wt\pi^t_s) \leq -\eta \EE_{a\sim \wt\pi^n_s} \wh{A}^t_{b^n} + \eta^2 W^2.
\end{align*}  Adding terms across rounds, and using the telescoping sum, we get:
\begin{align*}
\sum_{t=1}^T \EE_{a\sim \wt\pi^n_s}\wh{A}^t_{b^n}\sa \leq \frac{1}{\eta}\kl(\wt\pi^n_s, \pi^1_s) + \eta T W^2 \leq \frac{\log(A)}{\eta} + \eta T W^2, \quad \forall s\in \Kcal^n.
\end{align*} 
Add $\EE_{s\sim \wt{d}_{\calM^n}}$, we have:
\begin{align*}
\sum_{t=1}^T \EE_{\sa\sim \wt{d}_{\calM^n}} \left[\wh{A}^t_{b^n}\sa\one\{s\in\calK^n\}\right] \leq \frac{\log(A)}{\eta} + \eta T W^2 \leq 2W \sqrt{\log(A) T }.
\end{align*}
Hence, for regret on $\mathcal{M}_n$, we have:
\begin{align*}
&\sum_{t=1}^T \left( V^{\wt\pi^n}_{\calM^n} - V_{\calM^n}^{t} \right)\\
& \leq \sum_{t=1}^T \EE_{\sa\sim \wt{d}_{\calM^n}} \left[\wh{A}^t_{b^n}\sa \one\{s\in\calK^n\} \right]+ \sum_{t=1}^T \left( \EE_{\sa\sim \wt{d}_{\calM^n}}\left( A^t_{b^n}\sa - \wh{A}^t_{b^n}\sa \right)\one\{s\in\calK^n\} \right)\\
& \leq 2W\sqrt{\log(A) T} + \sum_{t=1}^T \left( \EE_{\sa\sim \wt{d}_{\calM^n}}\left( A^t_{b^n}\sa - \wh{A}^t_{b^n}\sa \right)\one\{s\in\calK^n\} \right).
\end{align*}
Now using the fact that $\pi^t$ never picks $a^\dagger$, we have $V^{t}_{\calM^n} = V^{t}_{{b^n}}$. 
This concludes the proof.
\end{proof}

Note that the second term of the RHS of the inequality in the above lemma measures the  average estimation error of $\wh{A}^t_{{b^n}}$.  Below, for \alg's analysis, we bound the critic prediction error under $d^\star$.

\section{Relationship between $\Mcal^n$ and $\Mcal$}


We need the following lemma to relate the probability of a known state being visited by $\widetilde{\pi}^n$ under $\calM^n$ and the probability of the same state being visited by $\widetilde{\pi}$ under $\calM_{b^n}$.  Note that intuitively as $\widetilde{\pi}^n$ always picks $a^\dagger$ outside $\calK^n$, it should have smaller probability of visiting the states inside $\calK^n$ (once $\widetilde{\pi}^n$ escapes, it will be absorbed and will never return back to $\calK^n$). Also recall that $\Mcal_{b^n}$ and $\Mcal$ share the same underlying transition dynamics. So for any policy, we simply have $d^{\pi}_{\Mcal_{b^n}} = d^{\pi}$.

The following lemma formally states this.
\begin{lemma}
\label{lem:prob_absorb}
Consider any state $s\in\calK^n$, we have:
\begin{align*}
\widetilde{d}_{\Mcal^n}(s,a) \leq d^{\widetilde{\pi}}\sa, \forall a \in \Acal,
\end{align*} where recall $\wt{d}_{\Mcal^n}$ is the state-action distribution of $\wt\pi^n$ under $\Mcal^n$.
\end{lemma}
\begin{proof}We prove by induction. Recall $\widetilde{d}_{\calM^n}$ is the state-action distribution of $\widetilde\pi^n$ under $\calM^n$, and $d^{\widetilde\pi}$ is the state-action distribution of $\widetilde\pi$ under both $\calM_{b^n}$ and $\calM$ as they share the same dynamics.

Starting at $h = 0$, we have:
\begin{align*}
\widetilde{d}_{\Mcal^n,0}(s_0, a) = d^{\widetilde{\pi}}_{0}(s_0,a),
\end{align*} as $s_0$ is fixed and $s_0\in\calK^n$, and $\widetilde{\pi}^n(\cdot | s_0) = \widetilde\pi(\cdot | s_0)$.

Now assume that at time step $h$, we have that for all $s\in \calK^n$, we have:
\begin{align*}
\widetilde{d}_{\Mcal^n, h}(s,a) \leq d^{\widetilde{\pi}}_{h}\sa, \forall a \in \Acal.
\end{align*} Now we proceed to prove that this holds for $h+1$. By definition, we have that for $s\in\calK^n$,
\begin{align*}
&\widetilde{d}_{\Mcal^n, h+1 }(s) = \sum_{s',a'} \widetilde{d}_{\Mcal^n, h}(s',a') P_{\calM^n}(s | s',a') \\
&= \sum_{s',a'} \one\{s'\in\calK^n\} \widetilde{d}_{\Mcal^n, h}(s',a') P_{\calM^n}(s | s',a') = \sum_{s',a'} \one\{s'\in\calK^n\} \widetilde{d}_{\Mcal^n, h}(s',a') P(s | s',a')
\end{align*} as if $s'\not\in\calK^n$, $\widetilde\pi^n$ will deterministically pick $a^\dagger$ (i.e., $a' = a^\dagger$) and $P_{\calM^n}(s | s' ,a^\dagger) = 0$.

On the other hand, for $d^{\widetilde{\pi}}_{h+1}\sa$, we have that for $s\in\calK^n$,
\begin{align*}
&d^{\widetilde{\pi}}_{h+1}\sa = \sum_{s',a'}d^{\widetilde{\pi}}_{h}(s',a') P(s|s',a') \\
&=\sum_{s',a'}\one\{s'\in\calK^n\} d^{\widetilde{\pi}}_{h}(s',a') P(s|s',a') + \sum_{s',a'}\one\{s\not\in\calK^n\}d^{\widetilde{\pi}}_{h}(s',a') P(s|s',a') \\
& \geq \sum_{s',a'}\one\{s'\in\calK^n\} d^{\widetilde{\pi}}_{h}(s',a') P(s|s',a') \\
& \geq \sum_{s',a'}\one\{s'\in\calK^n\} \widetilde{d}_{\calM^n,h}(s',a') P(s|s',a') = \widetilde{d}_{\Mcal^n, h+1}(s).
\end{align*} Using the fact that $\widetilde\pi^n(\cdot|s) = \widetilde\pi(\cdot | s)$ for $s\in\calK^n$, we conclude that the inductive hypothesis holds at $h+1$ as well. Thus it holds for all $h$.  Using the definition of average state-action distribution, we conclude the proof.
\end{proof}

We now establish a standard simulation lemma-style result to link the performance of policies on $\Mcal^n$ to the performance on the real MDP $\Mcal$, before bounding the error in the lemma using a linear bandits potential function argument as sketched above. These arguments allow us to translate the error bounds from Appendix~\ref{app:npg_analysis} from the augmented MDP $\mathcal{M}^n$ to the actual MDP $\calM$.

\begin{lemma}[Policy Performances on $\calM^n$, $\calM_{b^n}$ $\calM$]
At each episode $n$, denote $\{\pi^t\}_{t=1}^T$ as the sequence of policies generated from NPG in that episode. we have that for $\wt\pi^n$ and $\pi^t$ for any $t\in [T]$:
\begin{align*}
& V^{\wt\pi^n}_{\calM^n} \geq V^{\wt\pi}_{\calM},\\
& V^{t}_{\calM} \geq V^{t}_{{b^n}} - \frac{1}{1-\gamma}\left(\sum_{\sa\not\in\calK^n} d^t\sa\right) .
\end{align*}
\label{lem:perf_absorb}
\end{lemma}
\begin{proof}
Note that when running $\wt\pi^n$ under $\calM^n$, once $\wt\pi^n$ visits $s\not\in\calK^n$, it will be absorbed into $s$ and keeps looping there and receiving the maximum reward $1$.  Note that $\wt\pi$ receives reward no more than $1$ and in $\calM$ we do not have reward bonus.

Recall that $\pi^t$ never takes $a^\dagger$. Hence $d^t\sa = d^{t}_{\calM_{b^n}}\sa$ for all $\sa$.  Recall that the reward bonus is defined as $\frac{1}{1-\gamma}\one\{\sa\not\in\calK^n\}$. Using the definition of $b^n\sa$ concludes the proof.
\end{proof}

The lemma below relates the escaping probability to an elliptical potential function and quantifies the progress made by the algorithm by the maximum information gain quantity. 

\begin{lemma}[Potential Function Argument]
Consider the sequence of policies $\{\pi^n\}_{n=1}^N$ generated from \pref{alg:epoc}. We have:
\begin{align*}
\sum_{n=0}^{N-1} V^{\pi^{n+1}} \geq \sum_{n=0}^{N-1} V^{\pi^{n+1}}_{b^n} -  \frac{2\mathcal{I}_{N}(\lambda)}{ \beta (1-\gamma) } 
\end{align*}
\label{lem:potential_argument_n}
\end{lemma}
\begin{proof}
Denote the eigen-decomposition of $\Sigma_{\mix}^n$ as $U\Lambda U^{\top}$ and $\Sigma^n = \EE_{\sa\sim d^n}\phi\phi^{\top}$. We have:
\begin{align*}
&\tr\left( \Sigma^{n+1} \left(\Sigma_{\mix}^n\right)^{-1} \right) = \EE_{\sa\sim d^{n+1}}\tr\left( \phi\sa\phi\sa^{\top}\left( \Sigma_\mix^n \right)^{-1} \right)\\
& = \EE_{\sa\sim d^{n+1}}\phi\sa^{\top} \left(\Sigma_{\mix}^n\right)^{-1} \phi\sa \\
& \geq \EE_{\sa\sim d^{n+1}} \left[ \one\{\sa\not\in\calK^n\}\phi\sa^{\top} \left(\Sigma_\mix^n \right)^{-1}\phi\sa \right] \geq \beta \EE_{\sa\sim d^{n+1}} \one\{\sa\not\in\calK^n\}
\end{align*}together with \pref{lem:perf_absorb}, which implies that
\begin{align*}
V^{\pi^{n+1}}_{b^n} - V^{\pi^{n+1}} \leq \frac{\tr\left( \Sigma^{n+1} \left(\Sigma^n_{\mix} \right)^{-1} \right)}{\beta (1-\gamma)} .
\end{align*}

Now call~\pref{lem:trace_tele}, we have:
\begin{align*}
\sum_{n=0}^N \left(V^{\pi^{n+1}}_{b^n} - V^{\pi^{n+1}}\right) \leq \frac{2\log(\det\left( \Sigma_{\mix}^N \right) / \det(\lambda I))}{\beta(1-\gamma)}  \leq \frac{ 2 \mathcal{I}_N(\lambda) }{ \beta(1-\gamma)} 
\end{align*} where we use the definition of information gain $\mathcal{I}_{N}(\lambda)$.
\end{proof}

\section{Analysis of \alg for the Agnostic Setting (\pref{thm:agnostic})}
\label{app:rmaxpg_sample}

In this section, we analyze the performance of \alg using the NPG results we derived from the previous section. We begin with an assumption and a theorem statement which is the most general sample complexity result for \alg and from which all the statements of Section~\ref{sec:analysis} follow.

We first formally state the assumption of transfer bias $\varepsilon_{bias}$ which we have used as the condition in NPG analysis in \pref{lem:variance_bias_n}.



The following theorem states the detailed sample complexity of \alg (a detailed version of \pref{thm:agnostic}).
\begin{theorem}[Main Result: Sample Complexity of \alg] Fix $\delta\in (0,1/2)$ and $\epsilon\in (0, \frac{1}{1-\gamma})$. Setting hyperparameters as follows:
\begin{align*}
&T = \frac{4W^2 \log(A)}{ (1-\gamma)^2 \epsilon^2}, \quad \lambda = 1, \quad \beta = \frac{\epsilon^2(1-\gamma)^2}{4W^2}, \quad  N \geq \frac{4W^2 \mathcal{I}_N(1)}{ (1-\gamma)^3 \epsilon^3 },\\ 
& M =  \frac{ 144 W^4 \mathcal{I}_N(1)^2 \ln(NT/\delta )}{\epsilon^6(1-\gamma)^{10}}, \quad 
K = 32 N^2 \log\left(\frac{N\wh{d}}{\delta}\right),
\end{align*} Under \pref{ass:transfer_bias}, with probability at least $1-2\delta$, we have:
\begin{align*}
\max_{n\in [N]} V^{\pi^n} \geq V^{\wt\pi} - \frac{2\sqrt{A\varepsilon_{bias}}}{1-\gamma} - 4\epsilon,
\end{align*} for any comparator $\wt\pi\in \Pi_{linear}$, with at most total number of samples:
\begin{align*}
\frac{c \nu W^8 \mathcal{I}_N(1)^3 \ln(A)}{\epsilon^{11}(1-\gamma)^{15}},
\end{align*} where $c $ is a universal constant, and $\nu$ contains only log terms:
\begin{align*}
\nu & = \ln\left( \frac{4\widehat{d} W^2 \mathcal{I}_N(1) }{(1-\gamma)^3 \epsilon^3 \delta} \right) + \ln\left( \frac{16 W^4 \ln(A) \mathcal{I}_N(1)}{ \epsilon^5(1-\gamma)^5 \delta} \right).
\end{align*}
\label{thm:detailed_bound_rmax_pg}
\end{theorem}

\begin{remark}Note that in the above theorem, we require that the number of iterations $N$ to satisfy the constraint $N \geq 4W^2 \mathcal{I}_N(1) / ((1-\gamma)^3 \epsilon^3)$. The specific $N$ thus depends on the form of the maximum information gain $\mathcal{I}_N(1)$. For instance, when $\phi(s,a) \in \mathbb{R}^d$ with $\|\phi\|_2 \leq 1$, we have $\mathcal{I}_N(1) \leq d\log(N + 1)$. Hence setting $N \geq \frac{8 W^2 d}{ (1-\gamma)^3\epsilon^3 } \ln\left( \frac{4 W^2 d}{ (1-\gamma)^3\epsilon^3 } \right)$ suffices.  Another example is when $\phi$ lives in an RKHS with RBF kernel. In this case, we have $\mathcal{I}_N(1) = O( \log(N)^{d_{s,a}}  )$ (\cite{srinivas2010gaussian}), where $d_{s,a}$ stands for the dimension of the concatenated vector of state and action.  In this case, we can set $N =  O\left( \frac{ W^2 }{(1-\gamma)^3 \epsilon^3} \left(\ln\left( \frac{ W^2 }{(1-\gamma)^3 \epsilon^3} \right)  \right)^{d_{s,a}}  \right)$.\label{remark:kernel_discussion}
\end{remark}

In the rest of this section, we prove the theorem. Given the analysis of \pref{app:npg_analysis}, proving the theorem requires the following steps at a high-level:
\begin{enumerate}
\item Bounding the number of outer iterations $N$ in order to obtain a desired accuracy $\epsilon$. Intuitively, this requires showing that the probability with which we can reach an \emph{unknown state} with a positive reward bonus is appropriately small. We carry out this bounding by using arguments from the analysis of linear bandits~\citep{dani2008stochastic}. At a high-level, if there is a good probability of reaching unknown states, then NPG finds them based on our previous analysis as these states carry a high reward. But every time we find such states, the covariance matrix of the resulting policy contains directions not visited by the previous cover with a large probability (or else the quadratic form defining the unknown states would be small). In a $d$-dimensional linear space, the number of times we can keep finding significantly new directions is roughly $O(d)$ (or more precisely based on the intrinsic dimension), which allows us to bound the number of required outer episodes.
\item Bounding the prediction error of the critic in Lemma~\ref{lem:npg_construction}. This can be done by a standard regression analysis and we use a specific result for stochastic gradient descent to fit the critic.
\item Errors from empirical covariance matrices instead of their population counterparts have to be accounted for as well, and this is done by using standard inequalities on matrix concentration~\citep{tropp2015introduction}.
\end{enumerate}

\subsection{Proof of \pref{thm:detailed_bound_rmax_pg}}

We recall that we perform linear regression from $\phi\sa$ to $Q^{\pi}_{b^n}\sa - b^n\sa$, and set $\widehat{A}^t_{b^n}\sa$ as
\begin{align*}
& \widehat{A}^t_{b^n}\sa = \left(b^n\sa + \theta^t\cdot \phi\sa\right) - \EE_{a' \sim \pi^t_s}[ b^n(s,a') + \theta^t \cdot \phi(s,a')] \\
 & := \bar{b}^{n,t}\sa + \theta^t \cdot \bar\phi^t\sa,
 \end{align*} where for notation simplicitly, we denote centered bonus $\bar{b}^{n,t}\sa = b^n\sa - \EE_{a'\sim \pi^t_s} b^n(s,a')$, and centered feature $\bar\phi^t\sa = \phi\sa - \EE_{a'\sim \pi^t_s} \phi(s,a')$.

\begin{lemma}[Variance and Bias Tradeoff] \label{lem:variance_bias_n}Assume that at episode $n$ we have $\phi\sa^{\top}\left(\Sigma_{\mix}^{n}\right)^{-1}\phi\sa \leq \beta$ for $\sa\in\calK^n$. At iteration $t$ inside episode $n$, let us denote a best on-policy fit as $\theta^t_{\star} \in \argmin_{\|\theta\|\leq W} \EE_{\sa\sim \rho_{\mix}^n}\left( (Q^{t}_{b^n}\sa -b^n\sa) -  \theta\cdot \phi\sa \right)^2$. Assume the following condition is true for all $t\in [T]$:
\begin{align*}
L\left(\theta^t ; \rho^n_{\mix}, Q^{t}_{b^n} - b^n \right) \leq \min_{\theta:\|\theta\|\leq W} L\left(\theta ; \rho^n_{\mix}, Q^{t}_{b^n} - b^n \right) + \epsilon_{stat}, 
\end{align*} where $\varepsilon_{stat} \in\mathbb{R}^+$. Then under \pref{ass:transfer_bias} (with $\wt\pi$ as the comparator policy here), we have that for all $t\in [T]$:
\begin{align*}
\EE_{\sa\sim \wt{d}_{\calM^n}}\left( A^t_{{b^n}}\sa - \wh{A}^t_{{b^n}}\sa \right) \one\{s\in\calK^n\}\leq 2\sqrt{A\varepsilon_{bias}} + 2\sqrt{ \beta  \lambda W^2  } + 2\sqrt{ \beta n \varepsilon_{stat} }.
\end{align*}
\end{lemma}
\begin{proof}
We first show that under condition 1 above, $\EE_{\sa\sim \rho^n_{\mix}}\left( \theta^t_{\star} \phi\sa - \theta^t\phi\sa \right)^2$ is bounded by $\varepsilon_{stat}$.
\begin{align*}
&\EE_{\sa\sim \rho^n_\mix} \left( Q^t_{{b^n}} \sa - b^n\sa - \theta^t\cdot \phi\sa \right)^2 - \EE_{\sa\sim \rho^n_\mix}\left( Q^t_{{b^n}}\sa - b^n\sa - \theta^t_\star\cdot \phi\sa \right)^2\\
& = \EE_{\sa\sim \rho^n_\mix}\left( \theta^t_\star\cdot \phi\sa - \theta^t\cdot\phi\sa \right)^2 + 2\EE_{\sa\sim \rho^n_{\mix}} \left(  Q^t_{{b^n}}\sa -b^n\sa - \theta^t_\star\cdot\phi\sa \right)\phi\sa^{\top}\left( \theta^t_\star - \theta^t  \right).
\end{align*} Note that $\theta_\star$ is one of the minimizers of the constrained square loss $\EE_{\sa\sim \rho^n_\mix}(Q^t_{{b^n}}\sa-b^n\sa - \theta\cdot\phi\sa)^2$, via first-order optimality, we have:
\begin{align*}
\EE_{\sa\sim \rho^n_\mix}\left( Q^t_{{b^n}}\sa - b^n\sa - \theta^t_\star\cdot\phi\sa \right) (-\phi\sa^{\top})\left( \theta - \theta^t_\star \right)\geq 0,
\end{align*} for any $\|\theta\|\leq W$, which implies that:
\begin{align*}
&\EE_{\sa\sim \rho^n_\mix}\left( \theta^t_\star\cdot \phi\sa - \theta^t \cdot\phi\sa \right)^2 \\
&\leq \EE_{\sa\sim \rho^n_\mix} \left( Q^t_{{b^n}} \sa - b^n\sa  - \theta^t\cdot \phi\sa \right)^2 - \EE_{\sa\sim \rho^n_\mix}\left( Q^t_{{b^n}}\sa -b^n\sa - \theta^t_\star\cdot \phi\sa \right)^2 \leq \varepsilon_{stat}.
\end{align*}
Recall that $\Sigma^n_{\mix} = \sum_{i=1}^n \EE_{\sa\sim d^n}\phi\sa\phi\sa^{\top} + \lambda \mathbf{I} = n \left( \EE_{\sa\sim \rho^n_\mix}\phi\sa\phi\sa^{\top} + \lambda/n \mathbf{I}\right)$ . Denote $\bar{\Sigma}_{\mix}^n = \Sigma_{\mix}^n / n$. We have:
\begin{align*}
\left(\theta^t_\star - \theta^t \right)^{\top} \left( \EE_{\sa\sim \rho^n_\mix}\phi\sa\phi\sa^{\top} + \lambda/n \mathbf{I} \right) (\theta^t_\star - \theta^t) \leq \varepsilon_{stat} + \frac{\lambda}{n} W^2.
\end{align*} Hence for any $\sa\in\calK^n$, we must have the following point-wise estimation error:
\begin{align}
\left\lvert \phi\sa^{\top}\left( \theta^t_\star - \theta^t \right)\right\rvert \leq \| \phi\sa \|_{(\Sigma_{\mix}^n)^{-1}} \| \theta^t_\star - \theta^t \|_{\Sigma_\mix^n} \leq  \sqrt{ \beta n\varepsilon_{stat} + \beta \lambda W^2  }. \label{eq:point_wise_est}
\end{align}
Now we bound $\EE_{\sa\sim \wt{d}_{\calM^n}}\left( A^t_{{b^n}}\sa - \wh{A}^t_{{b^n}}\sa\right)\one\{s\in\calK^n\}$ as follows. 
\begin{align*}
&\EE_{\sa\sim \wt{d}_{\calM_{b^n}}}\left( A^t_{{b^n}}\sa - \wh{A}^t_{{b^n}}\sa\right)\one\{s\in\calK^n\} \\
&= \underbrace{\EE_{\sa\sim \wt{d}_{\calM^n}}\left( A^t_{{b^n}}\sa - (  \bar{b}^{n,t}\sa +   \theta^t_\star\cdot\bar\phi^t \sa ) \right)\one\{s\in\calK^n\}}_{\text{term A}} \\
& \qquad + \underbrace{\EE_{\sa\sim \wt{d}_{\calM^n}}\left( (\bar{b}^{n,t}\sa +  \theta^t_\star \cdot\bar\phi^t\sa) - ( \bar{b}^{n,t}\sa +   \theta^t\cdot\bar\phi^t\sa) \right)\one\{s\in\calK^n\}}_{\text{term B}}.
\end{align*}
We first bound term A above. 
\begin{align*}
&\EE_{\sa\sim \wt{d}_{\calM^n}}\left( A^t_{{b^n}}\sa - \bar{b}^{n,t}\sa -  \theta^t_\star\cdot\bar\phi^t \sa \right)\one\{s\in\calK^n\} \\
& = \EE_{\sa\sim \wt{d}_{\calM^n}}\left(Q^t_{{b^n}}\sa - b^n\sa - \theta^t_\star\cdot \phi\sa \right)\one\{s\in\calK^n\} \\
&\qquad + \EE_{s\sim \wt{d}_{\calM^n},a\sim \pi^t_s}\left(-Q^t_{{b^n}}\sa + b^n\sa + \theta^t_\star\cdot \phi\sa \right)\one\{s\in\calK^n\} \\
& \leq \sqrt{ \EE_{\sa\sim \wt{d}_{\calM^n}}\left(Q^t_{{b^n}}\sa -b^n\sa - \theta^t_\star\cdot \phi\sa \right)^2\one\{s\in\calK^n\} } \\
& \qquad + \sqrt{\EE_{s\sim \wt{d}_{\calM^n},a\sim \pi^t_s}\left(Q^t_{{b^n}}\sa - b^n\sa - \theta^t_\star\cdot \phi\sa \right)^2\one\{s\in\calK^n\} }\\
& \leq \sqrt{ \EE_{\sa\sim {d}^{\widetilde\pi}}\left(Q^t_{b^n}\sa -b^n\sa - \theta^t_\star\cdot \phi\sa \right)^2\one\{s\in\calK^n\} } \\
& \qquad + \sqrt{\EE_{s\sim d^{\widetilde\pi},a\sim \pi^t_s}\left(Q^t_{b^n}\sa - b^n\sa - \theta^t_\star\cdot \phi\sa \right)^2\one\{s\in\calK^n\} }\\
& \leq \sqrt{ \EE_{\sa\sim {d}^{\widetilde\pi}}\left(Q^t_{b^n}\sa -b^n\sa- \theta^t_\star\cdot \phi\sa \right)^2 } + \sqrt{\EE_{s\sim d^{\widetilde\pi},a\sim \pi^t_s}\left(Q^t_{b^n}\sa - b^n\sa - \theta^t_\star\cdot \phi\sa \right)^2 }\\
& \leq 2 \sqrt{A \epsilon_{bias}},
\end{align*} where the first inequality uses CS inequality, the second inequality uses \pref{lem:prob_absorb} for $s\in\calK^n$, and the last inequality uses the change of variable over action distributions and \pref{ass:transfer_bias}.

Now we bound term B above. We have:
\begin{align*}
&\EE_{\sa\sim \wt{d}_{\calM^n}}\left( \theta^t_\star \cdot\bar\phi^t\sa - \theta^t\cdot\bar\phi^t\sa \right)\one\{s\in\calK^n\} \\
& = \EE_{\sa\sim \wt{d}_{\calM^n}}\left( \theta^t_\star \phi\sa - \theta^t\cdot\phi\sa \right)\one\{s\in\calK^n\} \\
& \qquad - \EE_{s\sim \wt{d}_{\calM^n}}\EE_{a\sim \pi^t} \one\{s\in\calK^n\}\left( \theta^t_\star \phi\sa - \theta^t\cdot\phi\sa \right) \leq 2\sqrt{ \beta  \lambda W^2  } + 2\sqrt{ \beta n \epsilon_{stat} },
\end{align*} where we use the point-wise estimation guarantee from inequality \pref{eq:point_wise_est}.

Combine term A and term B together, we conclude the proof.
\end{proof}

Combine the the above lemma and \pref{lem:npg_construction}, we can see that as long as the on-policy critic achieves small statistical error (i.e., $\epsilon_{stat}$ is small), and our features $\phi\sa$ are sufficient to represent Q functions in a linear form (i.e., $\epsilon_{bias}$ is small), then we can guarantee  inside episode $n$, NPG succeeds by finding a policy that has low regret with respect to the comparator $\widetilde{\pi}^n$:
\begin{align}
\max_{t\in[T]} V^{t}_{b^n} \geq V^{\wt\pi^n}_{\calM^n} - \frac{1}{1-\gamma}\left( 2W\sqrt{\frac{\log(A)}{T}} + 2\sqrt{A\varepsilon_{bias}} + 2\sqrt{\beta \lambda W^2} + 2\sqrt{\beta n \varepsilon_{stat}} \right).
\label{eq:npg_perf}
\end{align}

 The term that contains $\epsilon_{stat}$ comes from the statistical error induced from constrained linear regression.  Note that in general, $\epsilon_{stat}$ decays in the rate of $O(1/\sqrt{M})$ with $M$ being the total number of data samples used for linear regression (\pref{line:learn_critic} in \pref{alg:npg}), and $\epsilon_{stat}$ usually does not polynomially depend on dimension of $\phi\sa$ explicitly. See Lemma~\ref{lemma:least_square_dim_free} for an example where linear regression is solved via stochastic gradient descent. 


%


Using \pref{lem:potential_argument_n}, now we can transfer the regret we computed under the sequence of models $\{\calM_{b^n}\}$ to regret under $\calM$. Recall that $V^{\pi}$ denotes $V^{\pi}(s_0)$ and $V^n$ is in short of $V^{\pi^n}$.

\begin{lemma}Assume the condition in~\pref{lem:variance_bias_n} and \pref{ass:transfer_bias} hold. For the sequence of policies $\{\pi^n\}_{n=1}^N$, we have:
\begin{align*}
\max_{n\in [N]} V^{n} \geq V^{\wt\pi}  -  \frac{1}{1-\gamma}\left(2W\sqrt{\frac{\log(A)}{T}} + 2 \sqrt{A\varepsilon_{bias}} + 2\sqrt{ \beta  \lambda W^2  } + 2\sqrt{\beta N \varepsilon_{stat}} +\frac{2\mathcal{I}_{N}(\lambda)}{N\beta}\right).
\end{align*}
\label{lem:regret_rmax_pg}
\end{lemma}
\begin{proof}
First combine \pref{lem:npg_construction} and \pref{lem:variance_bias_n}, we have:
\begin{align*}
\frac{1}{N} \sum_{n=0}^{N-1} V^{n+1}_{b^n} \geq \frac{1}{N}\sum_{n=0}^{N-1} V^{\wt\pi^n}_{\calM^n} -  \frac{1}{1-\gamma}\left(2W\sqrt{\frac{\log(A)}{T}} +   2\sqrt{A\varepsilon_{bias}} + 2\sqrt{ \beta  \lambda W^2  } + 2\sqrt{ \beta N \varepsilon_{stat} } \right).
\end{align*}

Use~\pref{lem:perf_absorb} and \pref{lem:potential_argument_n}, we have:
\begin{align*}
\frac{1}{N} \sum_{n=1}^{N} V^{n} \geq V^{\wt\pi} - \frac{1}{1-\gamma} \left(  2W\sqrt{\frac{\log(A)}{T}} +  2\sqrt{A \varepsilon_{bias}} + 2\sqrt{ \beta  \lambda W^2  } + 2\sqrt{ \beta N \varepsilon_{stat} } + \frac{\mathcal{I}_N(\lambda)}{N\beta} \right) ,
\end{align*} which concludes the proof.
\end{proof}

The following theorem shows that setting hyperparameters properly, we can guarantee to learn a near optimal policy.
\begin{theorem}Assume the conditions in~\pref{lem:variance_bias_n} and \pref{ass:transfer_bias} hold. Fix $\epsilon \in (0,1/(4(1-\gamma)))$. Setting hyperparameters as follows:
\begin{align*}
&T = \frac{4W^2 \log(A)}{ (1-\gamma)^2 \epsilon^2}, \quad \lambda = 1, \quad \beta = \frac{\epsilon^2(1-\gamma)^2}{4W^2}, \\
& N \geq \frac{4W^2 \mathcal{I}_N(1)}{ (1-\gamma)^3 \epsilon^3 }, \quad \epsilon_{stat} =  \frac{ \epsilon^3 (1-\gamma)^3 }{ 4 \mathcal{I}_{N}(1)}, 
\end{align*} we have:
\begin{align*}
\max_{n\in [N]} V^n \geq V^{\wt\pi} - \frac{2\sqrt{A\varepsilon_{bias}}}{1-\gamma} - 4\epsilon.
\end{align*}
\label{thm:regret_stat_error}
\end{theorem}
\begin{proof}
The theorem can be easily verified by substituting the values of hyperparameters into~\pref{lem:regret_rmax_pg}.
\end{proof}

The above theorem indicates that we need to control the $\varepsilon_{stat}$ statistical error from linear regression to be small in the order of $\wt{O}\left(\epsilon{^3}(1-\gamma)^3\right)$.
Recall that $M$ is the total number of samples we used for each linear regression. If $\varepsilon_{stat} = \wt{O}\left(1/\sqrt{M}\right)$, then we roughly will need $M$ to be in the order of $\wt{\Omega}\left( 1/(\epsilon^6 (1-\gamma)^6) \right)$.  Note that we do on-policy fit in each iteration $t$ inside each episode $n$, thus we will pay total number of samples in the order of $M\times (TN)$.

Another source of samples is the samples used to estimate covariance matrices $\Sigma^n$. As $\phi$ could be infinite dimensional, we need matrix concentration without explicit dependency on dimension of $\phi$.  Leveraging matrix Bernstein inequality with matrix intrinsic dimension, the following lemma shows concentration results of $\wh\Sigma^n$ on $\Sigma^n$, and of $\widehat\Sigma_\mix^n$ on $\Sigma_\mix^n$.

\begin{lemma}[Estimating Covariance Matrices]
Set $\lambda = 1$. Define $\wh{d}$ as:
\begin{align*}
\wh{d} = \max_{\pi} \tr\left( \Sigma^{\pi} \right)/\| \Sigma^{\pi}\|,
\end{align*} i.e., the maximum intrinsic dimension of the covariance matrix from a mixture policy.
For $K \geq 32 N^2 \ln\left(\wh{d}N/\delta\right)$ (a parameter in \pref{alg:epoc}), with probability at least $1-\delta$, for any $n\in [N]$, we have for all $x$ with $\|x\| \leq 1$,
\begin{align*}
(1/2)x^{\top}\left( \Sigma_\mix^n  \right)^{-1} x\leq x^{\top}\left( \wh\Sigma_\mix^n \right)^{-1} x \leq 2 x^{\top}\left( \Sigma_\mix^n \right)^{-1}x
\end{align*}
\label{lem:concentration_cov}
\end{lemma}
\begin{proof}
The proof of the above lemma is simply \pref{lem:inverse_covariance}.

\end{proof}

%


We are now ready to prove Theorem~\ref{thm:detailed_bound_rmax_pg}.

\begin{proof}[Proof of Theorem~\ref{thm:detailed_bound_rmax_pg}]
Assume the event in~\pref{lem:concentration_cov} holds.  In this case, we have for all $n\in [N]$,
\begin{align*}
(1/2) x^{\top} \left( {\Sigma}^n_\mix \right)^{-1} x \leq x^{\top}\left({\wh\Sigma}^n_\mix\right)^{-1} x \leq 2 x^{\top} \left({\Sigma}^n_\mix\right)^{-1} x,
\end{align*} for all $\|x\|\leq 1$ and the total number of samples used for estimating covariance matrices is:
\begin{align}
\label{eq:source_1}
&N\times K = N \times \left( 32 N^2 \ln\left( \wh{d} N /\delta \right)\right)  = 32 N^3 \ln\left( \wh{d} N /\delta \right) \\
&= \frac{(32\times 64)\mathcal{I}_N(1)^3 W^6}{ \epsilon^9(1-\gamma)^{9}} \ln\left( \frac{4\widehat{d} W^2 \mathcal{I}_N(1) }{(1-\gamma)^3 \epsilon^3 \delta} \right)
 =  \frac{c_1 \nu_1 \mathcal{I}_N(1)^3 W^6}{\epsilon^9(1-\gamma)^9},
\end{align} where $c_1$ is a  constant and $\nu_2$ contains log-terms $\nu_1:= \ln\left( \frac{4\widehat{d} W^2 \mathcal{I}_N(1) }{(1-\gamma)^3 \epsilon^3 \delta} \right)$

Since we set known state-action pair as $\phi\sa^{\top} \left( \wh{\Sigma}_\mix^n \right)^{-1}\phi\sa \leq \beta$,  then we must have that for any $\sa\in\calK^n$, we have:
\begin{align*}
\phi\sa^{\top} \left( \Sigma^n_{\mix} \right)^{-1} \phi\sa \leq 2\beta,
\end{align*} and any $\sa\not\in\Kcal^n$, we have:
\begin{align*}
\phi\sa^{\top} \left( \Sigma^n_{\mix} \right)^{-1} \phi\sa \geq \frac{1}{2} \beta.
\end{align*}
This allows us to call \pref{thm:regret_stat_error}. From \pref{thm:regret_stat_error}, we know that we need to set $M$ (number of samples for linear regression) large enough such that
\begin{align*}
\varepsilon_{stat} =  \frac{ \epsilon^3 (1-\gamma)^3 }{ 4 \mathcal{I}_{N}(1)}, 
\end{align*} Using \pref{lem:sgd_dim_free} for linear regression, we know that with probability at least $1-\delta$, for any $n,t$, $\varepsilon_{stat}$ scales in the order of:
\begin{align*}
\varepsilon_{stat} = \sqrt{\frac{ 9W^4 \log(N T/\delta) }{(1-\gamma)^4 M }},
\end{align*} where we have taken union bound over all episodes $n\in [N]$ and all iterations $t\in [T]$. Now solve for $M$, we have:
\begin{align*}
M =  \frac{ 144 W^4 \mathcal{I}_N(1)^2 \ln(NT/\delta )}{\epsilon^6(1-\gamma)^{10}}
\end{align*}
Considering every episode $n\in [N]$ and every iteration $t\in [T]$, we have the total number of samples needed for NPG is:
\begin{align*}
NT \cdot M  & = \frac{ 4 W^2 \mathcal{I}_N(1) }{\epsilon^3 (1-\gamma)^3} \times \frac{4W^2 \log(A)}{(1-\gamma)^2 \epsilon^2} \times  \frac{ 144 W^4 \mathcal{I}_N(1)^2 \ln(NT/\delta )}{\epsilon^6(1-\gamma)^{10}}\\
& =  \frac{c_2 W^8 \mathcal{I}_N(1)^{3} \ln(A) }{ \epsilon^{11}(1-\gamma)^{15} } \cdot \ln\left( \frac{16 W^4 \ln(A) \mathcal{I}_N(1)}{ \epsilon^5(1-\gamma)^5 \delta} \right) 
 = \frac{c_2 \nu_2 W^8 \mathcal{I}_N(1)^{3} \ln(A) }{ \epsilon^{11}(1-\gamma)^{15} },
\end{align*} where $c_2$ is a positive universal constant, and $\nu_2$ only contains log terms:
\begin{align*}
\nu_2 = \ln\left( \frac{16 W^4 \ln(A) \mathcal{I}_N(1)}{ \epsilon^5(1-\gamma)^5 \delta} \right).
\end{align*}
Combine two sources of samples, we have that the total number of samples is bounded as:
\begin{align*}
\frac{c_2 \nu_2 W^8 \mathcal{I}_N(1)^{3} \ln(A) }{ \epsilon^{11}(1-\gamma)^{15} } +  \frac{c_1 \nu_2 \mathcal{I}_N(1)^3 W^6}{ \epsilon^9(1-\gamma)^{9}},
\end{align*} 
This concludes the proof.
\end{proof}

\section{Analysis of \alg for Linear MDPs (\pref{thm:linear_mdp})}
\label{app:app_to_linear_mdp}


For linear MDP $\calM$, recall that we assume the following parameters' norms are bounded:
\begin{align*}
\|v^{\top}\mu \| \leq \xi \in\mathbb{R}^+, \quad \|\theta\| \leq \omega\in\mathbb{R}^+, \quad \forall v, \text{ s.t. } \|v\|_{\infty} \leq 1.
 \end{align*}With these bounds on linear MDP's parameters, we can show that for any policy $\pi$, we have $Q^{\pi}\sa = w^{\pi}\cdot \phi\sa$, with $\|w^{\pi}\|\leq \omega + V_{\max} \xi $, where $V_{\max} = \max_{\pi,s}V^{\pi}(s)$ is the maximum possible expected total value ($V_{\max}$ is at most $r_{\max}/(1-\gamma)$ with $r_{\max}$ being the maximum possible immediate reward).

At every episode $n$, recall that NPG is optimizing the MDP $\calM_{b^n} = \{P, r\sa + b^n\sa\}$ with $P, r$ being the true transition and reward of $\calM$ which is linear under $\phi\sa$.

Due to the reward bonus $b^n\sa$ in $\calM_{b^n}$, $\calM_{b^n}$ is not necessarily a linear MDP under $\phi\sa$ ($P$ is still linear under $\phi$ but $r\sa+b^n\sa$ it not linear anymore). Here we leverage an observation that we know $b^n\sa$ (as we designed it), and $Q^{\pi}(s,a;r+b^n) - b^n\sa$ is linear with respect to $\phi$ for any $\sa\in \Scal\times\Acal$. The following claim state this observation formally.



\begin{claim}[Linear Property of $(Q^{\pi}(s,a;r+b^n) - b^n\sa)$ under $\phi$]
Consider any policy $\pi$ and any reward bonus $b^n\sa\in [0,1/(1-\gamma)]$. We have that:
\begin{align*}
Q^{\pi}(s,a ; r+b^n) - b^n\sa = w\cdot \phi\sa, \forall s,a.
\end{align*} Further we have $\| w\| \leq \omega + \xi / (1-\gamma)^2$.
\label{claim:linear_property}
\end{claim}
\begin{proof}
By definition of $Q$-function, we have:
\begin{align*}
Q^{\pi}(s,a ; r + b^n) & = r(s,a) + b^n\sa + \gamma \phi\sa^{\top} \sum_{s'}\mu(s') V^{\pi}(s'; r+b^n) \\
& = b^n\sa + \phi\sa\cdot \left( \theta + \gamma \mu^{\top} V^{\pi}(\cdot; r+b^n) \right) := b^n\sa + \phi\sa\cdot w,
\end{align*} where note that $w$ is independent of $\sa$.
Rearrange terms, we prove that $Q^{\pi}(s,a; r+b^n) - b^n\sa = w\cdot \phi\sa$.

Further, using the norm bounds we have for $\theta$ and $\mu$, and the fact that $\|V^{\pi}(\cdot; r+b^n)\|_{\infty} \leq 1/(1-\gamma)^2$, we conclude the proof.
\end{proof}

The above claim supports our specific choice of critic $\widehat{A}^t_{b^n}$ in the algorithm, where we recall that we perform linear regression from $\phi\sa$ to $Q^{\pi}_{b^n}\sa - b^n\sa$, and set $\widehat{A}^t_{b^n}\sa$ as
\begin{align*}
& \widehat{A}^t_{b^n}\sa = \left(b^n\sa + \theta^t\cdot \phi\sa\right) - \EE_{a' \sim \pi^t_s}[ b^n(s,a') + \theta^t \cdot \phi(s,a')] \\
 & := \bar{b}^{n,t}\sa + \theta^t \cdot \bar\phi^t\sa,
 \end{align*} where $\bar{b}^{n,t}\sa = b^n\sa - \EE_{a'\sim \pi^t_s} b^n(s,a')$, and $\bar\phi^t\sa = \phi\sa - \EE_{a'\sim \pi^t_s} \phi(s,a')$.

We now prove \pref{thm:linear_mdp} by showing that $\epsilon_{bias}$ is zero. 
\begin{lemma} Consider \pref{ass:transfer_bias}. For any episode $n$, iteration $t$, we have $\epsilon_{bias} = 0$.
\end{lemma}
\begin{proof}
At iteration $t$, denote $\theta^t_\star$ as the linear parameterization of $Q^{\pi^t}_{b^n}\sa - b^n\sa$, i.e., $\theta^t_\star\cdot\phi\sa =Q^{\pi^t}_{b^n}\sa - b^n\sa$ (see \pref{claim:linear_property} for the existence of $\theta^t_\star$). we know that $\theta^t_\star \in \argmin_{\theta:\|\theta\| \leq W} L( \theta; \rho^n_\mix, Q^t_{b^n} - b^n ) $, as $L(\theta^t_\star; \rho^n_\mix, Q^t_{b^n} - b^n) = 0$. This indicates that $\theta^t_\star$ is one of the best on-policy fit. Now when transfer $\theta^\star_t$ to a different distribution $d^{\pi^\star}\text{Unif}_{\Acal}$, we simply have:
\begin{align*}
\EE_{\sa\sim d^{\pi^\star}\text{Unif}_{\Acal}} \left( \theta^t_\star \cdot \phi\sa - \left( Q^t_{b^n}\sa - b^n\sa \right) \right)^2 = 0.
\end{align*} This concludes the proof. 
\end{proof}

We can now conclude the proof of \pref{thm:linear_mdp} by invoking \pref{thm:agnostic} with $\epsilon_{bias} = 0$.  \qed

\input{app_state_act_aggregation_2}

%% file: app_state_act_aggregation_2.tex
\section{Analysis of \alg for State-Aggregation (\pref{thm:state_aggregation})}
\label{app:state_agg}

In this section, we analyze \pref{thm:state_aggregation} for state-aggregation.  Similar to the analysis for linear MDP, we provide a variance bias tradeoff lemma that is analogous to \pref{lem:variance_bias_n}. However, unlike linear MDP, here due to model-misspecification from state-aggregation, the transfer error $\epsilon_{bias}$ will not be zero. But we will show that the transfer error is related to a term that is an expected model-misspecification averaged over a fixed comparator's state distribution

First recall the definition of state aggregation $\phi:\mathcal{S}\times\mathcal{A} \to\mathcal{Z}$. We abuse the notation a bit, and denote $\phi\sa = \one\{\phi\sa = z\} \in\mathbb{R}^{|\mathcal{Z}|}$, i.e., the feature vector $\phi$ indicates which $z$ the state action pair $\sa$ is mapped to.  The following claim reasons the approximation of $Q$ values under state aggregation.
\begin{claim} \label{claim:state_agg} Consider any MDP with transition $P$ and reward $r$. Denote aggregation error $\epsilon_{z}$ as:
\begin{align*}
\max\left\{ \| P(\cdot | s,a) - P(\cdot | s', a') \|_1, \lvert r(s,a) - r(s',a') \rvert \right\} \leq \epsilon_{z}, \forall \sa,(s',a'), \text{ s.t., } \phi\sa = \phi(s',a') = z.
\end{align*} Then, for any policy $\pi$, $\sa,(s',a'), z$, such that $\phi(s,a) = \phi(s',a') = z$, we have:
\begin{align*}
\left\lvert Q^{\pi}(s, a) - Q^{\pi}(s', a')\right\rvert \leq  \frac{r_{\max} \epsilon_{z} }{1-\gamma},
\end{align*} where $r\sa\in [0, r_{\max}]$ for $r_{\max}\in \mathbb{R}^+$.
\end{claim}
\begin{proof}
Starting from the definition of $Q^{\pi}$, we have:
\begin{align*}
&\left\lvert Q^{\pi}(s, a) - Q^{\pi}(s', a')\right\rvert = \lvert r(s,a) - r(s' ,a') \rvert + \gamma \lvert \EE_{x'\sim P_{s,a}} V^{\pi}(s') - \EE_{x'\sim P_{s',a'}} V^{\pi}(s')  \rvert \\
& \leq \epsilon_{z} + \frac{r_{\max}\gamma}{1-\gamma} \left\| P_{s,a} - P_{s',a'} \right\|_1 \leq \frac{r_{\max}\epsilon_{z}}{1-\gamma},
\end{align*} where we use the assumption that $\phi(s,a) = \phi(s',a') = z$, and the fact that value function $\|V\|_{\infty} \leq r_{\max}/(1-\gamma)$ as $r\sa\in [0, r_{\max}]$.
\end{proof}


Now we state the bias and variance tradeoff lemma for state aggregation.
\begin{lemma}[Bias and Variance Tradeoff for State Aggregation] Set $W:=\sqrt{|\mathcal{Z}|}/(1-\gamma)^2$.
Consider any episode $n$. Assume that we have $\phi\sa^{\top}\left(\Sigma_{\mix}^{n}\right)^{-1}\phi\sa \leq \beta\in\mathbb{R}^+$ for $\sa\in\calK^n$, and the following condition is true for all $t\in \{0,\dots, T-1\}$:
\begin{align*}
L^t( \theta^t; \rho^n_\mix, Q^t_{b^n} - b^n ) \leq \min_{\theta:\|\theta\|\leq W} L^t(\theta; \rho^n_{\mix}, Q^t_{b^n} - b^n) + \epsilon_{stat} \in \mathbb{R}^+.
\end{align*}

We have that for all $t \in \{0, \dots, T-1\}$ at episode $n$:
\begin{align*}
& \EE_{\sa\sim \wt{d}_{\calM^n}}\left( A^t_{b^n}\sa - \wh{A}^t_{b^n}\sa \right) \one\{s\in\calK^n\} \\
& \quad \leq 2\sqrt{ \beta  \lambda W^2  } + 2\sqrt{ \beta n \epsilon_{stat} }  + \frac{2 \EE_{\sa\sim d^{\wt\pi} } \max_{a} \left[\epsilon_{\phi(s,a)}\right]  }{(1-\gamma)^2}.
\end{align*}\label{lem:variance_bias_state_agg}
\end{lemma}
Note that comparing to \pref{lem:variance_bias_n}, the above lemma replaces $\sqrt{A\epsilon_{bias}}$ by the average model-misspecification $ \frac{ \EE_{\sa\sim d^{\wt\pi} } \max_{a} \left[\epsilon_{\phi(s,a)}\right]  }{(1-\gamma)^2}$.
\begin{proof}
We first compute one of the minimizers of $L^t(\theta; \rho^n_\mix, Q^t_{b^n} - b^n)$. Recall the definition of $L^t(\theta; \rho^n_\mix, Q^t_{b^n} - b^n)$, we have:
\begin{align*}
&\EE_{ \sa\sim \rho^n_{\mix}} \left(\theta\cdot\phi\sa - Q^{\pi^t}_{b^n}\sa + b^n\sa \right)^2 \\
&= \EE_{\sa\sim \rho^n_{\mix}} \sum_{z} \one\{\phi(s,a) = z\} \left( \theta_{z} - Q^{\pi^t}_{b^n}\sa + b^n\sa \right)^2,
\end{align*} which means that for $\theta^t_{\star}$, we have:
\begin{align*}
\sum_{s,a} \rho^t_{\mix}(s,a) \one\{\phi(s,a) = z\} \left( \theta_{z} - Q^{\pi^t}_{b^n}(s,a)+ b^n\sa \right) = 0,
\end{align*} which implies that $\theta^t_{\star, z} := \frac{  \sum_{s,a} \rho^n_{\mix}(s,a)\one\{\phi(s,a) = z\}  (Q^{\pi^t}_{b^n}(s,a) - b^n\sa )}{ \sum_{s,a}\rho^n_{\mix}(s,a)\one\{\phi(s,a) = z\} } $. Note that $|\theta^t_{\star,z}| \leq \frac{1}{(1-\gamma)^2}$, hence $\| \theta^t_\star \|_2 \leq \sqrt{|\mathcal{Z}|}/(1-\gamma)^2 := W$. 

Hence, for any $s'',a'' $ such that $\phi(s'',a'') = z$, we must have:
\begin{align*}
&\left\lvert \theta^t_{\star, z} - (Q^{\pi^t}_{b^n}(s'', a'') - b^n(s'',a'')) \right\rvert \\
& = \left\lvert  \frac{  \sum_{s,a} \rho^n_{\mix}(s,a)\one\{\phi(s,a) = z\}  (Q^{\pi^t}_{b^n}(s,a)-b^n\sa)}{ \sum_{s,a}\rho^n_{\mix}(s,a)\one\{\phi(s,a) = z\} }  - Q^{\pi^t}_{b^n}(s'',a'')  + b^n(s'',a'')\right\rvert\\
& = \left\lvert  \frac{\sum_{s,a} \rho^n_{\mix}(s,a)\one\{\phi(s,a) = z\} \left( Q^{\pi^t}_{b^n}(s,a) - Q^{\pi^t}_{b^n}(s'',a'') \right)}{ \sum_{s,a}\rho^n_{\mix}(s,a)\one\{\phi(s,a) = z\} }         \right\rvert
 \leq \frac{\epsilon_{z}}{ (1-\gamma)^2  },
\end{align*} where we use Claim~\ref{claim:state_agg}, and the fact that $r\sa+b^n\sa \in [0, 1/(1-\gamma)]$, and the fact that $b^n(s,a) = b^n(s'',a'')$ if $\phi(s,a) = \phi(s'',a'')$ as the bonus is defined under feature $\phi$. 

With $\theta^t_\star$ and its optimality condition for loss $L^t(\theta; \rho^n_\mix)$, we can prove the same point-wise estimation guarantee, i.e., for any $\sa\in\Kcal^n$, we have:
\begin{align*}
\left\lvert \phi\sa \cdot ( \theta^t - \theta^t_\star ) \right\rvert \leq \sqrt{ \beta n\epsilon_{stat} + \lambda W^2 }.
\end{align*}
Now we bound $\EE_{\sa\sim \wt{d}_{\calM^n}}\left( A^t_{{b^n}}\sa - \wh{A}^t_{{b^n}}\sa\right)\one\{s\in\calK^n\}$ as follows.
\begin{align*}
&\EE_{\sa\sim \wt{d}_{\calM_{b^n}}}\left( A^t_{{b^n}}\sa - \wh{A}^t_{{b^n}}\sa\right)\one\{s\in\calK^n\} \\
&= \underbrace{\EE_{\sa\sim \wt{d}_{\calM^n}}\left( A^t_{{b^n}}\sa - \bar{b}^{t,n}\sa - \theta^t_\star\cdot\bar\phi^t \sa \right)\one\{s\in\calK^n\}}_{\text{term A}} \\
& \qquad + \underbrace{\EE_{\sa\sim \wt{d}_{\calM^n}}\left( \theta^t_\star \cdot\bar\phi^t\sa - \theta^t\cdot\bar\phi^t\sa \right)\one\{s\in\calK^n\}}_{\text{term B}}.
\end{align*}

Again, for term B, we can use the point-wise estimation error to bound it as:
\begin{align*}
\text{term B} \leq 2\sqrt{\beta \lambda W^2} + 2\sqrt{\beta n \epsilon_{stat}}.
\end{align*}

For term A, we have:
\begin{align*}
&\EE_{\sa\sim \wt{d}_{\calM^n}}\left( A^t_{{b^n}}\sa - \bar{b}^{t,n}\sa - \theta^t_\star\cdot\bar\phi^t \sa \right)\one\{s\in\calK^n\} \\
& \leq  \EE_{\sa\sim \wt{d}_{\calM^n}}\left\lvert Q^t_{{b^n}}\sa - b^n\sa - \theta^t_\star\cdot \phi\sa \right\rvert \one\{s\in\calK^n\} \\
&\qquad + \EE_{s\sim \wt{d}_{\calM^n},a\sim \pi^t_s}\left\lvert -Q^t_{{b^n}}\sa + b^n\sa + \theta^t_\star\cdot \phi\sa \right\rvert \one\{s\in\calK^n\} \\
& \leq \EE_{\sa\sim d^{\wt\pi}} \left\lvert Q^t_{{b^n}}\sa - b^n\sa - \theta^t_\star\cdot \phi\sa \right\rvert + \EE_{s\sim d^{\wt\pi}, a\sim \pi^t_s} \left\lvert -Q^t_{{b^n}}\sa + b^n\sa + \theta^t_\star\cdot \phi\sa \right\rvert,
\end{align*} where last inequality uses \pref{lem:prob_absorb} for $s\in\calK^n$ to switch from $\wt{d}_{\Mcal^n}$ to $d^{\wt\pi}$---the state-action distribution of the comparator $\wt\pi$ in the real MDP $\Mcal$.

Note that for any $d\in\Scal\times\Acal$, we have:
\begin{align*}
&\EE_{\sa\sim d} \left\lvert Q^{t}_{b^n}\sa -b^n\sa - \theta^t_\star\cdot \phi\sa\right\rvert \nonumber \\
& \leq \sum_{z} \EE_{\sa\sim d} \one\{\phi(s,a) = z\} \left\lvert Q^{t}_{b^n}(s,a) - b^n\sa - \theta^t_{\star,z} \right\rvert \leq \EE_{(z)\sim d} \frac{\epsilon_{z}}{(1-\gamma)^2}  = \frac{ \EE_{\sa \sim d} \epsilon_{\phi(s,a)}  }{(1-\gamma)^2}.
\end{align*} 
With this, we have:
\begin{align*}
\text{term A} &  \leq  \EE_{\sa\sim {d}^{\wt\pi}} \left\lvert Q^{\pi^t}_{b^n}\sa - b^n\sa -\theta^n_\star\cdot \phi\sa  \right\rvert + \EE_{s\sim d^{\wt\pi}, a\sim \pi^t_s}  \left\lvert -  Q^{\pi^t}_{b^n}\sa  +b^n\sa + \theta^t_\star \cdot \phi\sa   \right\rvert  \\
& \leq \EE_{s\sim {d}^{\wt\pi}} \max_{a} \left\lvert Q^{\pi^n}_{\wt\Mcal}\sa - b^n\sa - \theta^n_\star\cdot \phi\sa  \right\rvert + \EE_{s\sim d^{\wt\pi}}\max_a  \left\lvert -  Q^{\pi^t}_{b^n}\sa  + b^n\sa + \theta^t_\star \cdot \phi\sa   \right\rvert \\
& \leq 2 \left( \EE_{s\sim d^{\wt\pi} } \max_{a} \left\lvert Q^{\pi^t}_{b^n}\sa - b^n\sa - \theta^t_\star\cdot \phi\sa \right\rvert  \right) 
  \leq \frac{2  \EE_{s \sim d^{\wt\pi} } \max_a \left[\epsilon_{\phi(s,a)}\right]  }{(1-\gamma)^2}
\end{align*}
Combine term A and term B, we conclude the proof.
\end{proof}

The rest of the proof of \pref{thm:state_aggregation} is almost identical to the proof of \pref{thm:detailed_bound_rmax_pg} with $\sqrt{A \epsilon_{bias}}$ in \pref{thm:detailed_bound_rmax_pg} being replaced by $\frac{2  \EE_{s \sim d^{\wt\pi} } \max_a \left[\epsilon_{\phi(s,a)}\right]  }{(1-\gamma)^2}$.
\qed

%% file: app_examples.tex
\newcommand{\calS}{\mathcal{S}}

\section{Analysis of \alg for the Partially Well-specified Models (Corollary~\pref{cora:agnostic})}
\label{app:examples}

\begin{proof}[Proof of Corollary \pref{cora:agnostic}]
The proof involves showing that the transfer error  is $0$.
Specifically, we will show the following:
consider any state-action distribution $\rho$, and any policy $\pi$, and any bonus function $b$ with bounded value on all $b\sa$, 
there exists $\theta_\star$ as one of the best on-policy fit,\\ i.e.,
$\theta_\star\in \arg\min_{\theta:\|\theta\|\leq W} \EE_{\sa\sim
  \rho}\left( \theta\cdot \phi\sa - (Q^{\pi}\sa - b\sa)  \right)^2 $,  such
that:
\begin{align*}
\EE_{\sa \sim d^\star}  \left(Q^{\pi}\sa - b\sa - \theta_\star\cdot\phi\sa\right)^2 = 0,
\end{align*}  i.e., the transfer error is zero.

Let us denote a minimizer of $\EE_{\sa\sim \rho}\left( \theta\cdot \phi\sa -b\sa - Q^{\pi}\sa  \right)^2 $ as $\widetilde{\theta}$. We can modify the first three bits of $\widetilde{\theta}$. We set $\widetilde{\theta}_1 = Q^{\pi}(s_0, L) - b(s_0,L) = 1/2 - b(s_0,L)$,  $\widetilde{\theta}_2 = Q^{\pi}(s_0, R) - b(s_0,R)$, and $\widetilde{\theta}_3 = Q^{\pi}(s_1, a) -b(s_1,a) =  - b(s_1,a)$ for any $a\in\{L, R\}$.  Denote this new vector as $\theta_\star$. Note that due to the construction of $\phi$ (the feature vectors associated with states inside the binary tree is orthogonal to the span of $e_1,e_2,e_3$),  $\theta_\star$ will not change any prediction error for states inside the binary tree under $(s_0, R)$ and will only bring down the prediction error for $(s_0, a)$ for $a\in\{L,R\}$, and $(s_1, a)$ for $a\in\{L,R\}$. Hence $\theta_\star$ is also the minimizer of $\EE_{\sa\sim \rho}\left( \theta\cdot \phi\sa - (Q^{\pi}\sa -b\sa)  \right)^2 $.

For $\theta_\star$, we have $\theta_\star\cdot\phi(s_0, a) = Q^{\pi}(s_0,a) - b(s_0,a)$ for $a\in\{L,R\}$,  and $\theta_\star\cdot \phi(s_1, a) = Q^{\pi}(s_1,a) - b(s_1,a)$ for $a\in\{L, R\}$, thus, we can verify that $Q^{\pi}(s_0, a)-b(s_0,a) = \theta_\star\cdot\phi(s_0, a)$ and $Q^{\pi}(s_1, a) -b(s_1,a)= \theta_\star\cdot\phi(s_1,a)$ for $a\in\{L, R\}$. Since $\pi^\star$ only visits $s_0$ and $s_1$, we can conclude that $\EE_{\sa \sim d^\star} \left( Q^{\pi}\sa - b\sa -  \theta_\star\cdot\phi\sa\right)^2 = 0$.

With $\varepsilon_{bias} = 0$, we can conclude the proof by recalling \pref{thm:detailed_bound_rmax_pg}.
\end{proof}

%% file: appendix_new.tex
\section{Auxiliary Lemmas}
\label{app:tech_lemmas}

\begin{lemma}[Dimension-free Least Square Guarantees]
\label{lemma:least_square_dim_free} Consider the following learning process. Initialize $\theta_1 =  \mathbf{0}$.  For $i = 1,\dots, N$, draw $x_i, y_i \sim \nu$, $y_i \in [0, H]$, $\|x_i\| \leq 1$;%
Set $\theta_{i+1} =\prod_{\Theta:=\{\theta:\|\theta\|\leq W\}} \left(\theta_{i} - \eta_i (\theta_i\cdot x_i  - y_i) x_i\right)$ with $\eta_i = (W^2)/((W+H)\sqrt{N})$. Set $\hat{\theta} = \frac{1}{N} \sum_{i=1}^{N} \theta_i$, we have that with probability at least $1-\delta$:
\begin{align*}
\EE_{x\sim \nu}\left[\left(\hat\theta\cdot x - \EE\left[y|x\right]\right)^2\right] \leq  \EE_{x\sim \nu}\left[\left( \theta^\star\cdot x - \EE\left[y|x\right] \right)^2\right] +\frac{R\sqrt{\ln(1/\delta)}}{\sqrt{N}},
\end{align*} with any $\theta^\star$ such that $\|\theta^\star\|\leq W$ and $R = 3(W^2 + WH)$ which is dimension free and only depends on the norms of the feature and $\theta^\star$ and the bound on $y$.
\label{lem:sgd_dim_free}
\end{lemma}
\begin{proof}
Note that we compute $\theta_i$ using Projected Online Gradient Descent \citep{zinkevich2003online} on the sequence of loss functions $(\theta\cdot x_i - y_i)^2$. Using the projected online gradient descent regret guarantee, we have that:
\begin{align*}
\sum_{i=1}^{N} (\theta_i\cdot x_i - y_i)^2 \leq \sum_{i=1}^{N}(\theta^\star\cdot x_i -y_i)^2 + \underbrace{W(W+H)}_{:=Q}\sqrt{N}.
\end{align*}
Denote random variable $z_i = (\theta_i\cdot x_i - y_i)^2 - (\theta^\star\cdot x_i - y_i)^2$. Denote $\EE_{i}$ as the expectation taken over the randomness at step $i$ conditioned on all history $t=1$ to $i-1$. Note that for $\EE_{i}[z_i]$, we have:
\begin{align*}
&\EE_{i} \left[ (\theta_i\cdot x - y)^2 - (\theta^\star\cdot x - y)^2 \right]\\
& = \EE_{i} \left[ (\theta_i\cdot x - \EE[y|x])^2\right] \\
& \qquad \qquad - \EE_{i}\left[2(\theta_i\cdot x - \EE[y|x])( \EE[y|x] - y ) - (\theta^\star\cdot x - \EE[y|x])^2 + 2(\theta^\star\cdot x - \EE[y|x])(\EE[y|x] - y) ) \right]\\
& =  \EE_{i}\left[ (\theta_i\cdot x - \EE[y|x])^2 - ( \theta^\star\cdot x - \EE[y|x])^2 \right],
\end{align*} where we use $\EE[\EE[y|x] - y] = 0$.
Also for $|z_i|$, we can show that for $|z_i|$ we have:
\begin{align*}
\left\lvert z_i\right\rvert = \left\lvert (\theta_i\cdot x_i - \theta^\star\cdot x_i)(\theta_i\cdot x_i +\theta^\star\cdot x_i - 2y_i) \right\rvert \leq W( 2W + 2H ) = 2W(W+H).
\end{align*}
Note that $z_i$ forms a Martingale difference sequence. Using Azuma-Hoeffding's inequality, we have that with probability at least $1-\delta$:
\begin{align*}
\left\lvert\sum_{i=1}^N  z_i - \sum_{i=1}^N \EE_{i}\left[ (\theta_i \cdot x - \EE[y|x])^2  - (\theta^\star \cdot x - \EE[y|x])^2\right]\right\rvert  \leq 2W(W+H) \sqrt{{\ln(1/\delta)}{N}},
\end{align*} which implies that:
\begin{align*}
&\sum_{i=1}^N \EE_{i}\left[ (\theta_i \cdot x - \EE[y|x])^2  - (\theta^\star \cdot x - \EE[y|x])^2\right] \leq \sum_{i=1}^N z_i + 2W(W+H) \sqrt{{\ln(1/\delta)}{N}} \\
&\leq 2W(W+H) \sqrt{{\ln(1/\delta)}{N}} + Q\sqrt{N}.
\end{align*}

Apply Jensen's inequality on the LHS of the above inequality, we have that:
\begin{align*}
\EE\left( \hat{\theta}\cdot x - \EE[y|x]\right)^2 \leq \EE\left(\theta^\star\cdot x - \EE[y|x]\right)^2  + (Q+2W(W+H)) \sqrt{\frac{\ln(1/\delta)}{N}}.
\end{align*}
\end{proof}

\begin{lemma}
\label{lem:trace_tele}
Consider the following process.  For $n = 1, \dots, N$, $M_n = M_{n-1} + \Sigma_{n}$ with $M_{0} = \lambda \mathbf{I}$ and $\Sigma_n$ being PSD matrix with eigenvalues upper bounded by $1$. We have that:
\begin{align*}
2 \log\det( M_N) - 2 \log\det( \lambda\mathbf{I})  \geq  \sum_{n=1}^N \trace\left( \Sigma_{i} M_{i-1}^{-1} \right).
\end{align*}
\end{lemma}
\begin{proof} Note that $M_0$ is PD, and since $\Sigma_n$ is PSD for all $n$, we must have $M_n$ being PD as well.

Using matrix inverse lemma, we have:
\begin{align*}
\det(M_{n+1}) = \det(M_n) \det( \mathbf{I} + M_n^{-1/2} \Sigma_{n+1} M_n^{-1/2}  ).
\end{align*} Add $\log$ on both sides of the above equality, we have:
\begin{align*}
&\log\det(M_{n+1}) = \log\det(M_n) + \log\det( I + M_n^{-1/2}\Sigma_{n+1} M_n^{-1/2}).%
\end{align*} Denote the eigenvalues of $M_n^{-1/2}\Sigma_{n+1} M_n^{-1/2}$ as $\sigma_1,\dots, \sigma_d$, we have:
\begin{align*}
&\log\det(M_{n+1}) = \log\det( M_n) + \sum_{i=1}^{d} \log \left( 1 + \sigma_i \right)  %
\end{align*} Note that $\sigma_i \leq 1$, and we have $\log(1+x) \geq x/2$ for $x\in [0,1]$. Hence, we have:
\begin{align*}
&\log\det(M_{n+1}) \geq \log\det( M_n) + \sum_{i=1}^d \sigma_i / 2 = \log\det(M_n) + \frac{1}{2} \trace\left( M_n^{-1/2}\Sigma_{n+1} M_n^{-1/2} \right) \\
&= \log\det(M_n) + \frac{1}{2}\trace\left( \Sigma_{n+1} M_n^{-1} \right),
\end{align*}where we use the fact that $\trace(A B) = \trace(BA)$ and the trace of PSD matrix is the sum of its eigenvalues.
Sum over from $n = 0$ to $N$ and cancel common terms, we  conclude the proof.
\end{proof}

\begin{lemma}[Covariance Matrix Concentration]
\label{lemma:covariance_concentration}
Given $\nu\in \Delta(\Scal\times\Acal)$ and $N$ i.i.d samples $\{s_i,a_i\}\sim \nu$. Denote $\Sigma = \EE_{\sa\sim \nu}\phi\sa\phi\sa^{\top}$ and $X_i = \phi(s_i,a_i)\phi(s_i,a_i)^{\top}$ and $X = \sum_{i=1}^N X_i$. Note that $N \Sigma = \EE[X] = \sum_{i=1}^N \EE[X_i]$. %
Then, with probability at least $1-\delta$, we have that:
\begin{align*}
\left\lvert  x^{\top} \left( \sum_{i=1}^N \phi(s_i,a_i)\phi(s_i,a_i)^{\top}/N  - \Sigma \right) x  \right\rvert \leq \frac{2 \ln(8\widetilde{d}/\delta)}{ 3N } + \sqrt{ \frac{ 2\ln(8\widetilde{d}/\delta) }{N} },
\end{align*} with $\widetilde{d} = \trace(\Sigma)/\|\Sigma\|$ being the intrinsic dimension of $\Sigma$.
\end{lemma}
\begin{proof}
Denote random matrix $X_i = \phi(s_i,a_i)\phi(s_i,a_i)^{\top} - \Sigma$. Note that the maximum eigenvalue of $X_i$ is upper bounded by 1. Also note that $\EE[X_i] = 0$ for all $i$.  Denote $V = \sum_{i=1}^N \EE[ X_i^2]$. For any $i$, consider $X_i^2$. Denote the eigendecomposition of $X_i$ as $U_i \Lambda_i U_i^{\top}$. We have $X_i^2 = U_i \Lambda_i^2 U_i^{\top}$. Note that the maximum absolute value of the eigenvalues of $X_i$ is bounded by 1. Hence the maximum eigenvalue of $X_i^2$ is bounded by 1 as well. Hence $\EE[X_i^2]$'s maximum eigenvalue is also upper bounded by $1$. This implies that $\| V \| \leq N$.

Now apply Matrix Bernstein inequality \citep{tropp2015introduction}, we have that for any $t \geq \sqrt{N} + 1/3$,
\begin{align*}
\Pr\left( \sigma_{\max}(\sum_{i=1}^N X_i) \geq t \right) \leq 4\widetilde{d} \exp\left( \frac{-t^2/2}{N + t/3} \right).
\end{align*}  Since $\sigma_{\max}\left( \sum_{i=1}^N X_i \right) = N \sigma_{\max}\left(\sum_{i=1}^N X_i / N\right)$, we get that:
\begin{align*}
\Pr\left( \sigma_{\max}\left(\sum_{i=1}^N X_i / N \right) \geq \epsilon \right) \leq 4\widetilde{d}\exp\left( \frac{-\epsilon^2 N / 2 }{ 1 + \epsilon / 3}  \right),
\end{align*} for any $\epsilon \geq \frac{1}{\sqrt{N}} + \frac{1}{3N}$.
Set $4d\exp(-\epsilon^2 N / (2(1+\epsilon/3))) = \delta$, we get:
\begin{align*}
\epsilon = \frac{2 \ln(4\widetilde{d}/\delta)}{ 3N } + \sqrt{ \frac{ 2\ln(4\widetilde{d}/\delta) }{N} },
\end{align*} which is trivially bigger than $1/\sqrt{N} + 1/(3N)$ as long as $d\geq 1$ and $\delta \leq 1$.  This concludes that with probability at least $1-\delta$, we have:
\begin{align*}
\sigma_{\max}\left(\sum_{i=1}^N \phi(s_i,a_i)\phi(s_i,a_i)^{\top}/ N - \Sigma \right) \leq \frac{2 \ln(4\widetilde{d}/\delta)}{ 3N } + \sqrt{ \frac{ 2\ln(4\widetilde{d}/\delta) }{N} }.
\end{align*}

We can repeat the same analysis for random matrices $\{X_i: = \Sigma - \phi(s_i,a_i)\phi(s_i,a_i)^{\top}\}$ and we can show that with probability at least $1-\delta$, we have:
\begin{align*}
\sigma_{\max}\left( \Sigma - \sum_{i=1}^N \phi(s_i,a_i)\phi(s_i,a_i)^{\top}/N \right) \leq \frac{2 \ln(4\widetilde{d}/\delta)}{ 3N } + \sqrt{ \frac{ 2\ln(4\widetilde{d}/\delta) }{N} }.
\end{align*}
Hence, with probability $1-\delta$, for any $x$, we have:
\begin{align*}
&x^{\top}\left(\Sigma - \sum_{i=1}^N \phi(s_i,a_i)\phi(s_i,a_i)^{\top}/N \right) x \leq \frac{2 \ln(8\widetilde{d}/\delta)}{ 3N } + \sqrt{ \frac{ 2\ln(8\widetilde{d}/\delta) }{N} }, \\
& x^{\top}\left( \sum_{i=1}^N \phi(s_i,a_i)\phi(s_i,a_i)^{\top}/N  - \Sigma \right)x \leq \frac{2 \ln(8\widetilde{d}/\delta)}{ 3N } + \sqrt{ \frac{ 2\ln(8\widetilde{d}/\delta) }{N} }.
\end{align*} This concludes the proof.
\end{proof}

\begin{lemma}[ Concentration with the Inverse of Covariance Matrix]
\label{lem:inverse_covariance}
Consider a fixed $N$. Given $N$ distributions $\nu_1,\dots, \nu_N$ with $\nu_i\in\Delta(\Scal\times\Acal)$, assume we draw $K$ i.i.d samples from $\nu_i$ and form $\widehat{\Sigma}^i = \sum_{j=1}^K \phi_j\phi_j^{\top}/ K$ for all $i$. Denote $\Sigma = \sum_{i=1}^N \EE_{\sa\sim \nu_i} \phi\sa\phi\sa^{\top} + \lambda I$ and $\widehat\Sigma = \sum_{i=1}^N \widehat{\Sigma}^i + \lambda I$ with $\lambda\in (0,1]$. Setting $K = 32 N^2 \log\left(8 N \widetilde{d}/\delta\right)/\lambda^2$, with probability at least $1-\delta$, we have:
\begin{align*}
\frac{1}{2} x^T \left({\Sigma} + \lambda I \right)^{-1} x \leq x^T \left(\widehat{\Sigma} + \lambda I \right)^{-1} x \leq 2 x^T \left({\Sigma} + \lambda I \right)^{-1} x,
\end{align*} for all $x$ with $\|x \|_2 \leq 1$.
\end{lemma}
\begin{proof}
Denote $\Sigma^i = \EE_{\sa\sim \nu_i} \phi(x_i,a_i)\phi(x_i,a_i)^{\top}$.
Denote $\eta(K) = \frac{2 \ln(8N \widetilde{d}/\delta)}{ 3K } + \sqrt{ \frac{ 2\ln(8 N \widetilde{d}/\delta) }{K} }$. 
From Lemma~\ref{lemma:covariance_concentration}, we know that with probability $1-\delta$, for all $i$, we have:
\begin{align*}
\Sigma^i + \eta(K)\mathbf{I}  + (\lambda/N) \mathbf{I} \succeq \widehat\Sigma^i + (\lambda/N)\mathbf{I} \succeq \Sigma^i - \eta(K)\mathbf{I} + (\lambda/N)\mathbf{I},
\end{align*} which implies that:
\begin{align*}
\Sigma + N\eta(K) \mathbf{I} + \lambda \mathbf{I} \geq \widehat{\Sigma} + \lambda \mathbf{I} \geq \Sigma - N\eta(K)\mathbf{I} + \lambda\mathbf{I},
\end{align*}which further implies that:
\begin{align*}
\left(\Sigma - N \eta(K)\mathbf{I}  + \lambda \mathbf{I} \right)^{-1} \succeq \left(\widehat\Sigma + \lambda\mathbf{I} \right)^{-1} \succeq \left(\Sigma + N \eta(K)\mathbf{I} + \lambda\mathbf{I}\right)^{-1},
\end{align*} under the condition that $N\eta(K) \leq \lambda$ which holds under the condition of $K$.
Let $U\Lambda U^{\top}$ be the eigendecomposition of $\Sigma$.
\begin{align*}
&x^{\top}\left( \widehat{\Sigma} + \lambda \mathbf{I} \right)^{-1} x - x^{\top} \left( \Sigma + \lambda I \right)^{-1}  \leq  x^{\top} \left(\left(\Sigma +(- N \eta(K)+ \lambda)\mathbf{I}\right)^{-1} - \left({\Sigma} + \lambda \mathbf{I}\right)^{-1}\right)x  \\
& = \sum_{i} \left( (\sigma_i+\lambda - N\eta(K))^{-1} - (\sigma_i + \lambda ))^{-1}  \right)(x\cdot u_i)^2
\end{align*}
Since $\sigma_i + \lambda \geq  2N \eta(K)$ as $\sigma_i\geq 0$ and $N\eta(K)\leq \lambda/2$, we have that $2(\sigma_i + \lambda - N\eta(K) )\geq \sigma_i + \lambda$, which implies that $(1/2) (\sigma_i + \lambda - K\eta(N))^{-1} \leq (\sigma_i + \lambda )^{-1}$. Hence, we have:
\begin{align*}
&x^{\top}\left( \widehat{\Sigma} + \lambda \mathbf{I} \right)^{-1} x - x^{\top} \left( \Sigma + \lambda I \right)^{-1} x \leq \sum_{i=1}  (u_i\cdot x)^2 (\sigma_i + \lambda)^{-1} =  x^{\top}(\Sigma + \lambda\mathbf{I})^{-1} x.
\end{align*} The analysis for the other direction is similar.
This concludes the proof.
\end{proof}

%% file: appendix_exp.tex
\section{Experimental Details}
\label{app:exp}

\subsection{Algorithm Implementation}

We implemented two versions of the algorithm: one with a reward bonus which is added to the environment reward (shown in Algorithm \ref{alg:rmaxpg_implemented_reward_bonus}), and one which performs reward-free exploration, optionally followed by reward-based exploitation using the policy cover as a start distribution (shown in Algorithm \ref{alg:rmaxpg_implemented_reward_free}).

Both of these use NPG as a subroutine, which performs policy optimization using the restart distribution induced by a policy mixture $\Pi_\mathrm{mix}$. The implementation of NPG is described in Algorithm \ref{alg:npg_implemented}.
We sample states from the restart distribution by randomly sampling a roll-in policy from the cover and a horizon length $h'$, and following the sampled policy for $h'$ steps.
Rewards gathered during these roll-in steps are not used for optimization.
With probability $\epsilon$, a random action is taken at the beginning of the rollout.
We then roll out using the current policy being optimized, and use the rewards gathered for optimization. The policy parameters can be updated using any policy gradient method, we used PPO \citep{schulman2017proximal} in our experiments.

For all experiments, we optimized the policy mixture weights $\alpha_1,..., \alpha_n$ at each episode using $2000$ steps of gradient descent, using an Adam optimizer and a learning rate of $0.001$.
All implementations are done in PyTorch \citep{PyTorch}, and build on the codebase of \citep{deeprl}.
Experiments were run on a GPU cluster which consisted of a mix of 1080Ti, TitanV, K40, P100 and V100 GPUs.

\begin{algorithm}[h!]
\begin{algorithmic}[1]
\State \textbf{Require}: kernel function $\phi: \mathcal{S} \times \mathcal{A} \rightarrow \mathbb{R}^d$
\State Initialize policy $\pi_1$ randomly
\State Initialize policy mixture $\Pi_\mathrm{mix} \leftarrow \{\pi_1\}$
\State Initialize episode buffer: $\mathcal{R} \leftarrow \emptyset$
\For{episode $n = 1, \dots K$}
\For{trajectory $k = 1, \dots K$}
\State Gather trajectory $\tau_k = \{s_h^{(k)}, a_h^{(k)}\}_{h=1}^H$ following $\pi_n$
\State $\mathcal{R} \leftarrow \mathcal{R} \cup \{(s_h^{(k)}, a_h^{(k)})\}_{h=1}^H$
\EndFor
\State Compute empirical covariance matrix: $\hat{\Sigma}_n = \sum_{(s, a) \in \mathcal{R}} \phi(s, a) \phi(s, a)^\top$
\State Define exploration bonus: $b_n(s, a) = \phi(s, a)^\top \hat{\Sigma}_n^{-1} \phi(s, a)$
\State Optimize policy mixture weights: $\alpha^{(n)} = \argmin_{\alpha=(\alpha_1, ..., \alpha_n), \alpha_i \geq 0, \sum_i \alpha_i = 1} \log \det \Big[ \sum_{i=1}^n \alpha_i \hat{\Sigma}_i \Big]$
\State $\pi_{n+1} \leftarrow \mathrm{NPG}(\pi_n, \Pi_\mathrm{mix}, \alpha^{(n)}, N_\mathrm{update}, r + b_n)$
\State $\Pi_\mathrm{mix} \leftarrow \Pi_\mathrm{mix} \cup \{\pi_{n+1}\}$
\EndFor
\end{algorithmic}
\caption{\alg (reward bonus version)}
\label{alg:rmaxpg_implemented_reward_bonus}
\end{algorithm}

\begin{algorithm}[h!]
\begin{algorithmic}[1]
\State \textbf{Require}: kernel function $\phi: \mathcal{S} \times \mathcal{A} \rightarrow \mathbb{R}^d$
\State Initialize policy $\pi_1$ randomly
\State Initialize policy mixture $\Pi_\mathrm{mix} \leftarrow \{\pi_1\}$
\State Initialize episode buffer: $\mathcal{R} \leftarrow \emptyset$
\For{episode $n = 1, \dots K$}
\For{trajectory $k = 1, \dots K$}
\State Gather trajectory $\tau_k = \{s_h^{(k)}, a_h^{(k)}\}_{h=1}^H$ following $\pi_n$
\State $\mathcal{R} \leftarrow \mathcal{R} \cup \{(s_h^{(k)}, a_h^{(k)})\}_{h=1}^H$
\EndFor
\State Compute empirical covariance matrix: $\hat{\Sigma}_n = \sum_{(s, a) \in \mathcal{R}} \phi(s, a) \phi(s, a)^\top$
\State Define exploration bonus: $b_n(s, a) = \phi(s, a)^\top \hat{\Sigma}_n^{-1} \phi(s, a)$
\State Optimize policy mixture weights: $\alpha^{(n)} = \argmin_{\alpha=(\alpha_1, ..., \alpha_n), \alpha_i \geq 0, \sum_i \alpha_i = 1} \log \det \Big[ \sum_{i=1}^n \alpha_i \hat{\Sigma}_i \Big]$
\State $\pi_{n+1} \leftarrow \mathrm{NPG}(\pi_n, \Pi_\mathrm{mix}, \alpha^{(n)}, N_\mathrm{update}, b_n)$
\State $\Pi_\mathrm{mix} \leftarrow \Pi_\mathrm{mix} \cup \{\pi_{n+1}\}$
\EndFor
\State Initialize policy $\pi_\mathrm{exploit}$ randomly
\State $\pi_\mathrm{exploit} \leftarrow \mathrm{NPG}(\pi_\mathrm{exploit}, \Pi_\mathrm{mix}, \alpha^{(K)}, N_\mathrm{update}, r)$
\end{algorithmic}
\caption{\alg (reward-free exploration version)}
\label{alg:rmaxpg_implemented_reward_free}
\end{algorithm}

\begin{algorithm}[h!]
  \begin{algorithmic}[1]
    \State \textbf{Input} policy $\pi$, policy mixture $\Pi_\mathrm{mix}=\{\pi_1, ..., \pi_n\}$, mixture weights $(\alpha_1, ..., \alpha_n)$, optional reward bonus $b: \mathcal{S} \times \mathcal{A} \rightarrow [0, 1]$
\For{policy update $j = 1, \dots N_\mathrm{update}$}
\State Sample roll in policy index $j \sim \mathrm{Multinomial}\{\alpha_1, ..., \alpha_n\}$
\State Sample roll in horizon index $h' \sim \mathrm{Uniform}\{0, ..., H-1\}$
\State Sample start state $s_0 \sim P(s_0)$
\For{$h=1, \dots, h'$}
\State $a_h \sim \pi_j(\cdot | s_h), s_{h+1} \sim P(\cdot | s_h, a_h)$
\EndFor
\For{$h=h'+1, \dots, H$}
\State $a_h \sim \pi(\cdot | s_h)$ ($\epsilon$-greedy if $h=h'+1$)
\State $s_{h+1}, r_{h+1} \sim P(\cdot | s_h, a_h)$
\EndFor
\State Perform policy gradient update on return $R = \sum_{h=h'}^H r(s_h, a_h)$
\EndFor
\State Return $\pi$
\end{algorithmic}
\caption{$\mathrm{NPG}(\pi, \Pi_\mathrm{mix}, \alpha, N_\mathrm{update}, r)$}
\label{alg:npg_implemented}
\end{algorithm}

\subsection{Environments}


\subsubsection{Bidirectional Diabolical Combination Lock}
\label{app:combolock}

The environment consists of a start state $s_0$ where the agent is placed (deterministically) at the beginning of every episode.
The action space consists of $10$ discrete actions, $\mathcal{A} = \{1, 2, ..., 10\}$.
In $s_0$, actions $1-5$ lead the agent to the initial state of the first lock and actions $6-10$ lead the agent to the initial state of the second lock.
Each lock $l$ consists of $3H$ states, indexed by $s_{1, h}^l, s_{2, h}^l, s_{3, h}^l$ for $h \in \{1, ..., H\}$.
A high reward of $R_l$ is obtained at the last states $s_{1, H}^l, s_{2, H}^l$.
The states $\{s_{3, h}^l\}_{h=1}^H$ are all ``dead states'' which yield $0$ reward.
Once the agent is in a dead state $s_{3, h}^l$, it transitions deterministically to $s_{3, h+1}^l$; thus entering a dead state at any time makes it impossible to obtain the final reward $R^l$.
At each ``good'' state $s_{1, h}^l$ or $s_{2, h}^l$, a single action leads the agent (stochastically with equal probability) to one of the next good states $s_{1, h+1}^l, s_{2, h+1}^l$. All other $9$ actions lead the agent to the dead state $s_{3, h+1}^l$. The correct action changes at every horizon length $h$ and the stochastic nature of the transitions precludes algorithms which plan deterministically.
In addition, the agent receives a negative reward of $-1/H$ for transitioning to a good state, and a reward of $0$ for transitioning to a dead state.
Therefore, a locally optimal solution is to learn a policy which transitions to a dead state as quickly as possible, since this avoids the $-1/H$ penalty.

States are encoded using a binary vector. The start state $s_0$ is simply the zero vector. In each lock, the state $s_{i, h}^l$ is encoded as a binary vector which is the concatenation of one-hot encodings of $i, h, l$.

One of the locks (randomly chosen) gives a final reward of $5$, while the other lock gives a final reward of $2$.
Therefore, in addition to the locally optimal policy of quickly transitioning to the dead state (with return $0$), another locally optimal solution is to explore the lock with reward $2$ and gather the reward there. This leads to a return of $V = 2 - \sum_{h=1}^H \frac{1}{H} = 1$, whereas the optimal return for going to the end of lock with reward $5$ is $V^\star = 5 - \sum_{h=1}^H \frac{1}{H} = 4$.
In order to ensure that the optimal reward is discovered for every lock, the agent must therefore explore both locks to the end.
We used Algorithm \ref{alg:rmaxpg_implemented_reward_free} for this environment.

\subsubsection{Mountain Car}
\label{app:mountaincar}

We used the \texttt{MountainCarContinuous-v0} OpenAI Gym environment at \url{https://gym.openai.com/envs/MountainCarContinuous-v0/}.
This environment has a 2-dimensional continuous state space and a 1-dimensional continuous action space. 
We used Algorithm \ref{alg:rmaxpg_implemented_reward_bonus} for this environment.

\subsubsection{Mazes}

We used the source code from \url{https://github.com/junhyukoh/value-prediction-network/blob/master/maze.py} to implement the maze environment, with the following modifications: i) the blue channel (originally representing the goal) is set to zero ii) the same maze is used across all episodes iii) the reward is set to be a constant $0$. We set the maze size to be $20 \times 20$. There are $5$ actions: \texttt{\{up, down, left, right, no-op\}}.
We used Algorithm \ref{alg:rmaxpg_implemented_reward_free} for this environment, omitting the exploitation step.

\subsection{Hyperparameters}

All methods were based on the PPO implementation of \citep{deeprl}. For the Diabolical Combination Lock and the MountainCar environments, we used the same policy network architecture: a 2-layer fully connected network with 64 hidden units at each layer and ReLU non-linearities. For the Diabolical Combination Lock environment, the last layer outputs a softmax over 10 actions and for Mountain Car the last layer outputs the parameters of a 1D Gaussian. For the Maze environments, we used a convolutional network with 2 convolutional layers ($32$ kernels of size $3 \times 3$ for the first, $64$ kernels of size $3 \times 3$ for the second, both with stride 2), followed by a single fully-connected layer with $512$ hidden units, and a final linear layer mapping to a softmax over the $5$ actions.
In all cases the RND network has the same architecture as the policy network, except that the last linear layer mapping hidden units to actions is removed.
We found that tuning the intrinsic reward coefficient was important for getting good performance for RND. Hyperparameters are shown in Tables \ref{table:ppo-hparams-cont} and \ref{table:ppo-hparams-maze}.

\begin{table}[h!]
  \caption{PPO+RND Hyperparameters for Combolock and Mountain Car}
  \centering
  \begin{tabular}{lllll}
    \toprule
    Hyperparameter & Values Considered & Final Value (Combolock) & Final Value (Mountain Car) \\
    \hline
    Learning Rate & $10^{-3}, 5\cdot 10^{-4}, 10^{-4}$ & $10^{-3}$ & $10^{-4}$ \\
    Hidden Layer Size     & $64$  & $64$ & $64$ \\
    $\tau_\mathrm{GAE}$ & 0.95 & 0.95 & 0.95 \\
    Gradient Clipping & $5.0$ & $5.0$ & $5.0$ \\
    Entropy Bonus & $0.01$ & $0.01$ & $0.01$ \\
    PPO Ratio Clip & $0.2$ & $0.2$ & $0.2$ \\
    PPO Minibatch Size & $160$ & $160$ & $160$ \\
    PPO Optimization Epochs & $5$ & $5$ & $5$ \\
    Intrinsic Reward Normalization     & true, false  & false & false \\
    Intrinsic Reward coefficient     & $0.5, 1, 10, 10^2, 10^3, 10^4$  & $10^3$ & $10^3$ \\
    Extrinsic Reward coefficient     & $1.0$  & $1.0$ & $1.0$ \\
    \bottomrule
  \end{tabular}
  \label{table:ppo-hparams-cont}
\end{table}

\begin{table}[h!]
  \caption{PPO+RND Hyperparameters for Mazes}
  \centering
  \begin{tabular}{lllll}
    \toprule
    Hyperparameter & Values Considered & Final Value  \\
    \hline
    Learning Rate & $10^{-3}, 5\cdot 10^{-4}, 10^{-4}$ & $10^{-3}$ \\
    Hidden Layer Size     & $512$  & $512$ \\
    $\tau_\mathrm{GAE}$ & 0.95 & 0.95 \\
    Gradient Clipping & $0.5$ & $0.5$ \\
    Entropy Bonus & $0.01$ & $0.01$ \\
    PPO Ratio Clip & $0.1$ & $0.1$ \\
    PPO Minibatch Size & $128$ & $128$ \\
    PPO Optimization Epochs & $10$ & $10$ \\
    Intrinsic Reward Normalization     & true, false  & true \\
    Intrinsic Reward coefficient     & $1, 10, 10^2, 10^3, 10^4$  & $10^3$\\
    \bottomrule
  \end{tabular}
  \label{table:ppo-hparams-maze}
\end{table}

The hyperparameters used for \alg are given in Tables \ref{table:rmaxpg-hparams-cont} and \ref{table:rmaxpg-hparams-maze}.
For the Diabolical Combination Lock experiments, we used a kernel $\phi(s, a) = s$, where $s$ is the binary vector encoding the state described in Section \ref{app:combolock}.
For Mountain Car, we used a Random Kitchen Sinks kernel \citep{randomkitchensinks} with 10 features using the following implementation: \url{https://scikit-learn.org/stable/modules/generated/sklearn.kernel_approximation.RBFSampler.html}. For the Maze environments, we used a randomly initialized convolutional network with the same architecture as the RND network as a kernel.

\begin{table}[h]
  \caption{\alg Hyperparameters for Combolock and Mountain Car}
  \centering
  \begin{tabular}{llll}
    \toprule
    Hyperparameter & Values Considered & Final Value (Combolock) & Final Value (MountainCar) \\
    \hline
    Learning Rate & $10^{-3}, 5\cdot 10^{-4}, 10^{-4}$ & $10^{-3}$ & $5 \cdot 10^{-4}$ \\
    Hidden Layer Size     & $64$  & $64$ & $64$ \\
    $\tau_\mathrm{GAE}$ & 0.95 & 0.95 & 0.95 \\
    Gradient Clipping & $5.0$ & $5.0$ & $5.0$ \\
    Entropy Bonus & $0.01$ & $0.01$ & $0.01$ \\
    PPO Ratio Clip & $0.2$ & $0.2$ & $0.2$ \\
    PPO Minibatch Size & $160$ & $160$ & $160$ \\
    PPO Optimization Epochs & $5$ & $5$ & $5$ \\
    $\epsilon$-greedy sampling     & $0, 0.01, 0.05$  & $0.05$ & $0.05$ \\
    \bottomrule
  \end{tabular}
  \label{table:rmaxpg-hparams-cont}
\end{table}

\begin{table}[h]
  \caption{\alg Hyperparameters for Mazes}
  \centering
  \begin{tabular}{llll}
    \toprule
    Hyperparameter & Values Considered & Final Value \\
    \hline
    Learning Rate & $10^{-3}, 5\cdot 10^{-4}, 10^{-4}$ & $5 \cdot 10^{-4}$ \\
    Hidden Layer Size     & $512$  & $512$ \\
    $\tau_\mathrm{GAE}$ & 0.95 & 0.95 \\
    Gradient Clipping & $0.5$ & $0.5$ \\
    Entropy Bonus & $0.01$ & $0.01$ \\
    PPO Ratio Clip & $0.1$ & $0.1$ \\
    PPO Minibatch Size & $128$ & $128$ \\
    PPO Optimization Epochs & $10$ & $10$ \\
    $\epsilon$-greedy sampling     & $0.05$  & $0.05$\\
    \bottomrule
  \end{tabular}
  \label{table:rmaxpg-hparams-maze}
\end{table}

\clearpage


\clearpage

%% file: main.bbl
\begin{thebibliography}{76}
\providecommand{\natexlab}[1]{#1}
\providecommand{\url}[1]{\texttt{#1}}
\expandafter\ifx\csname urlstyle\endcsname\relax
  \providecommand{\doi}[1]{doi: #1}\else
  \providecommand{\doi}{doi: \begingroup \urlstyle{rm}\Url}\fi

\bibitem[Abbasi-Yadkori et~al.(2019{\natexlab{a}})Abbasi-Yadkori, Bartlett,
  Bhatia, Lazic, Szepesvari, and Weisz]{abbasi2019politex}
Yasin Abbasi-Yadkori, Peter Bartlett, Kush Bhatia, Nevena Lazic, Csaba
  Szepesvari, and Gell{\'e}rt Weisz.
\newblock Politex: Regret bounds for policy iteration using expert prediction.
\newblock In \emph{International Conference on Machine Learning}, pages
  3692--3702, 2019{\natexlab{a}}.

\bibitem[Abbasi-Yadkori et~al.(2019{\natexlab{b}})Abbasi-Yadkori, Lazic,
  Szepesvari, and Weisz]{abbasi2019exploration}
Yasin Abbasi-Yadkori, Nevena Lazic, Csaba Szepesvari, and Gellert Weisz.
\newblock Exploration-enhanced politex.
\newblock \emph{arXiv preprint arXiv:1908.10479}, 2019{\natexlab{b}}.

\bibitem[Agarwal et~al.(2019)Agarwal, Kakade, Lee, and
  Mahajan]{agarwal2019optimality}
Alekh Agarwal, Sham~M. Kakade, Jason~D. Lee, and Gaurav Mahajan.
\newblock On the theory of policy gradient methods: Optimality, approximation,
  and distribution shift, 2019.

\bibitem[Agrawal and Jia(2017)]{agrawal2017optimistic}
Shipra Agrawal and Randy Jia.
\newblock Optimistic posterior sampling for reinforcement learning: worst-case
  regret bounds.
\newblock In \emph{Advances in Neural Information Processing Systems}, pages
  1184--1194, 2017.

\bibitem[Antos et~al.(2008)Antos, Szepesv{\'a}ri, and Munos]{antos2008learning}
Andr{\'a}s Antos, Csaba Szepesv{\'a}ri, and R{\'e}mi Munos.
\newblock Learning near-optimal policies with bellman-residual minimization
  based fitted policy iteration and a single sample path.
\newblock \emph{Machine Learning}, 71\penalty0 (1):\penalty0 89--129, 2008.

\bibitem[Ayoub et~al.(2020)Ayoub, Jia, Szepesv{\'{a}}ri, Wang, and
  Yang]{ayoub2020}
Alex Ayoub, Zeyu Jia, Csaba Szepesv{\'{a}}ri, Mengdi Wang, and Lin~F. Yang.
\newblock Model-based reinforcement learning with value-targeted regression.
\newblock abs/2006.01107, 2020.
\newblock URL \url{https://arxiv.org/abs/2006.01107}.

\bibitem[Azar et~al.(2012)Azar, G\'{o}mez, and
  Kappen]{Azar:2012:DPP:2503308.2503344}
Mohammad~Gheshlaghi Azar, Vicen\c{c} G\'{o}mez, and Hilbert~J. Kappen.
\newblock Dynamic policy programming.
\newblock \emph{J. Mach. Learn. Res.}, 13\penalty0 (1), November 2012.
\newblock ISSN 1532-4435.

\bibitem[Azar et~al.(2017)Azar, Osband, and Munos]{azar2017minimax}
Mohammad~Gheshlaghi Azar, Ian Osband, and R{\'e}mi Munos.
\newblock Minimax regret bounds for reinforcement learning.
\newblock In \emph{Proceedings of the 34th International Conference on Machine
  Learning-Volume 70}, pages 263--272. JMLR. org, 2017.

\bibitem[Bagnell et~al.(2004)Bagnell, Kakade, Schneider, and Ng]{NIPS2003_2378}
J.~A. Bagnell, Sham~M Kakade, Jeff~G. Schneider, and Andrew~Y. Ng.
\newblock Policy search by dynamic programming.
\newblock In S.~Thrun, L.~K. Saul, and B.~Sch\"{o}lkopf, editors,
  \emph{Advances in Neural Information Processing Systems 16}, pages 831--838.
  MIT Press, 2004.

\bibitem[Bellemare et~al.(2016)Bellemare, Srinivasan, Ostrovski, Schaul,
  Saxton, and Munos]{bellemare2016pseudocounts}
Marc Bellemare, Sriram Srinivasan, Georg Ostrovski, Tom Schaul, David Saxton,
  and Remi Munos.
\newblock Unifying count-based exploration and intrinsic motivation.
\newblock In D.~D. Lee, M.~Sugiyama, U.~V. Luxburg, I.~Guyon, and R.~Garnett,
  editors, \emph{Advances in Neural Information Processing Systems 29}, pages
  1471--1479. Curran Associates, Inc., 2016.

\bibitem[Bhandari and Russo(2019)]{russoGlobal}
Jalaj Bhandari and Daniel Russo.
\newblock Global optimality guarantees for policy gradient methods.
\newblock \emph{CoRR}, abs/1906.01786, 2019.
\newblock URL \url{http://arxiv.org/abs/1906.01786}.

\bibitem[Brafman and Tennenholtz(2002)]{brafman2002r}
Ronen~I Brafman and Moshe Tennenholtz.
\newblock R-max-a general polynomial time algorithm for near-optimal
  reinforcement learning.
\newblock \emph{Journal of Machine Learning Research}, 3\penalty0
  (Oct):\penalty0 213--231, 2002.

\bibitem[Brafman and Tennenholtz(2003)]{brafman2003r}
Ronen~I Brafman and Moshe Tennenholtz.
\newblock R-max-a general polynomial time algorithm for near-optimal
  reinforcement learning.
\newblock \emph{The Journal of Machine Learning Research}, 3:\penalty0
  213--231, 2003.

\bibitem[Brockman et~al.(2016)Brockman, Cheung, Pettersson, Schneider,
  Schulman, Tang, and Zaremba]{brockman2016openai}
Greg Brockman, Vicki Cheung, Ludwig Pettersson, Jonas Schneider, John Schulman,
  Jie Tang, and Wojciech Zaremba.
\newblock {OpenAI} gym.
\newblock \emph{arXiv preprint arXiv:1606.01540}, 2016.

\bibitem[Burda et~al.(2019)Burda, Edwards, Storkey, and
  Klimov]{burda2018exploration}
Yuri Burda, Harrison Edwards, Amos Storkey, and Oleg Klimov.
\newblock Exploration by random network distillation.
\newblock In \emph{International Conference on Learning Representations}, 2019.
\newblock URL \url{https://openreview.net/forum?id=H1lJJnR5Ym}.

\bibitem[Cai et~al.(2020)Cai, Yang, Jin, and Wang]{cai2019provably}
Qi~Cai, Zhuoran Yang, Chi Jin, and Zhaoran Wang.
\newblock Provably efficient exploration in policy optimization.
\newblock \emph{arXiv preprint arXiv:1912.05830}, 2020.

\bibitem[Chen and Jiang(2019)]{chen2019information}
Jinglin Chen and Nan Jiang.
\newblock Information-theoretic considerations in batch reinforcement learning.
\newblock In \emph{International Conference on Machine Learning}, pages
  1042--1051, 2019.

\bibitem[Dani et~al.(2008)Dani, Hayes, and Kakade]{dani2008stochastic}
Varsha Dani, Thomas~P Hayes, and Sham~M Kakade.
\newblock Stochastic linear optimization under bandit feedback.
\newblock In \emph{COLT}, pages 355--366, 2008.

\bibitem[Dann and Brunskill(2015)]{dann2015sample}
Christoph Dann and Emma Brunskill.
\newblock Sample complexity of episodic fixed-horizon reinforcement learning.
\newblock In \emph{Advances in Neural Information Processing Systems}, pages
  2818--2826, 2015.

\bibitem[Dann et~al.(2018)Dann, Jiang, Krishnamurthy, Agarwal, Langford, and
  Schapire]{dann2018polynomial}
Christoph Dann, Nan Jiang, Akshay Krishnamurthy, Alekh Agarwal, John Langford,
  and Robert~E Schapire.
\newblock On polynomial time {PAC} reinforcement learning with rich
  observations.
\newblock \emph{arXiv preprint arXiv:1803.00606}, 2018.

\bibitem[Dong et~al.(2019)Dong, Van~Roy, and Zhou]{dong2019provably}
Shi Dong, Benjamin Van~Roy, and Zhengyuan Zhou.
\newblock Provably efficient reinforcement learning with aggregated states.
\newblock \emph{arXiv preprint arXiv:1912.06366}, 2019.

\bibitem[Du et~al.(2019)Du, Luo, Wang, and Zhang]{du2019provably}
Simon~S Du, Yuping Luo, Ruosong Wang, and Hanrui Zhang.
\newblock Provably efficient $ q $-learning with function approximation via
  distribution shift error checking oracle.
\newblock \emph{arXiv preprint arXiv:1906.06321}, 2019.

\bibitem[Efroni et~al.(2020)Efroni, Shani, Rosenberg, and
  Mannor]{efroni2020optimistic}
Yonathan Efroni, Lior Shani, Aviv Rosenberg, and Shie Mannor.
\newblock Optimistic policy optimization with bandit feedback.
\newblock \emph{arXiv preprint arXiv:2002.08243}, 2020.

\bibitem[Even-Dar et~al.(2009)Even-Dar, Kakade, and
  Mansour]{even-dar2009online}
Eyal Even-Dar, Sham~M Kakade, and Yishay Mansour.
\newblock Online {M}arkov decision processes.
\newblock \emph{Mathematics of Operations Research}, 34\penalty0 (3):\penalty0
  726--736, 2009.

\bibitem[Fazel et~al.(2018)Fazel, Ge, Kakade, and Mesbahi]{fazel2018global}
Maryam Fazel, Rong Ge, Sham~M Kakade, and Mehran Mesbahi.
\newblock Global convergence of policy gradient methods for the linear
  quadratic regulator.
\newblock \emph{arXiv preprint arXiv:1801.05039}, 2018.

\bibitem[Geist et~al.(2019)Geist, Scherrer, and Pietquin]{geist2019theory}
Matthieu Geist, Bruno Scherrer, and Olivier Pietquin.
\newblock A theory of regularized markov decision processes.
\newblock \emph{arXiv preprint arXiv:1901.11275}, 2019.

\bibitem[Gutmann and Hyvärinen(2010)]{pmlr-v9-gutmann10a}
Michael Gutmann and Aapo Hyvärinen.
\newblock Noise-contrastive estimation: A new estimation principle for
  unnormalized statistical models.
\newblock In Yee~Whye Teh and Mike Titterington, editors, \emph{Proceedings of
  the Thirteenth International Conference on Artificial Intelligence and
  Statistics}, volume~9 of \emph{Proceedings of Machine Learning Research},
  pages 297--304, Chia Laguna Resort, Sardinia, Italy, 13--15 May 2010. PMLR.
\newblock URL \url{http://proceedings.mlr.press/v9/gutmann10a.html}.

\bibitem[Henaff(2019)]{henaff2019explicit}
Mikael Henaff.
\newblock Explicit explore-exploit algorithms in continuous state spaces.
\newblock In \emph{Advances in Neural Information Processing Systems}, pages
  9377--9387, 2019.

\bibitem[Jaksch et~al.(2010)Jaksch, Ortner, and Auer]{jaksch2010optimal}
Thomas Jaksch, Ronald Ortner, and Peter Auer.
\newblock Near-optimal regret bounds for reinforcement learning.
\newblock \emph{Journal of Machine Learning Research}, 11\penalty0
  (Apr):\penalty0 1563--1600, 2010.

\bibitem[Jarrett et~al.(2009)Jarrett, Kavukcuoglu, Ranzato, and
  LeCun]{JarrettKRL09}
Kevin Jarrett, Koray Kavukcuoglu, Marc'Aurelio Ranzato, and Yann LeCun.
\newblock What is the best multi-stage architecture for object recognition?
\newblock In \emph{ICCV}, pages 2146--2153. IEEE Computer Society, 2009.
\newblock ISBN 978-1-4244-4419-9.
\newblock URL
  \url{http://dblp.uni-trier.de/db/conf/iccv/iccv2009.html#JarrettKRL09}.

\bibitem[Jiang et~al.(2017)Jiang, Krishnamurthy, Agarwal, Langford, and
  Schapire]{jiang2017contextual}
Nan Jiang, Akshay Krishnamurthy, Alekh Agarwal, John Langford, and Robert~E.
  Schapire.
\newblock Contextual decision processes with low {B}ellman rank are
  {PAC}-learnable.
\newblock In \emph{International Conference on Machine Learning}, 2017.

\bibitem[Jin et~al.(2018)Jin, Allen-Zhu, Bubeck, and Jordan]{jin2018q}
Chi Jin, Zeyuan Allen-Zhu, Sebastien Bubeck, and Michael~I Jordan.
\newblock Is q-learning provably efficient?
\newblock In \emph{Advances in Neural Information Processing Systems}, pages
  4863--4873, 2018.

\bibitem[Jin et~al.(2019)Jin, Yang, Wang, and Jordan]{jin2019provably}
Chi Jin, Zhuoran Yang, Zhaoran Wang, and Michael~I Jordan.
\newblock Provably efficient reinforcement learning with linear function
  approximation.
\newblock \emph{arXiv preprint arXiv:1907.05388}, 2019.

\bibitem[Kakade(2001)]{Kakade01}
S.~Kakade.
\newblock A natural policy gradient.
\newblock In \emph{NIPS}, 2001.

\bibitem[Kakade and Langford(2002)]{kakade2002approximately}
Sham Kakade and John Langford.
\newblock {Approximately Optimal Approximate Reinforcement Learning}.
\newblock In \emph{Proceedings of the 19th International Conference on Machine
  Learning}, volume~2, pages 267--274, 2002.

\bibitem[Kakade(2003)]{kakade2003sample}
Sham~Machandranath Kakade.
\newblock \emph{On the sample complexity of reinforcement learning}.
\newblock PhD thesis, University of College London, 2003.

\bibitem[Kearns and Singh(2002)]{kearns2002optimal}
Michael Kearns and Satinder Singh.
\newblock Near-optimal reinforcement learning in polynomial time.
\newblock \emph{Machine Learning}, 49\penalty0 (2-3):\penalty0 209--232, 2002.

\bibitem[Konda and Tsitsiklis(2000)]{konda2000actor}
Vijay~R Konda and John~N Tsitsiklis.
\newblock Actor-critic algorithms.
\newblock In \emph{Advances in neural information processing systems}, pages
  1008--1014, 2000.

\bibitem[Lattimore and Hutter(2014{\natexlab{a}})]{lattimore2012pac}
Tor Lattimore and Marcus Hutter.
\newblock Near-optimal pac bounds for discounted mdps.
\newblock volume 558, pages 125--143. Elsevier, 2014{\natexlab{a}}.

\bibitem[Lattimore and Hutter(2014{\natexlab{b}})]{lattimore2014near}
Tor Lattimore and Marcus Hutter.
\newblock Near-optimal pac bounds for discounted mdps.
\newblock \emph{Theoretical Computer Science}, 558:\penalty0 125--143,
  2014{\natexlab{b}}.

\bibitem[Li(2009)]{li2009unifying}
Lihong Li.
\newblock \emph{A unifying framework for computational reinforcement learning
  theory}.
\newblock PhD thesis, Rutgers, The State University of New Jersey, 2009.

\bibitem[Li et~al.(2006)Li, Walsh, and Littman]{li2006towards}
Lihong Li, Thomas~J Walsh, and Michael~L Littman.
\newblock Towards a unified theory of state abstraction for {MDP}s.
\newblock In \emph{Proceedings of the 9th International Symposium on Artificial
  Intelligence and Mathematics}, pages 531--539, 2006.

\bibitem[Liu et~al.(2019)Liu, Cai, Yang, and Wang]{caiTRPO}
Boyi Liu, Qi~Cai, Zhuoran Yang, and Zhaoran Wang.
\newblock Neural proximal/trust region policy optimization attains globally
  optimal policy.
\newblock \emph{CoRR}, abs/1906.10306, 2019.
\newblock URL \url{http://arxiv.org/abs/1906.10306}.

\bibitem[Lu et~al.(2018)Lu, Schuurmans, and Boutilier]{lu2018non}
Tyler Lu, Dale Schuurmans, and Craig Boutilier.
\newblock Non-delusional q-learning and value-iteration.
\newblock In \emph{Advances in neural information processing systems}, pages
  9949--9959, 2018.

\bibitem[Misra et~al.(2020)Misra, Henaff, Krishnamurthy, and Langford]{homer}
Dipendra Misra, Mikael Henaff, Akshay Krishnamurthy, and John Langford.
\newblock Kinematic state abstraction and provably efficient rich-observation
  reinforcement learning.
\newblock In \emph{International conference on machine learning}, 2020.

\bibitem[Munos(2005)]{munos2005error}
R{\'e}mi Munos.
\newblock Error bounds for approximate value iteration.
\newblock 2005.

\bibitem[Neu et~al.(2017)Neu, Jonsson, and
  G{\'{o}}mez]{DBLP:journals/corr/NeuJG17}
Gergely Neu, Anders Jonsson, and Vicen{\c{c}} G{\'{o}}mez.
\newblock A unified view of entropy-regularized markov decision processes.
\newblock \emph{CoRR}, abs/1705.07798, 2017.

\bibitem[Oh et~al.(2017)Oh, Singh, and Lee]{VPN}
Junhyuk Oh, Satinder Singh, and Honglak Lee.
\newblock Value prediction network.
\newblock In I.~Guyon, U.~V. Luxburg, S.~Bengio, H.~Wallach, R.~Fergus,
  S.~Vishwanathan, and R.~Garnett, editors, \emph{Advances in Neural
  Information Processing Systems 30}, pages 6118--6128. Curran Associates,
  Inc., 2017.
\newblock URL
  \url{http://papers.nips.cc/paper/7192-value-prediction-network.pdf}.

\bibitem[Osband et~al.(2014)Osband, Van~Roy, and Wen]{osband2014generalization}
Ian Osband, Benjamin Van~Roy, and Zheng Wen.
\newblock Generalization and exploration via randomized value functions.
\newblock \emph{arXiv preprint arXiv:1402.0635}, 2014.

\bibitem[Paszke et~al.(2019)Paszke, Gross, Massa, Lerer, Bradbury, Chanan,
  Killeen, Lin, Gimelshein, Antiga, Desmaison, Kopf, Yang, DeVito, Raison,
  Tejani, Chilamkurthy, Steiner, Fang, Bai, and Chintala]{PyTorch}
Adam Paszke, Sam Gross, Francisco Massa, Adam Lerer, James Bradbury, Gregory
  Chanan, Trevor Killeen, Zeming Lin, Natalia Gimelshein, Luca Antiga, Alban
  Desmaison, Andreas Kopf, Edward Yang, Zachary DeVito, Martin Raison, Alykhan
  Tejani, Sasank Chilamkurthy, Benoit Steiner, Lu~Fang, Junjie Bai, and Soumith
  Chintala.
\newblock Pytorch: An imperative style, high-performance deep learning library.
\newblock In H.~Wallach, H.~Larochelle, A.~Beygelzimer, F.~d\textquotesingle
  Alch\'{e}-Buc, E.~Fox, and R.~Garnett, editors, \emph{Advances in Neural
  Information Processing Systems 32}, pages 8024--8035. Curran Associates,
  Inc., 2019.

\bibitem[Pathak et~al.(2017)Pathak, Agrawal, Efros, and
  Darrell]{pathak2017curiosity}
Deepak Pathak, Pulkit Agrawal, Alexei~A. Efros, and Trevor Darrell.
\newblock Curiosity-driven exploration by self-supervised prediction.
\newblock In \emph{ICML}, 2017.

\bibitem[Rahimi and Recht(2009)]{randomkitchensinks}
Ali Rahimi and Benjamin Recht.
\newblock Weighted sums of random kitchen sinks: Replacing minimization with
  randomization in learning.
\newblock In D.~Koller, D.~Schuurmans, Y.~Bengio, and L.~Bottou, editors,
  \emph{Advances in Neural Information Processing Systems 21}, pages
  1313--1320. Curran Associates, Inc., 2009.
\newblock URL
  \url{http://papers.nips.cc/paper/3495-weighted-sums-of-random-kitchen-sinks-replacing-minimization-with-randomization-in-learning.pdf}.

\bibitem[Russo(2019)]{russo2019worst}
Daniel Russo.
\newblock Worst-case regret bounds for exploration via randomized value
  functions.
\newblock In \emph{Advances in Neural Information Processing Systems}, pages
  14410--14420, 2019.

\bibitem[Scherrer(2014)]{Scherrer:API}
Bruno Scherrer.
\newblock Approximate policy iteration schemes: A comparison.
\newblock In \emph{Proceedings of the 31st International Conference on
  International Conference on Machine Learning - Volume 32}, ICML'14. JMLR.org,
  2014.

\bibitem[Scherrer and Geist(2014)]{scherrer2014local}
Bruno Scherrer and Matthieu Geist.
\newblock Local policy search in a convex space and conservative policy
  iteration as boosted policy search.
\newblock In \emph{Joint European Conference on Machine Learning and Knowledge
  Discovery in Databases}, pages 35--50. Springer, 2014.

\bibitem[Schulman et~al.(2015)Schulman, Levine, Abbeel, Jordan, and
  Moritz]{schulman2015trust}
John Schulman, Sergey Levine, Pieter Abbeel, Michael Jordan, and Philipp
  Moritz.
\newblock Trust region policy optimization.
\newblock In \emph{International Conference on Machine Learning}, pages
  1889--1897, 2015.

\bibitem[Schulman et~al.(2017)Schulman, Wolski, Dhariwal, Radford, and
  Klimov]{schulman2017proximal}
John Schulman, Filip Wolski, Prafulla Dhariwal, Alec Radford, and Oleg Klimov.
\newblock Proximal policy optimization algorithms.
\newblock \emph{arXiv preprint arXiv:1707.06347}, 2017.

\bibitem[Shangtong(2018)]{deeprl}
Zhang Shangtong.
\newblock Modularized implementation of deep rl algorithms in pytorch.
\newblock \url{https://github.com/ShangtongZhang/DeepRL}, 2018.

\bibitem[Shani et~al.(2019)Shani, Efroni, and Mannor]{shani2019adaptive}
Lior Shani, Yonathan Efroni, and Shie Mannor.
\newblock Adaptive trust region policy optimization: Global convergence and
  faster rates for regularized mdps, 2019.

\bibitem[Srinivas et~al.(2010)Srinivas, Krause, Kakade, and
  Seeger]{srinivas2010gaussian}
Niranjan Srinivas, Andreas Krause, Sham Kakade, and Matthias Seeger.
\newblock Gaussian process optimization in the bandit setting: no regret and
  experimental design.
\newblock In \emph{Proceedings of the 27th International Conference on
  International Conference on Machine Learning}, pages 1015--1022, 2010.

\bibitem[Strehl et~al.(2006)Strehl, Li, Wiewiora, Langford, and
  Littman]{strehl2006pac}
Alexander~L Strehl, Lihong Li, Eric Wiewiora, John Langford, and Michael~L
  Littman.
\newblock {PAC} model-free reinforcement learning.
\newblock In \emph{Proceedings of the 23rd international conference on Machine
  learning}, pages 881--888. ACM, 2006.

\bibitem[Sun et~al.(2019)Sun, Jiang, Krishnamurthy, Agarwal, and
  Langford]{sun2019model}
Wen Sun, Nan Jiang, Akshay Krishnamurthy, Alekh Agarwal, and John Langford.
\newblock Model-based rl in contextual decision processes: Pac bounds and
  exponential improvements over model-free approaches.
\newblock In \emph{Conference on Learning Theory}, pages 2898--2933, 2019.

\bibitem[Sutton et~al.(1999)Sutton, McAllester, Singh, and
  Mansour]{sutton1999policy}
Richard~S Sutton, David~A McAllester, Satinder~P Singh, and Yishay Mansour.
\newblock Policy gradient methods for reinforcement learning with function
  approximation.
\newblock In \emph{Advances in Neural Information Processing Systems},
  volume~99, pages 1057--1063, 1999.

\bibitem[Szepesv{\'a}ri and Munos(2005)]{szepesvari2005finite}
Csaba Szepesv{\'a}ri and R{\'e}mi Munos.
\newblock Finite time bounds for sampling based fitted value iteration.
\newblock In \emph{Proceedings of the 22nd international conference on Machine
  learning}, pages 880--887. ACM, 2005.

\bibitem[Szita and Szepesv{\'a}ri(2010)]{szita2010model}
Istv{\'a}n Szita and Csaba Szepesv{\'a}ri.
\newblock Model-based reinforcement learning with nearly tight exploration
  complexity bounds.
\newblock In \emph{Proceedings of the 27th International Conference on Machine
  Learning (ICML-10)}, pages 1031--1038, 2010.

\bibitem[Tropp et~al.(2015)]{tropp2015introduction}
Joel~A Tropp et~al.
\newblock An introduction to matrix concentration inequalities.
\newblock \emph{Foundations and Trends{\textregistered} in Machine Learning},
  8\penalty0 (1-2):\penalty0 1--230, 2015.

\bibitem[Valko et~al.(2013)Valko, Korda, Munos, Flaounas, and
  Cristianini]{valko2013finite}
Michal Valko, Nathaniel Korda, R{\'e}mi Munos, Ilias Flaounas, and Nelo
  Cristianini.
\newblock Finite-time analysis of kernelised contextual bandits.
\newblock \emph{arXiv preprint arXiv:1309.6869}, 2013.

\bibitem[van~den Oord et~al.(2018)van~den Oord, Li, and Vinyals]{CPC}
A{\"{a}}ron van~den Oord, Yazhe Li, and Oriol Vinyals.
\newblock Representation learning with contrastive predictive coding.
\newblock \emph{CoRR}, abs/1807.03748, 2018.
\newblock URL \url{http://arxiv.org/abs/1807.03748}.

\bibitem[Vincent et~al.(2010)Vincent, Larochelle, Lajoie, Bengio, and
  Manzagol]{DAE}
Pascal Vincent, Hugo Larochelle, Isabelle Lajoie, Yoshua Bengio, and
  Pierre-Antoine Manzagol.
\newblock Stacked denoising autoencoders: Learning useful representations in a
  deep network with a local denoising criterion.
\newblock \emph{J. Mach. Learn. Res.}, 11:\penalty0 3371–3408, December 2010.
\newblock ISSN 1532-4435.

\bibitem[Wen and Van~Roy(2013)]{wen2013efficient}
Zheng Wen and Benjamin Van~Roy.
\newblock Efficient exploration and value function generalization in
  deterministic systems.
\newblock In \emph{Advances in Neural Information Processing Systems}, pages
  3021--3029, 2013.

\bibitem[Williams(1992)]{williams1992simple}
Ronald~J Williams.
\newblock Simple statistical gradient-following algorithms for connectionist
  reinforcement learning.
\newblock \emph{Machine learning}, 8\penalty0 (3-4):\penalty0 229--256, 1992.

\bibitem[Yang and Wang(2019{\natexlab{a}})]{yang2019reinforcement}
Lin~F Yang and Mengdi Wang.
\newblock Reinforcement leaning in feature space: Matrix bandit, kernels, and
  regret bound.
\newblock \emph{arXiv preprint arXiv:1905.10389}, 2019{\natexlab{a}}.

\bibitem[Yang and Wang(2019{\natexlab{b}})]{yang2019sample}
Lin~F. Yang and Mengdi Wang.
\newblock Sample-optimal parametric q-learning using linearly additive
  features.
\newblock In \emph{International Conference on Machine Learning}, pages
  6995--7004, 2019{\natexlab{b}}.

\bibitem[Zanette et~al.(2020)Zanette, Brandfonbrener, Brunskill, Pirotta, and
  Lazaric]{pmlr-v108-zanette20a}
Andrea Zanette, David Brandfonbrener, Emma Brunskill, Matteo Pirotta, and
  Alessandro Lazaric.
\newblock Frequentist regret bounds for randomized least-squares value
  iteration.
\newblock In Silvia Chiappa and Roberto Calandra, editors, \emph{Proceedings of
  the Twenty Third International Conference on Artificial Intelligence and
  Statistics}, volume 108 of \emph{Proceedings of Machine Learning Research},
  pages 1954--1964, Online, 26--28 Aug 2020. PMLR.

\bibitem[Zhou et~al.(2020)Zhou, He, and Gu]{zhou2020provably}
Dongruo Zhou, Jiafan He, and Quanquan Gu.
\newblock Provably efficient reinforcement learning for discounted mdps with
  feature mapping.
\newblock \emph{arXiv preprint arXiv:2006.13165}, 2020.

\bibitem[Zinkevich(2003)]{zinkevich2003online}
Martin Zinkevich.
\newblock Online convex programming and generalized infinitesimal gradient
  ascent.
\newblock In \emph{Proceedings of the 20th International Conference on Machine
  Learning (ICML-03)}, pages 928--936, 2003.

\end{thebibliography}
